\documentclass[twoside,11pt]{article}

\usepackage{blindtext}
\usepackage{pdflscape}
\usepackage{jmlr2e}
\usepackage{import}
\usepackage{amsmath,amssymb,amsfonts}
\usepackage{stmaryrd}
\usepackage{mathtools,marvosym}
\usepackage{soul}
\usepackage{subcaption}
\usepackage{dblfloatfix}
\usepackage{float}
\usepackage{caption}
\usepackage{xparse, xspace}
\usepackage{algorithm}
\usepackage{algpseudocode}
\usepackage{MnSymbol}
\usepackage[shortlabels]{enumitem}
\usepackage{tikz-cd}
\usepackage{tikz}
\usepackage{multicol}
\usepackage{wasysym}
\usepackage{pst-node, multido}
\usetikzlibrary{automata, positioning, arrows, trees, decorations.pathreplacing, tikzmark}
\usepackage{float,graphicx}

\setlist{topsep=0pt}

\makeatletter
\renewcommand{\subset}{\subseteq}

\DeclareMathOperator{\Ex}{\mathbb{E}}
\DeclareMathOperator{\Var}{Var}

\newcommand{\N}{\mathbb{N}}

\newcommand{\R}{\mathbb{R}}

\newcommand{\tspace}{\allowbreak\ }
\NewDocumentCommand\given{s}{
    \IfBooleanTF{#1}{
        \,\middle|\,
    }{
        \,|\,
    }
}

\NewDocumentCommand\p{s m}{
    \IfBooleanTF{#1}{
        \left [ #2 \right ]
    }{
        \left ( #2 \right )
    }
}

\newcommand{\createwordcommand}[2]{
    \NewDocumentCommand{#1}{s s s}{\IfBooleanTF{##1}{
            \IfBooleanTF{##2}{
                \IfBooleanTF{##3}
                {#2}{#2\tspace}
            }{\tspace#2}
        }{\tspace#2\tspace}
    }
}

\NewDocumentCommand\prob{s m}{
    \IfBooleanTF{#1}{
        P\left(#2\right)
    }{
        P(#2)
    }
}

\newcommand\Dist{\Delta}

\newtheorem{claim}[theorem]{Claim} 

\usepackage{xr}
\makeatletter
\newcommand*{\addFileDependency}[1]{
  \typeout{(#1)}
  \@addtofilelist{#1}
  \IfFileExists{#1}{}{\typeout{No file #1.}}
}
\makeatother

\DeclareMathAlphabet{\mathpzc}{OT1}{pzc}{m}{it}

\newcommand{\x}{\times}
\newcommand{\reals}{\mathbf{R}}

\newcommand{\cA}{\mathcal{A}}

\newcommand{\cF}{\mathcal{F}}

\newcommand*\samethanks[1][\value{footnote}]{\footnotemark[#1]}

\usepackage{lastpage}
\jmlrheading{25}{2024}{1-\pageref{LastPage}}{10/23; Revised
2/24}{2/24}{23-1415}{Prakash Panangaden, Sahand Rezaei-Shoshtari, Rosie Zhao, David Meger, Doina Precup}
\ShortHeadings{Policy Gradient Methods in the Presence of Symmetries and State Abstractions}{Panangaden$^*$, Rezaei-Shoshtari$^*$, Zhao$^*$, Meger, Precup}

\firstpageno{1}

\begin{document}

\title{Policy Gradient Methods in the Presence of Symmetries and State Abstractions}

\author{\name Prakash Panangaden\thanks{Equal contributions; alphabetically ordered.} \email prakash@cs.mcgill.ca \\
       \addr School of Computer Science, McGill University\\
       and Mila -- Quebec AI Institute\\
      Montreal, QC, Canada
       \AND
       \name Sahand Rezaei-Shoshtari\samethanks \email srezaei@cim.mcgill.ca \\
       \addr School of Computer Science, McGill University\\
       and Mila -- Quebec AI Institute\\       
       Montreal, QC, Canada
       \AND
       \name Rosie Zhao\samethanks \email rosiezhao@g.harvard.edu \\
       \addr School of Engineering and Applied Sciences\\
       Harvard University\\
       Cambridge, MA, USA
       \AND
       \name David Meger \email dmeger@cim.mcgill.ca \\
       \addr School of Computer Science, McGill University\\
       and Mila -- Quebec AI Institute\\       
       Montreal, QC, Canada       
       \AND
       \name Doina Precup \email dprecup@cs.mcgill.ca \\
       \addr School of Computer Science, McGill University\\
       and Mila -- Quebec AI Institute\\
       and DeepMind\\
       Montreal, QC, Canada}

\editor{Martha White}

\maketitle

\begin{abstract}
Reinforcement learning (RL) on high-dimensional and complex problems relies on abstraction for improved efficiency and generalization. In this paper, we study abstraction in the continuous-control setting, and extend the definition of Markov decision process (MDP) homomorphisms to the setting of continuous state and action spaces. We derive a policy gradient theorem on the abstract MDP for both stochastic and deterministic policies. Our policy gradient results allow for leveraging approximate symmetries of the environment for policy optimization. Based on these theorems, we propose a family of actor-critic algorithms that are able to learn the policy and the MDP homomorphism map simultaneously, using the lax bisimulation metric. Finally, we introduce a series of environments with continuous symmetries to further demonstrate the ability of our algorithm for action abstraction in the presence of such symmetries. We demonstrate the effectiveness of our method on our environments, as well as on challenging visual control tasks from the DeepMind Control Suite. Our method's ability to utilize MDP homomorphisms for representation learning leads to improved performance, and the visualizations of the latent space clearly demonstrate the structure of the learned abstraction.

\end{abstract}


        

\section{Introduction}
\label{sec:intro}
Reinforcement learning on high-dimensional observations relies on representation learning and abstraction for learning a simpler problem that can be solved efficiently \citep{li2006towards, abel2016near}. A major obstacle, however, is the coupling between states, actions, and rewards, particularly in complex continuous control problems. Strategies have been developed to find ways to reduce the state space by capturing \emph{behavioral equivalence} between individual states. One formalization of this for MDPs is \emph{bisimulation}~\citep{givan2003equivalence}, which was originally introduced for labelled transition systems in the early 1980's~\citep{milner1989communication}. Bisimulation defines an equivalence relation over the state space, which allows one to quotient the state space by considering the equivalence classes under this relation. Bisimulation and their associated bisimulation metrics~\citep{Ferns04}— which are used to approximate this equivalence relation — have previously been used for abstraction and model minimization.

Alternatively, one could use the quotiented state space to define a new environment with transition dynamics and rewards that preserve the structure of the original state space, and define a function between the original and new MDP. Thus, closely related to bisimulation are \emph{MDP homomorphisms} \citep{ravindran2004algebraic, ravindran2001symmetries, ravindran2004approximate}, which capture behavioral equivalence via maps between MDPs that have certain preservation properties. Similar to bisimulation, one can use MDP homomorphisms to exploit (approximate) symmetries of an MDP for joint state-action abstraction.



MDP homomorphisms, developed in the context of discrete state and action spaces, are structure-preserving maps between MDPs that preserve value functions. Typically, they are used to map an MDP to an abstract MDP in a way such that no relevant information is lost.  \citet{ravindran2001symmetries} show that policies can be pulled back, or \emph{lifted}, from the abstract MDP to the original one while preserving optimality.  Pulling back a policy in this way is a tricky construction and explicitly uses the finiteness of the state and action spaces. From the practical perspective, recent works have shown that using MDP homomorphisms are effective in guiding the learning in discrete problems \citep{van2020plannable, van2020mdp, biza2019online}. Figure \ref{fig:intro} shows schematics and key properties of MDP homomorphisms, which we formally define in Section \ref{sec:background}. 

Our first contribution is that we extend MDP homomorphisms to the continuous setting. This is crucial if we are to use these ideas for control of dynamical systems in physical spaces, as in robotics.  The mathematics involved is significantly deeper than in the finite case and in some cases the finite case provides no guidance on how to proceed.  We show that the value functions and the optimal value function are preserved for both stochastic and deterministic policies, as in the finite case.  Lifting the policy from the abstract space to the original is one crucial example where we have to do something completely different from \citet{ravindran2001symmetries}, where we appeal to using classical tools in functional analysis.

\begin{figure}[b!]
    \centering
    \begin{subfigure}[b]{0.44\textwidth}
         \centering
         \includegraphics[width=\textwidth]{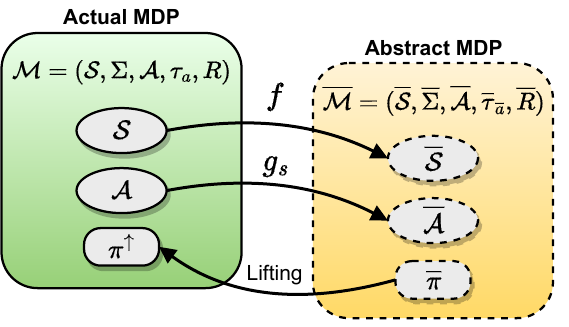}
         \caption{Components of an MDP homomorphism.}
         \label{fig:intro_a}
    \end{subfigure}
    \hfill
    \begin{subfigure}[b]{0.54\textwidth}
         \centering
         \includegraphics[width=\textwidth]{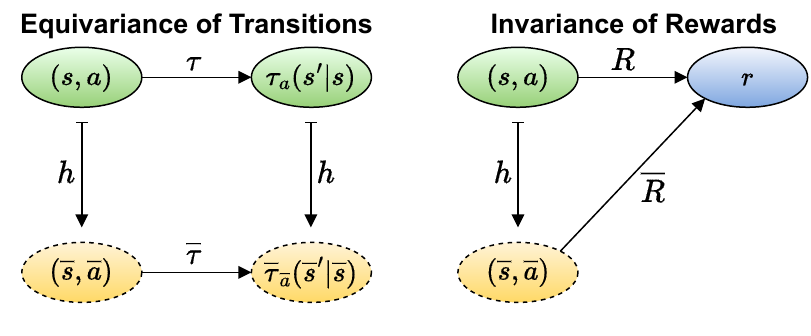}
         \caption{Commutative diagrams for MDP homomorphisms.}
         \label{fig:intro_b}
    \end{subfigure}    
    \caption{Overview of an MDP homomorphism $h = (f, g_s)$. \textbf{(a)} Components of an MDP homomorphism map, and the relation between the actual and abstract MDPs. \textbf{(b)} Commutative diagrams for MDP homomorphisms demonstrating the equivariance of transitions and the invariance of rewards. Diagram is adapted from \citet{ravindran2001symmetries}.}
    \label{fig:intro}
\end{figure}

The next significant contribution is that we derive a version of the policy gradient theorem \citep{sutton2000policy, silver2014deterministic} that tightly integrates the MDP homomorphism in the policy optimization process. In other words, one can use the policy gradient obtained from the abstract MDP, referred to as the \emph{homomorphic policy gradient} (HPG), to optimize for the performance measure defined on the original MDP. We rigorously prove this result for both deterministic and stochastic policies, and show that HPG can act as an additional gradient estimator capable of utilizing approximate symmetries for improved sample efficiency. As policy gradient methods remain a key family of RL algorithms, particularly for continuous control problems \citep{kiran2021deep, arulkumaran2017deep}, our homomorphic policy gradient derivation can have significant outcome for policy gradient algorithms in the presence of state abstraction.

Our third contribution is that we propose a deep actor-critic algorithm, referred to as the \emph{deep homomorphic policy gradient} (DHPG) algorithm based on our novel theoretical results. DHPG is able to simultaneously optimize the policy, learn the homomorphism map, and exploit the abstraction of MDP homomorphisms for policy optimization. We empirically show that state-action abstractions learned through MDP homomorphisms provide a natural inductive bias for representation learning on challenging visual control problems, resulting in performance and sample efficiency improvements over strong baselines. 

Finally, we show how to collapse an MDP when there is a group of symmetries which is also continuous.  Thus, for example if a system is spherically symmetric the system is invariant under the action of the rotation group $SO(3)$ and this is certainly not a finite group.  Discrete symmetries can and do occur in continuous systems but in general one will be dealing with continuous symmetries. Additionally, to demonstrate the ability of DHPG in learning continuous symmetries, we have developed a series of environments with continuous symmetries. In summary, our contributions can be listed as: 
\begin{enumerate}[noitemsep]
    \item Defining continuous MDP homomorphisms on continuous state and action spaces, and proving the existence of the lifted policy in the general case of stochastic policies.  
    \item Proving the value and optimal value equivalence of MDP homomorphisms for the general case of stochastic policies.
    \item Deriving the homomorphic policy gradient theorem for both stochastic and deterministic policies. 
    \item Developing a family of deep actor-critic algorithms that are able to learn the optimal policy simultaneously with the MDP homomorphisms map. Our algorithm works for both stochastic and deterministic policies through the use of a novel and computationally efficient policy lifting procedure.
    \item Developing a series of novel RL environments with continuous symmetries that allow for benchmarking the ability of agents in learning and leveraging continuous environmental symmetries.
\end{enumerate}
Notably, compared to the prior work of \citet{rezaei2022continuous}, our theoretical and empirical contributions are not limited to deterministic policies and bijective action encoders. Instead, we prove the value equivalence property and the homomorphic policy gradient theorem for the general case of stochastic policies and surjective action encoders, and propose a computationally efficient way for lifting a general stochastic policy. Empirically, we show that stochastic DHPG is superior to deterministic DHPG in environments with continuous symmetries as it is capable of a more powerful action abstraction. Our code for DHPG and the novel environments with continuous symmetries are publicly available\footnote{\href{https://github.com/sahandrez/homomorphic_policy_gradient}{\texttt{https://github.com/sahandrez/homomorphic\_policy\_gradient}}}.

The paper is structured as follows: in Sections~\ref{sec:related_work} and \ref{sec:background} we provide an overview of related work and introduce relevant background, including finite MDP homomorphisms, bisimulation, and policy gradient methods. In Section~\ref{sec:cont_MDP_homomorphism} we formally introduce continuous MDPs and continuous MDP homomorphisms and prove key equivalence properties. In Section~\ref{sec:hom_policy_gradient_thm} we prove the stochastic and deterministic homomorphic policy gradient theorems and subsequently introduce the DHPG algorithm in Section~\ref{sec:hom_algorithm}. Finally, we provide experimental results of DHPG on continuous control tasks in Section~\ref{sec:exp_results}.

\section{Related Work}
\label{sec:related_work}
\paragraph{State Abstraction.} Abstraction can be defined as a process that maps the original representation to an abstract representation that is more compact and easier to work with \citep{li2006towards}. Probabilistic bisimulation, which we will refer to as just ``bisimulation''\citep{larsen1991bisimulation} is one notion of behavioral equivalence between systems.  It was extended to continuous state spaces by \citet{blute1997bisimulation} and \citet{Desharnais02} and extended to MDPs by \citet{givan2003equivalence}. Bisimulation metrics \citep{Desharnais99b, ferns2005metrics, ferns2006methods, ferns2011bisimulation} define a pseudometric to quantify the degree of behavioural similarity.  Recently, \citet{zhang2020learning} defined a loss function for learning representations via bisimilarity of latent states, and \citet{kemertas2021towards} have further improved its robustness. \citet{castro2020scalable} has proposed a method to approximate the bisimulation metric for deterministic MDPs with continuous states but discrete actions. \citet{van2020plannable} have defined a contrastive loss based on MDP homomorphisms for learning an abstract MDP for planning, but their method is only applicable to finite MDPs. Another approach is to directly embed the MDP homomorphic relation in the network architecture \citep{van2020mdp, van2021multi}. Other recently proposed metrics \citep{le2021metrics} seek to learn representations that preserve values \citep{grimm2020value, grimm2021proper} or policies \citep{agarwal2020contrastive}, or via a sampling-based similarity metric \citep{castro2021mico}. Recently, \citet{kemertas2022approximate} have incorporated the bisimulation relation within the approximate policy iteration. Finally, state abstractions can in principle help improve transferring of policies \citep{abel2019state, castro2010using, soni2006using, sorg2009transfer, rajendran2009learning}, or learning temporally extended actions \citep{castro2011automatic, wolfe2006decision, wolfe2006defining, sutton1999between}.

\paragraph{Action Abstraction.} Action representations are often studied in the context of large discrete action spaces \citep{sallans2004reinforcement} as a form of a look-up embedding that is known \emph{a-priori} \citep{dulac2015deep}, factored representations \citep{sharma2017learning}, or policy decomposition \citep{chandak2019learning}. Action representations can also be learned from expert demonstrations \citep{tennenholtz2019natural}. More related to our work is dynamics-aware embeddings \citep{whitney2019dynamics} where a combined state-action embedding for continuous control is learned. In contrast, we use the notion of homomorphisms to learn the state-dependent action representations, while preserving values. Lastly, action representations can be combined with temporal abstraction \citep{sutton1999between} for discovering temporally extended actions \citep{ravindran2003relativized, abel2020value, castro2010using, castro2011automatic}.

\paragraph{State Representation Learning.} Extant methods for learning the
underlying state space from raw observations often use latent models
\citep{gelada2019deepmdp, hafner2019dream, hafner2019learning, ha2018world, biza2021learning},
auxiliary prediction tasks \citep{jaderberg2016reinforcement, liu2019self,
  lyle2021effect}, physics-inspired inductive biases
\citep{jonschkowski2015learning, cranmer2020lagrangian,
  greydanus2019hamiltonian}, unsupervised learning \citep{hjelm2018learning,
  liu2021aps}, or self-supervised learning \citep{anand2019unsupervised,
  sinha2021s4rl, hansen2020self, hansen2021generalization,
  fan2021secant}. From another point of view, representation learning can be
effectively decoupled from the RL problem \citep{eslami2018neural,
  stooke2021decoupling}.  Symmetries of the environment can also be used for
representation learning \citep{mondal2022eqr, mahajan2017symmetry, park2022learning, wang2021so2, higgins2018towards,
  higgins2021symetric, quessard2020learning, caselles2019symmetry}. In fact, MDP
homomorphisms are specializations of such approaches for RL.  A key distinguishing factor of MDP homomorphisms is their ability to take actions into account for representation learning in the same premises as \citet{thomas2017independently}. Recently, simple image augmentation methods have shown significant improvements in RL performance \citep{yarats2020image, lee2019network}. Since these approaches are in general orthogonal to state abstractions, they can be combined together.

\paragraph{Equivariant Representation Learning.} Using equivariance to leverage symmetries in data has been a fruitful line of machine learning research, where enforcing equivariance properties in the model architecture has led to state-of-the-art performance across several data modalities and applications. These domains include segmentation and classification tasks in computer vision~\citep{cohen2016group}, medical imaging~\citep{winkels2019pulmonary, veeling2018rotation}, 3D model classification~\citep{thomas2018tensor, chen2021equivariant}, quantum chemistry~\citep{qiao2021unite, satorras2021n, batzner20223}, and protein structure classification~\citep{eismann2021hierarchical, ganea2021independent, jumper2021highly}. Since the utility of translation equivariance was demonstrated for traditional CNNs~\citep{lecun1989handwritten, lecun1995convolutional}, in recent years these convolutional layers have been generalized to be equivariant to discrete groups--- such as finite rotations, translations, and reflections~\citep{cohen2016group}--- and continuous groups--- such as the rotation group SO(3) and the Euclidean and special Euclidean groups E(3) and SE(3)~\citep{cohen2018spherical, kondor2018clebsch, weiler20183d, cohen2019gauge}. The equivariance constraints imposed on these architectures are very rigid, and previous work has shown that true equivariance is difficult to achieve~\citep{azulay2018deep, engstrom2017rotation}. Further, the groups for these equivariant networks are typically fixed and known apriori; however, there has been work which constructs the appropriate equivariant network for arbitrary matrix Lie groups~\citep{finzi2021practical} and presents algorithms to automatically discover symmetries pertaining to Lie groups~\citep{dehmamy2021automatic}.

\section{Background}
\label{sec:background}
\subsection{Markov Decision Processes}
Reinforcement learning is based on an agent interacting with its environment and acquiring
rewards as it does so.  It seeks to maximize the expected reward and learns to do this
through its interaction with the environment.  Markov decision processes are the
basic model formalizing the interaction between an agent and its environment.

\begin{definition}[MDP]
  A \emph{Markov decision process} (MDP) is a tuple $\mathcal{M} = (\mathcal{S}, \mathcal{A}, \tau_a, R, \gamma )$
  where $\mathcal{S}$ is a set of states, $\mathcal{A}$ is a set of actions, for each $a \in \mathcal{A}$ we have
  $\tau_a: \mathcal{S} \to\Dist(\mathcal{S})$ where $\Dist(\mathcal{S})$ denotes the set of probability distributions
  over $\mathcal{S}$, $R: \mathcal{S} \times \mathcal{A} \to \R$ is the reward function, and $\gamma \in [0, 1)$ is the
  discount factor.
\end{definition}
Initially, we assume $\mathcal{S}$ and $\mathcal{A}$ to be finite; in
Section~\ref{sec:cont_mdp_homomorphism}, we will define MDPs on more general state and
action spaces.  From a state $s \in \mathcal{S}$, an agent acting according to policy
$\pi: \mathcal{S} \to\Dist(\mathcal{A})$ selects actions $a \sim \pi(\cdot | s)$ and transitions to
$s' \sim \tau_a(\cdot | s)$, yielding reward $r = R(s,a)$.  The objective is to maximize
the expected return by learning an optimal policy:
\[\pi^* = \arg\max_\pi \Ex_{\tau}[\sum_{t=0}^\infty \gamma^t R(s_t, a_t)].\]

Here, note that we assume $\gamma < 1$ to ensure convergence of the return (although $\gamma = 1$ is permitted for episodic tasks).

The \emph{value function} $V^\pi(s)$ gives the expected return starting from state $s$
and following policy $\pi$.  The \emph{action-value function} $Q^\pi(s,a)$ gives the
expected return starting from state $s$, taking action $a$ and thereafter following
$\pi$.  The value function is the fixed point of the Bellman operator $T^\pi: \R^{\mathcal{S} \times \mathcal{A}} \to \R^{\mathcal{S} \times \mathcal{A}}$ defined as:
$$T^\pi V(s) := \Ex_{\substack{a \sim \pi(\cdot | s) \\ s' \sim \tau_a(\cdot | s)}} [r + \gamma V(s')].$$
Similarly the optimal value function $V^*$ is the fixed point of the Bellman optimality operator $T^*: \R^{\mathcal{S} \times \mathcal{A}} \to \R^{\mathcal{S} \times \mathcal{A}}$:
$$T^* V(s) := \max_a \left[\Ex_{\substack{s'\sim \tau_a(\cdot | s)}} [r + \gamma V(s')]\right].$$
Analogous Bellman equations are defined for $Q^\pi$ and $Q^*$~\citep{sutton2018reinforcement}.

\subsection{Policy Gradient Theorems}
\label{sec:background_pg}
Reinforcement learning algorithms can be broadly divided into \emph{value-based} and \emph{policy gradient} (PG) methods. While value-based methods select actions via a greedy maximization step  based on the learned action-values, policy gradient methods directly optimize a parameterized policy $\pi_\theta$ based on the gradient of the performance measure $J(\theta)$, defined as:
\begin{equation}
    \label{eq:perf_measure}
    J(\theta) = \mathbb{E}_{\pi}[V^\pi (s)],
\end{equation}
where the expectation is taken with respect to the policy, transitions, and the initial state distribution of the actual MDP. Unlike value-based methods, policy gradient algorithms inherit the strong, albeit local, convergence guarantees of the gradient descent and are naturally extendable to continuous actions. The fundamental theorem underlying policy gradient methods is the \emph{policy gradient theorem} \citep{sutton2000policy}:
\begin{theorem}[\citet{sutton2000policy}] 
    \label{thm:stoch_pg}
    Let $\pi_\theta : \mathcal{S} \to \Dist(\mathcal{A})$ be a stochastic policy defined on $\mathcal{M}$. Then the gradient of the performance measure $J(\theta)$ w.r.t. $\theta$ is: 
    $$
        \nabla_\theta J(\pi_\theta) = \int_{s \in {\mathcal{S}}} \rho^{\pi_\theta}(s) \int_{a \in {\mathcal{A}}} \nabla_\theta \pi_\theta(da | s) Q^{\pi_\theta}(s, a) ds,
    $$
    where $\rho^{\pi_\theta} (s) = \lim_{t \rightarrow \infty} \gamma^t P(s_t=s | s_0, a_{0:t} \sim \pi_\theta)$ is the discounted stationary distribution of states under $\pi_\theta$.
\end{theorem}
 In Theorem \ref{thm:stoch_pg}, $\rho^{\pi_\theta} (s)$ is assumed to exist and to be independent of the initial state distribution (ergodicity assumption). The significance of the policy gradient theorem is that the effect of policy changes on the state distribution does not appear in its expression, allowing for a sample-based estimate of the gradient \citep{williams1992simple}. The deterministic policy gradient (DPG) is derived for deterministic policies by \citet{silver2014deterministic} as:
\begin{theorem}[\citet{silver2014deterministic}] 
    \label{thm:det_pg}
    Let $\pi_\theta : \mathcal{S} \to \mathcal{A}$ be a deterministic policy defined on $\mathcal{M}$. Then the gradient of the performance measure $J(\theta)$ w.r.t. $\theta$ is:
    $$
        \nabla_\theta J(\pi_\theta) = \int_{s \in {\mathcal{S}}} \rho^{\pi_\theta}(s) \nabla_\theta \pi_\theta(s) \nabla_a Q^{\pi_\theta}(s, a) \big|_{a=\pi_\theta(s)} ds,
    $$
    where $\rho^{\pi_\theta} (s) = \lim_{t \rightarrow \infty} \gamma^t P(s_t=s | s_0, a_{0:t} \sim \pi_\theta)$ is the discounted stationary distribution of states under $\pi_\theta$.
\end{theorem}
Since DPG does not need to integrate over the action space, it is often more sample-efficient than the stochastic policy gradient \citep{silver2014deterministic}. However, noise needs to be manually injected during exploration as the deterministic policy does not have any inherent means of exploration. Finally, it is worth noting that due to the differentiation of the value function with respect to $a$, DPG is only applicable to continuous actions. 

\subsection{Bisimulation and Bisimulation Metrics}
Bisimulation is a fundamental equivalence relation on the state space which captures the
idea of behavioural similarity.  
It was introduced in the late 1970's and early 1980's by
\citet{Milner80,Milner89} and \citet{Park81b} in a non-probabilistic context and then
extended to probabilistic systems by~\citet{Larsen91}.  The extension to continuous state
spaces was done by \citet{Blute97} and \citet{Desharnais02}.  These models did
not involve rewards but it is a minor modification to add rewards as was done
by~\citet{givan2003equivalence}. The bisimulation relation on an MDP is formally defined as:

\begin{definition}[Bisimulation]
A \emph{bisimulation relation} on an MDP $\mathcal{M} = (\mathcal{S}, \mathcal{A}, \tau_a, R, \gamma)$ is an equivalence
relation $B$ on $\mathcal{S}$ such that if \(sBt\) holds then for any action $a$ and any equivalence
class of $C$ of $B$ we have:
\begin{itemize}[noitemsep]
\item $R(s,a) = R(t,a)$ and
\item \( \tau_a(C | s) = \tau_a(C | t)\).
\end{itemize}
If there exists such a relation between two states $s$ and $t$ we say that $s$ and $t$ are
bisimilar and write \(s \sim t\).
\end{definition}
It is possible to define bisimulation as the greatest fixed point of a suitable operator
on the complete lattice of equivalence relations on $\mathcal{S}$~\citep{milner1989communication}. Bisimulation is not robust to small perturbations in the system parameters.  In a
quantitative setting like MDPs we need to use metrics which give a quantitative notion to
similarity.

In order to define a ``metric'' which can be viewed as a quantitative version of
bisimulation, it is natural to think of a pseudometric with the property that its kernel is
the bisimulation equivalence relation.  This is defined through a fixed-point
construction.  We equip $\mathfrak{M}$, the space of $1$-bounded pseudometrics on $\mathcal{S}$, with the following
metric:
\[ \Delta(m_1,m_2) := \sup_{x,y\in \mathcal{S}} |m_1(x,y)-m_2(x,y)|. \]
Here, $m_1,m_2$ are elements of $\mathfrak{M}$, \emph{i.e.}\ $1$-bounded pseudometrics. We then define an operator called $T_K: \mathfrak{M}\to\mathfrak{M}$ as follows:
\[ T_K(m)(x,y) = \max_{a\in\cA}[|R(x,a)-R(y,a)| + \gamma W_1(m)(\tau_a(x),\tau_a(y))].\]
Here $\tau_a(x)$ represents the probability distribution over the state space when the
system executes an $a$-transition starting from $x$ and similarly for $\tau_a(y)$.  The
metric $W_1$ on probability distributions is the well-known Kantorovich
metric\footnote{More often called the ``Wasserstein'' metric for reasons that have no
  historical validity.} which depends on $m$.  One can readily show that the space $\mathfrak{M}$
equipped with $\Delta$ is a complete metric space and that the function or operator
\(T_K\) is contractive with respect to the metric $\Delta$.  Thus, by the Banach
fixed-point theorem, it has a unique fixed point.  This is the bismulation
metric\footnote{In \citet{Ferns05} a different fixed-point theorem based on lattice theory
  was used.}.

\subsection{Finite MDP Homomorphisms}
Closely related to the concept of behavioural equivalence of states in MDPs are model
minimization methods, which identify reductions in the original MDP to obtain an
equivalent, smaller MDP.  This gave rise to the notion of MDP homomorphism, originally
proposed by \citet{ravindran2001symmetries}.  We will present the definitions and various
results about MDP homomorphisms assuming the state and action spaces are finite.

\begin{definition}[MDP Homomorphism]
    An \emph{MDP homomorphism} $h$ between MDPs $\mathcal{M} \!=\! (\mathcal{S}, \mathcal{A}, \tau_a, R, \gamma)$ and
    $\mathcal{\overline{M}} \!=\! (\mathcal{\overline{S}}, \mathcal{\overline{A}}, \overline{\tau}_{\overline{a}}, \overline{R}, \overline{\gamma})$ is a tuple of
    surjective maps $h \!=\! (f, g_s)$ where $f\!:\! \mathcal{S} \!\to\! \mathcal{\overline{S}}$ and $g_s\!:\! \mathcal{A}
    \!\to\! \mathcal{\overline{A}}$ for each $s \in \mathcal{S}$ such that: 
    \begin{enumerate}[noitemsep]
        \item $R(s,a) = \overline{R}(f(s), g_s(a))$  for every $s \in \mathcal{S}, a \in \mathcal{A}$;
        \item For every $s, s' \in \mathcal{S}$ and $a \in \mathcal{A}$, 
        \begin{equation}
        \label{eq:finite_hom_transitions}
            \overline{\tau}_{g_s(a)}(f(s') | f(s)) = \sum_{s'' \in [s']_{B_h\mid \mathcal{S}}} \tau_a(s'' | s),
        \end{equation}
        where $B_h$ is the partition of $\mathcal{S} \times \mathcal{A}$ induced by the equivalence relation of $h$, $B_h \mid \mathcal{S}$ is the projection of $B_h$ onto $\mathcal{S}$, and $[s']_{B_h \mid \mathcal{S}}$ is the partition of $B_h \mid \mathcal{S}$ containing $s'$.
    \end{enumerate}
\end{definition}
In other words, the probability of transitioning between $f(s)$ and $f(s')$ in the image MDP $\mathcal{\overline{M}}$ under action $g_s(a)$ equals the probability of transitioning from $s$ to the subset $[s']_{B_h \mid \mathcal{S}}$ in the original MDP $\mathcal{M}$ under action $a$. Figure \ref{fig:intro_b} shows the commutative diagram of MDP homomorphisms. A key property of MDP homomorphisms is the \emph{optimal value equivalence}, showing the optimal value function is preserved under this mapping.

\begin{theorem}[\citet{ravindran2001symmetries}]
\label{thm:optvalue_discrete}
Let $h$ be an MDP homomorphism from $\mathcal{M} = (\mathcal{S}, \mathcal{A}, \tau_a, R, \gamma)$ to $\mathcal{\overline{M}} = (\mathcal{\overline{S}}, \mathcal{\overline{A}}, \overline{\tau}_{\overline{a}}, \overline{R}, \overline{\gamma})$.  Then for any $(s,a) \in \mathcal{S} \times \mathcal{A}$, we have:
$$
    Q^*(s,a) = \overline{Q}^*(f(s), g_s(a)).
$$
\end{theorem}
The optimal policies of an MDP and its image under an MDP homomorphism are also closely related.  Given a policy on the image MDP, we can define a new, lifted policy on the original MDP that has the ``equivalent behaviour''.

\begin{definition}[Lifted Policy]
\label{def:lifted_finite}
Let $h$ be an MDP homomorphism from $M \!=\! (S, A, \tau_a, R, \gamma)$ to $\mathcal{\overline{M}} \!=\! (\mathcal{\overline{S}}, \mathcal{\overline{A}}, \overline{\tau}_{\overline{a}}, \overline{R}, \overline{\gamma})$, and let $\overline{\pi} : \mathcal{\overline{S}} \to \Dist(\mathcal{\overline{A}})$ be a policy on $\mathcal{\overline{M}}$.  Then $\overline{\pi}$ lifted to $\mathcal{M}$ is a policy $\pi^\uparrow : \mathcal{S} \to \Dist(\mathcal{A})$ such that for any $(s,a) \in \mathcal{S} \times \mathcal{A}$, we have:
$$
    \pi^\uparrow(a | s) = \frac{\overline{\pi}(g_s(a) | f(s))}{|g_s^{-1}(\{g_s(a)\})|}.
$$
\end{definition}
Note that for these results to hold, it suffices for the lifted policy to satisfy:
\begin{equation}
\label{eq:liftedpushforward_finite}
   \sum_{a \in g_s^{-1}(\{g_s(a)\})} \pi^\uparrow(a | s) = \overline{\pi}(g_s(a) | f(s)) \quad \forall \; s \in \mathcal{S} 
\end{equation}
but in order to make the lifted policy unique, \citet{ravindran2001symmetries} choose to uniformly spread the probability of taking $g_s(a)$ from $f(s)$ across all actions $a'$ satisfying $g_s(a) = g_s(a')$.  We have the following result that the lifted policy of the optimal policy of $\mathcal{\overline{M}}$ is an optimal policy for $\mathcal{M}$:
\begin{theorem}[\citet{ravindran2001symmetries}]
Let $h$ be an MDP homomorphism from $\mathcal{M} \!=\! (\mathcal{S}, \mathcal{A}, \tau_a, R, \gamma)$ to $\mathcal{\overline{M}} \!=\! (\mathcal{\overline{S}}, \mathcal{\overline{A}}, \overline{\tau}_{\overline{a}}, \overline{R}, \overline{\gamma})$, and let $\overline{\pi}^* \!:\! \mathcal{\overline{S}} \!\to\! \Dist(\mathcal{\overline{A}})$ be an optimal policy on $\mathcal{\overline{M}}$.  Then the lifted policy $\pi^* : \mathcal{S} \to \Dist(\mathcal{A})$ is an optimal policy for $\mathcal{M}$.
\end{theorem}
Furthermore, \citet{rezaei2022continuous} show that given this definition of a lifted policy, we have a \textit{value equivalence} result, showing that all value functions---not just the optimal one--- are preserved under the MDP homomorphism mapping.

\begin{theorem}[\citet{rezaei2022continuous}]
Let $h$ be an MDP homomorphism from $\mathcal{M} = (\mathcal{S}, \mathcal{A}, \tau_a, R, \gamma)$ to $\mathcal{\overline{M}} = (\mathcal{\overline{S}}, \mathcal{\overline{A}}, \overline{\tau}_{\overline{a}}, \overline{R}, \overline{\gamma})$.  Then for any $(s,a) \in \mathcal{S} \times \mathcal{A}$, abstract policy $\overline{\pi} : \mathcal{\overline{S}} \to \Dist(\mathcal{\overline{A}})$, and its lifted policy $\pi^\uparrow : \mathcal{S} \to \Dist(\mathcal{A})$, we have:
$$
    Q^{\pi^\uparrow}(s,a) = \overline{Q}^{\overline{\pi}}(f(s), g_s(a)).
$$
\end{theorem}

\section{Continuous MDP Homomorphisms}
\label{sec:cont_mdp_homomorphism}
\label{sec:cont_MDP_homomorphism}
Our introduction of MDP homomorphisms in the previous section was strictly applicable where the state and action spaces were finite. In this section, we will formalize MDP homomorphisms for general continuous domains. First, we define continuous MDPs and state our underlying assumptions, which require care regarding measurability and differentiability of spaces.

\begin{definition}[Continuous MDP]
    \label{def:cont_mdp}
  A \emph{continuous Markov decision process (MDP)} is a $6$-tuple:
  \[\mathcal{M} = (\mathcal{S},\Sigma,\mathcal{A},\forall a\in \mathcal{A}\;\; \tau_a:\mathcal{S}\x\Sigma\to[0,1],R:\mathcal{S}\x \mathcal{A}\to \R, \gamma)\] where $\mathcal{S}$, the state space is assumed to be an appropriate topological space, $\Sigma$ is a $\sigma$-algebra on $\mathcal{S}$\footnote{Usually the Borel algebra.}, $\mathcal{A}$, the space of \emph{actions}, is a
  locally compact metric space, usually taken to be a subset of $\R^n$, $\tau_a$ is
  the transition probability kernel for each possible action $a$, for each fixed $s$,
  $\tau_a(\cdot|s)$ is a probability distribution on $\Sigma$ while $R$ is the reward
  function, and $\gamma$ is the discount factor. Furthermore, for all $s \in \mathcal{S}$ and $B \in \Sigma$ the map $a \mapsto \tau_a(B|s)$ is smooth.
\end{definition}
The last assumption is required for differentiability with respect to actions $a$, which is needed in Section \ref{sec:homomorphic_pg} for deriving the homomorphic policy gradient theorem. 

Probability theory on continuous spaces works well when the underlying space is assumed to be Polish (see Appendix \ref{supp:math_tools} for definitions) but we do not need the properties of Polish spaces for our results.  The assumption on the action space is needed for the proof that policies can be lifted; it is possible that this could be proved with different assumptions but locally compact metric spaces are general enough to cover any example we have seen.

Next we will define continuous MDP homomorphisms and establish results for both optimal value equivalence and value equivalence.
\begin{definition}[Continuous MDP Homomorphism]
\label{def:cont_mdp_homo}
    A \emph{continuous MDP homomorphism} is a map $h = ( f, g_s ): \mathcal{M} \to \overline{\mathcal{M}}$ where $f: \mathcal{S} \to \overline{\mathcal{S}}$ and for every $s \in \mathcal{S}$, $g_s: \mathcal{A} \to \overline{\mathcal{A}}$ are measurable, surjective maps such that the following hold:
    \begin{align}
        \text{Invariance of reward: }& \overline{R}(f(s), g_s(a)) = R(s,a) \qquad \forall s \in \mathcal{S}, a \in \mathcal{A} \\
        \text{Equivariance of transitions: }& \overline{\tau}_{g_s(a)}(\overline{B}| f(s)) = \tau_a(f^{-1}(\overline{B})| s) \qquad \forall \; s \in \mathcal{S}, a \in \mathcal{A}, \overline{B} \in \overline{\Sigma}
    \end{align}
\end{definition}
Note that if $g_s$ is the identity map, the second condition reduces to $\overline{\tau}_a(\overline{B} | f(s)) \!=\! \tau_a(f^{-1}(\overline{B}) | s)$ which is simply the condition for preservation of transition probabilities as used in bisimulation \citep{Desharnais02}.

The condition on the rewards translates directly from the finite case. The equivariance of transitions is defined using the $\sigma$-algebra defined on the image MDP; it states that the measure $\overline{\tau}_{g_s(a)}(\cdot| f(s))$ is the pushforward measure of $\tau_a(\cdot | s)$ under the state mapping $f$. In the results to follow, we will see the reason we require this condition.

\subsection{Optimal Value Equivalence}

In this continuous setting, we will show that optimal value equivalence still holds. The proof is similar to Theorem~\ref{thm:optvalue_discrete}, however, we utilize the change of variables formula (see Theorem \ref{thm:cov} in Appendix \ref{supp:math_tools}) to change the domain of integration in the continuous Bellman equation instead of re-indexing the summation.

\begin{theorem}[Optimal Value Equivalence]
\label{thm:opt_continuous}
Let $\overline{\mathcal{M}} = (\overline{\mathcal{S}}, \overline{\Sigma}, \overline{\mathcal{A}}, \overline{\tau}_{\overline{a}}, \overline{R}, \overline{\gamma})$ be the image of a continuous MDP homomorphism $h = (f, g_s)$ from $\mathcal{M} = (\mathcal{S}, \Sigma, \mathcal{A}, \tau_a, R, \gamma)$. Then for any $(s, a) \in \mathcal{S} \times \mathcal{A}$ we have:
$$Q^*(s,a) = \overline{Q}^*(f(s), g_s(a)),$$
where $Q^*,  \overline{Q}^*$ are the optimal action-value functions for $\mathcal{M}$ and $\overline{\mathcal{M}}$, respectively.
\end{theorem}
\begin{proof}
We will first prove the following claim:
\begin{claim}
    For $m \geq 1$, define the sequence $Q_m: \mathcal{S} \times \mathcal{A} \to \R$ as:
    $$Q_m(s,a) = R(s,a) + \gamma \int_{s' \in \mathcal{S}}\tau_a(ds' | s) \sup_{a' \in \mathcal{A}}Q_{m-1}(s', a')$$
    and $Q_0(s,a) = R(s,a)$. Define the sequence $\overline{Q}_m: \overline{\mathcal{S}} \times \overline{\mathcal{A}} \to \R$ analogously. Then for any $(s, a) \in \mathcal{S} \times \mathcal{A}$ we claim:
    $$Q_m(s,a) = \overline{Q}_m(f(s), g_s(a)).$$
\end{claim}
We will prove this claim by induction on $m$. The base case $m= 0$ follows from the reward invariance property of continuous MDP homomorphisms:
$$Q_0(s,a) = R(s,a) = \overline{R}(f(s), g_s(a)) = \overline{Q}_0(f(s), g_s(a)).$$
For the inductive case, note that:
\begin{align}
    Q_m(s,a) &= R(s,a) + \gamma \int_{s' \in \mathcal{S}}\tau_a(ds' | s) \sup_{a' \in \mathcal{A}}Q_{m-1}(s', a') \\
    &= \overline{R}(f(s), g_s(a)) + \gamma \int_{s' \in \mathcal{S}}\tau_a(ds' | s) \sup_{a' \in \mathcal{A}}\overline{Q}_{m-1}(f(s'), g_{s'}(a')) \label{thm:opt2}\\
    &= \overline{R}(f(s), g_s(a)) + \gamma \int_{s' \in \mathcal{S}}\tau_a(ds' | s) \sup_{a' \in \overline{\mathcal{A}}}\overline{Q}_{m-1}(f(s'), a') \label{thm:opt3}\\
    &= \overline{R}(f(s), g_s(a)) + \gamma \int_{s' \in \overline{\mathcal{S}}}\overline{\tau}_{g_s(a)}(ds' | f(s)) \sup_{a' \in \overline{\mathcal{A}}}\overline{Q}_{m-1}(s', a') \label{thm:opt4}\\
    &= \overline{Q_{m-1}}(f(s), g_s(a)),
\end{align}
where Equation~\ref{thm:opt2} follows from the inductive hypothesis, Equation~\ref{thm:opt3} follows from $g_s$ being surjective, and Equation~\ref{thm:opt4} follows from the change of variables formula; indeed, from Definition~\ref{def:cont_mdp_homo} we have the pushforward measure of $\tau_a(\cdot | s)$ with respect to $f$ equals $\overline{\tau}_{g_s(a)}(\cdot | f(s))$ and here we are integrating a function from $\mathcal{S} \to \R$ defined as $s' \mapsto \sup_{a' \in \mathcal{A}} Q_{m-1}(s', a')$. This concludes the induction proof. Since $\lim_{m \to \infty} Q_{m}(s,a) = Q^*(s,a)$, it follows that $Q^*(s,a) = \overline{{Q^*}}(f(s), g_s(a))$.
\end{proof}

\subsection{Lifting Policies and Value Equivalence}
\label{sec:policy_lifting_value_eqv}
Recall that in the finite setting, we had an exact equation defining lifted policies via an MDP homomorphism. In the continuous case, finding a lifted policy that exists in general and that also gives a value equivalence result is not trivial. We will use the following condition to define a lifted policy for continuous MDP homomorphisms.

\begin{definition}[Policy Lifting]
\label{def:liftedpolicy}
Let $\overline{\mathcal{M}} = (\overline{\mathcal{S}}, \overline{\Sigma}, \overline{\mathcal{A}}, \overline{\tau}_{\overline{a}}, \overline{R}, \overline{\gamma})$ be the image of a continuous MDP homomorphism $h = (f, g_s)$ from $\mathcal{M} = (\mathcal{S}, \Sigma, \mathcal{A}, \tau_a, R, \gamma)$. Then for any policy $\overline{\pi} : \overline{\mathcal{S}} \to \Dist(\overline{\mathcal{A}})$ defined on $\overline{\mathcal{M}}$, a policy $\pi^\uparrow : S \to \Dist(\mathcal{A})$ on $\mathcal{M}$ is a \emph{lifted} policy of $\overline{\pi}$ if:
\begin{equation}
    \pi^\uparrow(g_s^{-1}(\beta) | s) = \overline{\pi}(\beta | f(s))
    \label{eq:policy_lifting}
\end{equation}
for every Borel set $\beta \subset \overline{\mathcal{A}}$ and $s \in S$. In other words, $\overline{\pi}(f(s), \cdot)$ is the pushforward measure of  $\pi^\uparrow(s, \cdot)$ for all $s \in S$ with respect to $g_s$.
\end{definition}
Note that Definition~\ref{def:liftedpolicy} does not define a measure, since we need to specify a value assigned to $\pi^\uparrow(s, B)$ for \emph{all} Borel sets $B$ in $\mathcal{A}$, not just those arising as inverse images $g_s^{-1}(\beta)$. However, naively defining:
$$\pi^\uparrow(B | s) = \overline{\pi}(g_s(B) | f(s))$$
poses immediate issues because $g_s$ does not map Borel sets to Borel sets and $B \subsetneq g_s^{-1}(g_s(B))$ in general. In other words, we could only use this definition if $g_s$ is bijective and maps measurable sets to measurable sets. However, as shown in the next result, such a measure satisfying the condition in Definition~\ref{def:liftedpolicy} indeed exists in general, assuming $\mathcal{A}$ and $\overline{\mathcal{A}}$ are locally compact metric spaces. The proof uses results in functional analysis, specifically the Hahn-Banach and Riesz Representation theorem. Notably, the bijection assumption of $g_s$ is one of the limitations of the prior work of \citet{rezaei2022continuous}, which is removed in our paper.

\begin{proposition}
\label{prop:lift_existence}
    Let $\overline{\mathcal{M}} = (\overline{\mathcal{S}}, \overline{\Sigma}, \overline{\mathcal{A}}, \overline{\tau}_{\overline{a}}, \overline{R}, \overline{\gamma})$ be the image of a continuous MDP homomorphism $h = (f, g_s)$ from $\mathcal{M} = (\mathcal{S}, \Sigma, \mathcal{A}, \tau_a, R, \gamma)$, where $\mathcal{A}$ and $\overline{\mathcal{A}}$ are locally compact metric spaces. Then for any policy $\overline{\pi} : \overline{\mathcal{S}} \to \Dist(\overline{\mathcal{A}})$ defined on $\overline{\mathcal{M}}$, there exists a lifted policy $\pi^\uparrow: S \to \Dist(\mathcal{A})$ in the sense of Definition~\ref{def:liftedpolicy}.
\end{proposition}
\begin{proof}
Define the functional $p: C_0(\mathcal{A}) \to \R$ as:
$$p(\psi) = \max_{a \in \mathcal{A}} \psi(a).$$
Clearly $p(\varphi + \psi) \leq p(\varphi) + p(\psi)$ and $p(\alpha\psi) = \alpha p(\psi)$ for every $\psi, \varphi \in C_0(\mathcal{A})$ and $0 < \alpha < \infty$. Indeed, $p$ is a semi-norm. Note that since $g_s$ is surjective, we can define the subspace $U := \{\eta \circ g_s : \eta \in C_0(\overline{\mathcal{A}})\} \subset C_0(\mathcal{A})$. Let $\rho$ be the linear functional on $U$ defined as
$$\rho(\eta \circ g_s) = \int_{a' \in \overline{\mathcal{A}}} \eta(a') \overline{\pi}(da' | f(s)).$$
We have:
$$\rho(\eta \circ g_s) \leq \overline{\pi}(f(s), \overline{\mathcal{A}}) \max_{a' \in \overline{\mathcal{A}}} \eta(a') = \max_{a \in \mathcal{A}} (\eta \circ g_s)(a) = p(\eta \circ g_s),$$
since $\overline{\pi}(\cdot | f(s))$ is a probability measure and $g_s$ is surjective. By the Hahn-Banach theorem, we can extend $\rho$ to a linear functional $\hat{\rho}$ on $C_0(\mathcal{A})$ where $\hat{\rho}(\psi) \leq p(\psi)$ for every $\psi \in C_0(\mathcal{A})$. It follows that if $\psi \leq 0$ then $\hat{\rho}(\psi) \leq 0$, whence if $\psi \geq 0$ then $\hat{\rho}(\psi) = -\hat{\rho}(-\psi) \geq 0$. Since this implies that $\hat{\rho}$ is a positive linear functional and $\mathcal{A}$ is a locally compact metric space, by the Riesz Representation theorem there is a unique Radon measure $\mu$ on $\mathcal{A}$ such that:
$$\hat{\rho}(\psi) = \int_{a \in \mathcal{A}}\psi(a) d\mu(a).$$
It follows that for every $\eta \in C_0(\overline{\mathcal{A}})$:
$$\int_{a \in \overline{\mathcal{A}}} \eta(a') \overline{\pi}(da' | f(s)) = \rho(\eta \circ g_s) = \int_{a \in \mathcal{A}} (\eta \circ g_s)(a) d\mu(a) = \int_{a' \in \overline{\mathcal{A}}} \eta(a') d{g_s}_*\mu(da'),$$
where the first equality is by definition of $\rho$, the second equality follows from $\mu$ extending $\rho$, and the last equality following by the change of variables formula. Thus $\overline{\pi}(\cdot | f(s))$ is the pushforward measure of $\mu$ with respect to $g_s$. Setting $\pi^\uparrow(\cdot | s) = \mu$ gives the result.
\end{proof}

Recall that finding a lifted policy reduces to the following question: given a surjective measurable function $g_s: \mathcal{A} \to \bar{\mathcal{A}}$ and a probability measure $\bar{\pi}$ on $\bar{\mathcal{A}}$, does there exist a measure $\pi^\uparrow$ on $\mathcal{A}$ such that the resulting pushforward measure ${g_s}_*\pi^\uparrow = \bar{\pi}$? This is a result that holds more generally for analytic subsets of Polish spaces, the original result proven in \citet{varadarajan1963groups} (see Lemma 2.2).

Now that we have proven a lifted policy exists for continuous setting, we proceed to prove a value equivalence result for continuous MDP homomorphisms. The proof is very similar to optimal value equivalence, and in fact only requires one more application of change of variables with respect to the lifted policy.

\begin{theorem}[Value Equivalence]
\label{thm:value_equivalence}
Let $\overline{\mathcal{M}} = (\overline{\mathcal{S}}, \overline{\Sigma}, \overline{\mathcal{A}}, \overline{\tau}_{\overline{a}}, \overline{R}, \overline{\gamma})$ be the image of a continuous MDP homomorphism $h = (f, g_s)$ from $\mathcal{M} = (\mathcal{S}, \Sigma, \mathcal{A}, \tau_a, R, \gamma)$, and let $\pi^\uparrow$ be a lifted policy corresponding to $\overline{\pi}$. Then for any $(s, a) \in \mathcal{S} \times \mathcal{A}$ we have:
$$
    Q^{\pi^\uparrow}(s, a) = \overline{Q}^{\overline{\pi}}(f(s), g_s(a)),
$$
where $Q^{\pi^\uparrow}(s, a)$ and $\overline{Q}^{\overline{\pi}}(f(s), g_s(a))$ are the action-value functions for policies $\pi^\uparrow$ and $\overline{\pi}$ respectively.
\end{theorem}
\begin{proof}
    Similarly as in Theorem~\ref{thm:opt_continuous}, we define the sequence $Q^{\pi^\uparrow}_m: \mathcal{S} \times \mathcal{A} \to \R$ as:
    $$Q^{\pi^\uparrow}_m(s,a) = R(s,a) + \gamma\int_{s' \in \mathcal{S}} \tau_a(ds' | s)\int_{a' \in \mathcal{A}} \pi^\uparrow(da' | s') Q^{\pi^\uparrow}_{m-1}(s', a')$$
    for $m \geq 1$ and $Q^{\pi^\uparrow}_0(s,a) = 0$. Analogously define $\overline{Q}_{m-1}^{\overline{\pi}}: \overline{\mathcal{S}} \times \overline{\mathcal{A}} \to \R$. For the inductive case, we can perform change of variables twice to change the domain of integration from $S$ to $\overline{\mathcal{S}}$ and $\mathcal{A}$ to $\overline{\mathcal{A}}$ respectively:
    \begin{align}
        Q^{\pi^\uparrow}_m(s,a) &= R(s,a) + \gamma\int_{s' \in \mathcal{S}} \tau_a(ds' | s)\int_{a' \in \mathcal{A}} \pi^\uparrow(da' | s') Q^{\pi^\uparrow}_{m-1}(s', a') \\
        &= \overline{R}(f(s), g_s(a)) + \gamma \int_{s' \in \mathcal{S}}\tau_a(ds'|s) \int_{a' \in \mathcal{A}} \pi^\uparrow(da' | s') \overline{Q}_{m-1}^{\overline{\pi}}(f(s'), g_s(a')) \\
        &= \overline{R}(f(s), g_s(a)) + \gamma \int_{s' \in \mathcal{S}}\tau_a(ds' | s) \int_{\overline{a} \in \overline{\mathcal{A}}}\overline{\pi}(d\overline{a} | f(s')) \overline{Q}^{\overline{\pi}}_{m-1}(f(s'), \overline{a}) \\
        &= \overline{R}(f(s), g_s(a)) + \gamma \int_{\overline{s} \in \overline{\mathcal{S}}}\overline{\tau}_{g_s(a)}(d\overline{s} | f(s)) \int_{\overline{a} \in \overline{\mathcal{A}}}\overline{\pi}(d\overline{a} | \overline{s}) \overline{Q}^{\overline{\pi}}_{m-1}(\overline{s}, \overline{a}) \\
        &= \overline{Q}_{m-1}^{\overline{\pi}}(f(s), g_s(a)).
    \end{align}
    In a similar manner to Theorem~\ref{thm:opt_continuous}, we conclude that $ Q^{\pi^\uparrow}(s, a) = Q^{\overline{\pi}}(f(s), g_s(a))$.
\end{proof}

Theorem \ref{thm:value_equivalence} posits that the value function of any policy on the reduced MDP equals the value function of its corresponding lifted policy on the original MDP. Since this holds true for any optimal policy, it follows from Theorem \ref{thm:opt_continuous} that a lifted optimal policy is optimal for the original MDP. Thus, we have recovered all desirable properties for continuous MDP homomorphisms from the finite case.

\section{Homomorphic Policy Gradient}
\label{sec:homomorphic_pg}
\label{sec:hom_policy_gradient_thm}
In order to directly integrate the notion of MDP homomorphisms into policy gradients and incorporate their state-action abstraction as an inductive bias for policy optimization, we derive the \emph{homomorphic policy gradient} (HPG) theorem. Notably, our results are derived for continuous states and actions and hold for both stochastic and deterministic policies; this is in contrast to the prior work of \citet{rezaei2022continuous} in which the derivation of the homomorphic policy gradient theorem is limited to deterministic policies.

In this section, we assume the policy is parameterized by differentiable functions (e.g., neural networks) and the MDP homomorphic image can be obtained through a parameterized homomorphism map. Importantly, learning such parameterized MDP homomorphism map is detailed in Section \ref{sec:homomorphic_ac}. Finally, following  the prior works on policy gradient methods \citep{sutton2000policy, silver2014deterministic}, we define the performance measure on the actual MDP as described in equation \eqref{eq:perf_measure} .

Since the derivation of the policy gradient theorem for stochastic and deterministic policies are substantially different and require distinct steps and assumptions, in the remainder of this section, we derive the homomorphic policy gradient theorem for stochastic and deterministic policies independently from one another.

\subsection{Stochastic HPG Theorem}
\label{sec:stochastic_pg}
The stochastic HPG theorem can be derived with the underlying assumptions of continuous MDP homomorphisms, as in Definition \ref{def:cont_mdp_homo}, and the regularity conditions described in Appendix \ref{sec:assumptions}. Notably, the only requirement on the MDP homomorphism map is that $f : \mathcal{S} \to \overline{\mathcal{S}}$ and $g_s : \mathcal{A} \to \overline{\mathcal{A}}$ are measurable, surjective maps adhering to the invariance of reward and equivariance of transitions in Definition \ref{def:cont_mdp_homo}. This is in contrast to the deterministic HPG theorem which poses further restrictions on $g_s$.
\begin{theorem}[Stochastic Homomorphic Policy Gradient]
\label{thm:stochastic_hpg}
Let $\overline{\mathcal{M}} = (\overline{\mathcal{S}}, \overline{\Sigma}, \overline{\mathcal{A}}, \overline{\tau}, \overline{R} , \overline{\gamma})$ be the image of a continuous MDP homomorphism $h = ( f, g_s )$ from $\mathcal{M} = (\mathcal{S}, \Sigma, \mathcal{A}, \tau, R, \gamma)$, and let $\overline{\pi}_\theta : \overline{\mathcal{S}} \to \Dist(\overline{\mathcal{A}})$ be a stochastic policy defined on $\overline{\mathcal{M}}$. Then the gradient of the performance measure $J(\theta)$ w.r.t. $\theta$ is: 
$$
    \nabla_\theta J(\theta) = \int_{\overline{s} \in \overline{\mathcal{S}}} \rho^{\overline{\pi}_\theta}(\overline{s})  \int_{\overline{a} \in \overline{\mathcal{A}}} \overline{Q}^{\overline{\pi}_\theta}(\overline{s}, \overline{a}) \nabla_\theta \overline{\pi}_\theta(d\overline{a} | \overline{s}) d\overline{s},
$$
where $\rho^{\overline{\pi}_\theta}(\overline{s})$ is the discounted state distribution of $\overline{\mathcal{M}}$ following the stochastic policy $\overline{\pi}_\theta(\overline{a} | \overline{s})$.
\end{theorem}

\begin{proof}
The proof follows along the same lines of the stochastic policy gradient theorem  \citep{sutton2000policy}. First, we derive a recursive expression for $\nabla_\theta V^{{\pi^\uparrow_\theta}}(s)$ as:
\begingroup
\allowdisplaybreaks
\begin{align}
    \nabla_\theta V^{\pi^\uparrow_\theta} (s) &= \nabla_\theta \int_{a \in \mathcal{A}} \pi_\theta^\uparrow(da | s) Q^{\pi^\uparrow_\theta}(s, a) \nonumber \\
    &= \int_{a \in \mathcal{A}} \Big[\nabla_\theta \pi^\uparrow_\theta(da | s) Q^{\pi^\uparrow_\theta}(s, a) + \pi_\theta^\uparrow(da | s) \nabla_\theta Q^{\pi^\uparrow_\theta}(s, a) \Big] \nonumber \\
    &= \int_{a \in \mathcal{A}} \Big[\nabla_\theta \pi^\uparrow_\theta(da | s) Q^{\pi^\uparrow_\theta}(s, a) + \pi_\theta^\uparrow(da | s) \nabla_\theta \Big( R(s, a) + \gamma \int_{s' \in \mathcal{S}} \tau_a(ds' | s) V^{\pi^\uparrow_\theta}(s')  \Big)  \Big] \nonumber \\
    &= \int_{a \in \mathcal{A}} \nabla_\theta \pi^\uparrow_\theta(da | s) Q^{\pi^\uparrow_\theta}(s, a) + \gamma \int_{a \in \mathcal{A}} \pi_\theta^\uparrow(da | s) \int_{s' \in \mathcal{S}} \tau_a(ds' | s) \nabla_\theta V^{\pi^\uparrow_\theta}(s') \nonumber \\
    &= \int_{a \in \mathcal{A}} \nabla_\theta \pi^\uparrow_\theta(da | s) \overline{Q}^{\overline{\pi}_\theta}(f(s), g_s(a)) + \gamma \int_{a \in \mathcal{A}} \pi_\theta^\uparrow(da | s) \int_{s' \in \mathcal{S}} \tau_a(ds' | s) \nabla_\theta V^{\overline{\pi}_\theta}(f(s')) \label{eq:stochastic_hpg_5}\\
    &= \int_{a \in \mathcal{A}} \nabla_\theta \pi^\uparrow_\theta(da | s) \overline{Q}^{\overline{\pi}_\theta}(f(s), g_s(a)) + \gamma \int_{a \in \mathcal{A}} \pi_\theta^\uparrow(da | s) \int_{\overline{s} \in \overline{\mathcal{S}}} \overline{\tau}_{g_s(a)}(d\overline{s} | f(s)) \nabla_\theta V^{\overline{\pi}_\theta}(\overline{s}) \label{eq:stochastic_hpg_6} \\
    &= \int_{\overline{a} \in \overline{\mathcal{A}}}\nabla_\theta \overline{\pi}_\theta(d\overline{a} | f(s))\overline{Q}^{\overline{\pi}_\theta}(f(s), \overline{a}) + \gamma \int_{\overline{a} \in \overline{\mathcal{A}}} \overline{\pi}_\theta( d\overline{a} | f(s))\int_{\overline{s} \in \overline{\mathcal{S}}} \overline{\tau}_{\overline{a}}( d\overline{s} | f(s))\nabla_\theta V^{\overline{\pi}_\theta}(\overline{s}) \nonumber.
\end{align}
\endgroup
Here we apply value equivalence in equation~\eqref{eq:stochastic_hpg_5}, a change of variables from $\mathcal{S}$ to $\overline{\mathcal{S}}$ over $\tau_a(\cdot | s)$ and $\overline{\tau}_{g_s(a)}(\cdot | f(s))$ respectively in equation~\eqref{eq:stochastic_hpg_6}, and a change of variables from $\mathcal{A}$ to $\overline{\mathcal{A}}$ over the lifted policy and the policy over the abstract MDP respectively. Note that here, some care may be necessary to rigorously verify the interchanging of the gradient over $\theta$ and the integral over $\mathcal{A}$; however, this is a necessary condition to prove any type of policy gradient result on continuous domains, not specifically to the stochastic HPG theorem.

As in the proof of the stochastic policy gradient theorem, we can continue to roll out the definition of $\nabla_\theta V^{\overline{\pi}_\theta}(\overline{s})$ in the space of the \emph{abstract MDP} $\overline{\mathcal{M}}$. Denoting $\overline{\pi}^k(\cdot | f(s))$ to be the probability distribution over $\overline{\mathcal{S}}$ taking $k$ steps following $\overline{\pi}$ from state $f(s)$, we have:
\begin{align}
    \nabla_\theta V^{\pi^\uparrow_\theta} (s) &= \sum_{k=0}^\infty \gamma^k \int_{\overline{s} \in \overline{\mathcal{S}}} \overline{\pi}_\theta^k( d\overline{s} | f(s))\int_{\overline{a} \in \overline{\mathcal{A}}}\nabla_\theta \overline{\pi}_\theta( d\overline{a} | f(s))\overline{Q}^{\overline{\pi}_\theta}(f(s), \overline{a}). \nonumber
\end{align}
Finally, we conclude that:
\begin{align*}
    \nabla_\theta J(\theta) &= \nabla_\theta V^{\pi^\uparrow_\theta} (s) \\
    &= \int_{\overline{s} \in \overline{\mathcal{S}}} \sum_{k=0}^\infty \gamma^k \overline{\pi}_\theta^k( d\overline{s} | f(s))\int_{\overline{a} \in \overline{\mathcal{A}}}\nabla_\theta \overline{\pi}_\theta(d\overline{a} | f(s))\overline{Q}^{\overline{\pi}_\theta}(f(s), \overline{a}) \\
    &= \int_{\overline{s} \in \overline{\mathcal{S}}} \rho^{\overline{\pi}_\theta}(d\overline{s})  \int_{\overline{a} \in \overline{\mathcal{A}}} \overline{Q}^{\overline{\pi}_\theta}(\overline{s}, \overline{a}) \nabla_\theta \overline{\pi}_\theta(d\overline{a} | \overline{s}),
\end{align*}
as desired, where $\rho^{\bar{\pi}_\theta}(\bar{s})$ is the discounted stationary distribution induced by the policy $\bar{\pi}_\theta$.
\end{proof}

\subsection{Deterministic HPG Theorem}
\label{sec:deterministic_pg}
In contrast to stochastic HPG where the homomorphism map can be any measurable surjective map, the deterministic case requires the action encoder $g_s : \mathcal{A} \to \overline{\mathcal{A}}$ to be a \emph{local diffeomorphism} (see Appendix \ref{supp:math_tools} for definitions). The important implication of this requirement is that the action encoder $g_s$ needs to be locally bijective, hence the abstract action space must have the same dimensionality as the actual action space. First, we show the \emph{equivalence of policy gradients}:
\begin{theorem}[Equivalence of Deterministic Policy Gradients]
    \label{thm:det_grad_equiv}
    Let \newline $\overline{\mathcal{M}} = (\overline{\mathcal{S}}, \overline{\Sigma}, \overline{\mathcal{A}}, \overline{\tau}, \overline{R} , \overline{\gamma})$ be the image of a continuous MDP homomorphism $h = ( f, g_s )$ from $\mathcal{M} = (\mathcal{S}, \Sigma, \mathcal{A}, \tau, R, \gamma)$, and let $\pi^\uparrow_\theta : \mathcal{S} \to \mathcal{A}$ be the lifted deterministic policy corresponding to the abstract deterministic policy $\overline{\pi}_\theta : \overline{\mathcal{S}} \to \overline{\mathcal{A}}$. Then for any $(s, a) \in \mathcal{S} \times \mathcal{A}$ we have:
    $$
    \nabla_a Q^{\pi^\uparrow_\theta}(s, a) \Big|_{a = \pi^\uparrow_\theta(s)} \nabla_\theta \pi^\uparrow_\theta(s) = \nabla_{\overline{a}} \overline{Q}^{\overline{\pi}_\theta}(\overline{s}, \overline{a}) \Big|_{\overline{a} = \overline{\pi}_\theta(\overline{s})} \nabla_\theta\overline{\pi}_\theta(\overline{s}).
    $$
\end{theorem}
\begin{proof} Assuming the conditions described in Appendix \ref{sec:assumptions}, we first take the derivative of the deterministic policy lifting relation w.r.t. the policy parameters $\theta$ using the chain rule:
\begin{align}
    (g_s \circ \pi_\theta^\uparrow)(s) &= (\overline{\pi}_\theta \circ f)(s) \nonumber \\
    d(g_s \circ \pi_\theta^\uparrow)_\theta (s) &= d(\overline{\pi}_\theta \circ f)_\theta (s) \nonumber \\
    d(g_s)_{\pi_\theta^\uparrow(s)} \circ d(\pi_\theta^\uparrow)_\theta (s) &= d(\overline{\pi}_\theta \circ f)_\theta (s) \nonumber \\
    \underbrace{\nabla_a g_s(a) \big|_{a = \pi_\theta^\uparrow(s)}}_{P} \nabla_\theta \pi_\theta^\uparrow(s) &= \nabla_\theta \overline{\pi}_\theta(f(s)),
    \label{eq:grad_equiv_1}
\end{align}
where $\circ$ is the composition operator and the dimensions of the matrices are $P \in \R^{|\overline{\mathcal{A}}| \times |A|}$, $\nabla_\theta \pi_\theta^\uparrow(s) \in \R^{|A| \times |\theta|}$, and $\nabla_\theta \overline{\pi}_\theta(\overline{s}) \in \R^{|\overline{\mathcal{A}}| \times |\theta|}$. Second, we take the derivative of the value equivalence theorem w.r.t. the actions $a$ using the chain rule: 
\begin{align}
    Q^{\pi_\theta^\uparrow}(s, a) &= \overline{Q}^{\overline{\pi}_\theta} (f(s), g_s(a)) \nonumber \\
    dQ^{\pi_\theta^\uparrow}(s, a)_a &= d\overline{Q}^{\overline{\pi}_\theta}(f(s), g_s(a))_a \nonumber \\
    \nabla_a Q^{\pi_\theta^\uparrow}(s, a) \big|_{a = \pi_\theta^\uparrow(s)} &= \nabla_{\overline{a}} \overline{Q}^{\overline{\pi}_\theta}(f(s), \overline{a}) \big|_{\overline{a} = \overline{\pi}_\theta(f(s))} \underbrace{\nabla_a g_s(a) \Big|_{a = g_s^{-1}(\overline{\pi}_\theta(f(s)))}}_{P},
    \label{eq:grad_equiv_2}
\end{align}
where the dimensions of the matrices are $\nabla_a Q^{\pi_\theta^\uparrow}(s, a) \in \R^{|A|}$, $\nabla_{\overline{a}} \overline{Q}^{\overline{\pi}_\theta}(\overline{s}, \overline{a}) \in \R^{|\overline{\mathcal{A}}|}$, and similarly as before $P \in \R^{|\overline{\mathcal{A}}| \times |A|}$. As we assumed the $g_s$ to be a local diffeomorphism, the inverse function theorem states that the matrix $P$ is invertible, thus we right-multiply both sides of equation \eqref{eq:grad_equiv_2} by $P^{-1}$ and left-multiply the resulting equation by equation \eqref{eq:grad_equiv_1} to obtain the desired result:
\begin{align}
    \nabla_a Q^{\pi_\theta^\uparrow}(s, a) \big|_{a = \pi_\theta^\uparrow(s)} P^{-1} P  \nabla_\theta \pi_\theta^\uparrow(s) &= \nabla_{\overline{a}} \overline{Q}^{\overline{\pi}_\theta}(f(s), \overline{a}) \big|_{\overline{a} = \overline{\pi}(f(s))} \nabla_\theta \overline{\pi}_\theta(f(s)) \nonumber \\ 
    \nabla_a Q^{\pi_\theta^\uparrow}(s, a) \big|_{a = \pi_\theta^\uparrow(s)} \nabla_\theta \pi_\theta^\uparrow(s) &= \nabla_{\overline{a}} \overline{Q}^{\overline{\pi}_\theta}(f(s), \overline{a}) \big|_{\overline{a} = \overline{\pi}_\theta(f(s))} \nabla_\theta \overline{\pi}_\theta(f(s)). \nonumber
\end{align}
\end{proof}

Theorem \ref{thm:det_grad_equiv} highlights that the gradient of the abstract MDP is equivalent to that of the original, despite the underlying spaces being abstracted. This implies that performing HPG on the abstract MDP is equivalent to performing DPG on the actual MDP, allowing us to use them synergistically to update the same parameters $\theta$, as shown in Figure \ref{fig:hpg_diagram}. 

While one can naively use Theorem \ref{thm:det_grad_equiv} to substitute gradients of the standard DPG, theoretically this does not produce any useful results as the expectation remains estimated with respect to the stationary state distribution of the actual MDP $\mathcal{M}$ under $\pi_\theta^\uparrow(s)$. However, using properties of continuous MDP homomorphisms, we can change the integration space from $\mathcal{\mathcal{S}}$ to $\mathcal{\overline{\mathcal{S}}}$, and consequently estimate the policy gradient with respect to the stationary distribution of the abstract MDP $\mathcal{\overline{M}}$ under $\overline{\pi}_\theta(\overline{s})$:
\begin{theorem} [Deterministic Homomorphic Policy Gradient]
\label{thm:det_hpg}
Let \newline $\overline{\mathcal{M}} = (\overline{\mathcal{S}}, \overline{\Sigma}, \overline{\mathcal{A}}, \overline{\tau}, \overline{R} , \overline{\gamma})$ be the image of a continuous MDP homomorphism $h = ( f, g_s )$ from $\mathcal{M} = (\mathcal{S}, \Sigma, \mathcal{A}, \tau, R, \gamma)$, and let $\overline{\pi}_\theta : \overline{\mathcal{S}} \to \overline{\mathcal{A}}$ be a deterministic abstract policy defined on $\overline{\mathcal{M}}$. Then the gradient of the performance measure $J(\theta)$, defined on the actual MDP $\mathcal{M}$, w.r.t. $\theta$ is:  
$$
    \nabla_\theta J(\theta) = \int_{\overline{s} \in \overline{\mathcal{S}}} \rho^{\overline{\pi}_\theta}(\overline{s})  \nabla_{\overline{a}} \overline{Q}^{\overline{\pi}_\theta}(\overline{s}, \overline{a}) \Big|_{\overline{a} = \overline{\pi}_\theta(\overline{s})} \nabla_\theta \overline{\pi}_\theta(\overline{s}) d\overline{s},
$$
where $\rho^{\overline{\pi}_\theta}(\overline{s})$ is the discounted state distribution of $\overline{\mathcal{M}}$ following the deterministic policy $\overline{\pi}_\theta(\overline{s})$.
\end{theorem}

\begin{proof}
The proof follows along the same lines of the deterministic policy gradient theorem \citep{silver2014deterministic}, but with additional steps for changing the integration space from $\mathcal{S}$ to $\mathcal{\overline{S}}$. First, we derive a recursive expression for $\nabla_\theta V^{{\pi^\uparrow_\theta}}(s)$ as:
\begingroup
\allowdisplaybreaks
\begin{align}
    \nabla_\theta V^{\pi^\uparrow_\theta} (s) &= \nabla_\theta Q^{\pi^\uparrow_\theta}\Big(s, \pi^\uparrow_\theta(s) \Big) \nonumber \\
        &= \nabla_\theta \Big[R (s, \pi^\uparrow_\theta(s)) + \gamma \int_{s' \in \mathcal{S}} \tau_{\pi^\uparrow_\theta(s)} (ds' | s) V^{\pi^\uparrow_\theta}(s') \Big] \nonumber \\
        &=\nabla_\theta \pi^\uparrow_\theta(s) \nabla_a R(s, a) \Big|_{a = \pi^\uparrow_\theta(s)} \nonumber \\ &\qquad + \gamma \int_{s' \in \mathcal{S}} \Big[ \tau_{\pi^\uparrow_\theta(s)} (ds' | s) \nabla_\theta V^{\pi^\uparrow_\theta}(s') + \nabla_\theta \pi^\uparrow_\theta(s) \nabla_a \tau_{a} (ds' | s) \Big|_{a = \pi^\uparrow_\theta(s)} V^{\pi^\uparrow_\theta}(s') \Big] \label{eq:hpg_eq_1} \\
        &= \nabla_\theta \pi^\uparrow_\theta(s) \nabla_a \Big[ R \big(s, a \big) \!+\! \gamma \!\! \int_{s' \in \mathcal{S}} \tau_{a} (ds' | s) V^{\pi^\uparrow_\theta}(s') \Big] \Big|_{a = \pi^\uparrow_\theta(s)} \!\!\!+\! \gamma\!\! \int_{s' \in \mathcal{S}} \tau_{\pi^\uparrow_\theta(s)} (ds' | s) \nabla_\theta V^{\pi^\uparrow_\theta}(s') \nonumber \\
        &= \nabla_\theta \pi^\uparrow_\theta(s) \nabla_a Q^{\pi^\uparrow_\theta}(s, a)\Big|_{a = \pi^\uparrow_\theta(s)} + \gamma \int_{s' \in \mathcal{S}} \tau_{\pi^\uparrow_\theta(s)} (ds' | s) \nabla_\theta V^{\overline{\pi}_\theta}(f(s')) \label{eq:hpg_eq_2} \\
        &= \nabla_\theta \pi^\uparrow_\theta(s) \nabla_a Q^{\pi^\uparrow_\theta}(s, a)\Big|_{a = \pi^\uparrow_\theta(s)} + \gamma \int_{\overline{s}' \in \overline{\mathcal{S}}} \overline{\tau}_{g_s(\pi^\uparrow_\theta(s))} (d\overline{s}' | f(s)) \nabla_\theta V^{\overline{\pi}_\theta}(f(s')) \label{eq:hpg_eq_3} \\
        &=  \nabla_\theta \overline{\pi}_\theta(f(s)) \nabla_{\overline{a}} \overline{Q}^{\overline{\pi}_\theta}(f(s), \overline{a})  \Big|_{\overline{a} = \overline{\pi}_\theta(f(s))} + \gamma \int_{\overline{s}' \in \overline{\mathcal{S}}} \overline{\tau}_{\overline{\pi}_\theta(\overline{s})}(d\overline{s}' | \overline{s}) \nabla_\theta V^{\overline{\pi}_\theta}(\overline{s}') \label{eq:hpg_eq_4} \\
        &=  \nabla_\theta \overline{\pi}_\theta(f(s)) \nabla_{\overline{a}} \overline{Q}^{\overline{\pi}_\theta}(f(s), \overline{a})  \Big|_{\overline{a} = \overline{\pi}_\theta(f(s))} + \gamma \int_{\overline{s}' \in \overline{\mathcal{S}}}  p(\overline{s} \rightarrow \overline{s}', 1, \overline{\pi}_\theta ) \nabla_\theta V^{\overline{\pi}_\theta}(\overline{s}') d\overline{s}'. \nonumber
\end{align}
Where $p(\overline{s} \rightarrow \overline{s}', t, \overline{\pi}_\theta)$ is the probability of going from $\overline{s}$ to $\overline{s}'$ under the policy $\overline{\pi}_\theta(\overline{s})$ in $t$ time steps. In equation \eqref{eq:hpg_eq_1} we were able to apply the Leibniz integral rule to exchange the order of derivative and integration because of the regularity conditions on the continuity of the functions. In equation \eqref{eq:hpg_eq_2} we used the value equivalence property, and in equation \eqref{eq:hpg_eq_3} we used the change of variables formula based on the pushforward measure of $\tau_a(\cdot | s)$ with respect to $f$. Finally, in equation \eqref{eq:hpg_eq_4} we used the equivalence of policy gradients from Theorem \ref{thm:det_grad_equiv}. By recursively rolling out the formula above, we obtain:
\begin{align}
    \nabla_\theta V^{\pi^\uparrow_\theta} (s) &= \nabla_\theta \overline{\pi}_\theta(f(s)) \nabla_{\overline{a}} \overline{Q}^{\overline{\pi}_\theta}(f(s), \overline{a}) \Big|_{\overline{a} = \overline{\pi}_\theta(f(s))} \nonumber \\ 
    &\quad + \gamma \int_{\overline{s}' \in \overline{\mathcal{S}}}  p(\overline{s} \rightarrow \overline{s}', 1, \overline{\pi}_\theta ) \nabla_\theta \overline{\pi}_\theta(f(s')) \nabla_{\overline{a}} \overline{Q}^{\overline{\pi}_\theta}(f(s'), \overline{a}) \Big|_{\overline{a} = \overline{\pi}_\theta(f(s'))}  d\overline{s}' \nonumber \\
    &\quad + \gamma^2 \int_{\overline{s}' \in \overline{\mathcal{S}}}  p(\overline{s} \rightarrow \overline{s}', 1, \overline{\pi}_\theta ) \int_{\overline{s}'' \in \overline{\mathcal{S}}}  p(\overline{s}' \rightarrow \overline{s}'', 1, \overline{\pi}_\theta ) \nabla_\theta V^{\pi^\uparrow_\theta}(f(s'')) d\overline{s}'' d\overline{s}' \nonumber \\
    &= \nabla_\theta \overline{\pi}_\theta(f(s)) \nabla_{\overline{a}} \overline{Q}^{\overline{\pi}_\theta}(f(s), \overline{a}) \Big|_{\overline{a} = \overline{\pi}_\theta(f(s))} \nonumber \\ 
    &\quad + \gamma \int_{\overline{s}' \in \overline{\mathcal{S}}}  p(\overline{s} \rightarrow \overline{s}', 1, \overline{\pi}_\theta ) \nabla_\theta \overline{\pi}_\theta(f(s')) \nabla_{\overline{a}} \overline{Q}^{\overline{\pi}_\theta}(f(s'), \overline{a}) \Big|_{\overline{a} = \overline{\pi}_\theta(f(s'))}  d\overline{s}' \nonumber \\
    &\quad + \gamma^2 \int_{\overline{s}'' \in \overline{\mathcal{S}}}  p(\overline{s} \rightarrow \overline{s}'', 2, \overline{\pi}_\theta ) \nabla_\theta V^{\overline{\pi}_\theta}(f(s'')) d\overline{s}'' \label{eq:hpg_eq_5}\\
    &\vdots \nonumber \\
    &= \int_{\overline{s}' \in \overline{\mathcal{S}}} \sum_{t=0}^\infty \gamma^t p(\overline{s} \rightarrow \overline{s}', t, \overline{\pi}_\theta ) \nabla_\theta \overline{\pi}_\theta(f(s)) \nabla_{\overline{a}} \overline{Q}^{\overline{\pi}_\theta}(f(s), \overline{a}) \Big|_{\overline{a} = \overline{\pi}_\theta(f(s))} d\overline{s}'.
\end{align}
Where in equation \eqref{eq:hpg_eq_5} we exchanged the order of integration using the Fubini's theorem that requires the boundedness of $\| \nabla_\theta V^{\overline{\pi}_\theta} (s) \|$ as described in the regularity conditions. Finally, we take the expectation of $\nabla_\theta V^{\pi^\uparrow_\theta}(s)$ over the initial state distribution:
\begin{align}
    \nabla_\theta J(\theta) &= \nabla_\theta \int_{s \in \mathcal{S}} p_1(s) V^{\pi^\uparrow_\theta}(s) ds \nonumber \\
    &= \int_{s \in \mathcal{S}} p_1(s) \nabla_\theta V^{\pi^\uparrow_\theta}(s) ds \nonumber \\
    &= \int_{s \in \mathcal{S}} p_1(s) \int_{\overline{s}' \in \overline{\mathcal{S}}} \sum_{t=0}^\infty \gamma^t p(\overline{s} \rightarrow \overline{s}', t, \overline{\pi}_\theta ) \nabla_\theta \overline{\pi}_\theta(f(s)) \nabla_{\overline{a}} \overline{Q}^{\overline{\pi}_\theta}(f(s), \overline{a}) \Big|_{\overline{a} = \overline{\pi}_\theta(f(s))} d\overline{s}' ds \nonumber \\
    &= \int_{\overline{s} \in \overline{\mathcal{S}}} \overline{p}_1(\overline{s}) \int_{\overline{s}' \in \overline{\mathcal{S}}} \sum_{t=0}^\infty \gamma^t p(\overline{s} \rightarrow \overline{s}', t, \overline{\pi}_\theta ) \nabla_\theta \overline{\pi}_\theta(f(s)) \nabla_{\overline{a}} \overline{Q}^{\overline{\pi}_\theta}(f(s), \overline{a}) \Big|_{\overline{a} = \overline{\pi}_\theta(f(s))} d\overline{s}' d\overline{s} \label{eq:hpg_eq_6} \\
    &= \int_{\overline{s} \in \overline{\mathcal{S}}} \rho^{\overline{\pi}_\theta}(\overline{s}) \nabla_\theta \overline{\pi}_\theta(\overline{s}) \nabla_{\overline{a}} \overline{Q}^{\overline{\pi}_\theta}(\overline{s}, \overline{a}) \Big|_{\overline{a} = \overline{\pi}_\theta(\overline{s})} d\overline{s}.
\end{align}
\endgroup
Where $\rho^{\overline{\pi}_\theta}(\overline{s})$ is the discounted stationary distribution induced by the policy $\overline{\pi}_\theta$. In equation \eqref{eq:hpg_eq_6} we used the change of variable formula. Similar to the steps before, we have used the Leibniz integral rule to exchange the order of integration and derivative, used Fubini's theorem to exchange the order of integration.
\end{proof}

\subsection{Comparing the Stochastic and Deterministic HPG Theorems}
\label{sec:comparison_hpg}

The significance of the homomorphic policy gradients (Theorems \ref{thm:stochastic_hpg} and \ref{thm:det_hpg}), which form the basis of our proposed homomorphic actor-critic algorithms, is twofold. First, we can get another estimate for the policy gradient based on the approximate MDP homomorphic image in addition to the standard policy gradient estimator. Although the two policy gradient estimates are not statistically independent from one another as they are tied through the homomorphism map, HPG will potentially have less variance at the expense of some bias due to the approximation of the MDP homomorphism.

Second, since the minimal image of an MDP is the MDP homomorphic image \citep{ravindran2001symmetries}, the abstract critic $\overline{Q}^{\overline{\pi}_\theta}$ is trained on a simplified problem. In other words, each abstract state-action pair $(\overline{s}, \overline{a})$ used to train $\overline{Q}^{\overline{\pi}_\theta}$ represents all $(s, a )$ pairs that are equivalent under the MDP homomorphism relation, thus improving sample efficiency. However, the amount of complexity reduction is dependent on the approximate symmetries of the environment, as also supported by our empirical results. 

Figure \ref{fig:hpg_diagram} shows the schematics of the homomorphic policy gradient theorem and its tangential use alongside the standard PG theorem. To conclude this section, we provide a conceptual comparison between the stochastic and deterministic HPG variants, following up with an empirical comparison in Section \ref{sec:experiments}.

\paragraph{Dimensionality Reduction in the Action Space.}
A key aspect of MDP homomorphisms is the notion of ``\emph{collapse}'': the state map $f$ and state-dependent action maps $g_s$ are specifically surjective. For instance, continuous symmetries of a physical system with respect to an action corresponds to an invariance of a quantity, and effectively allows for reduction in the dimensionality of the action space \citep{noether1971invariant, bluman2013symmetries}. In the context of RL agents, the ability to identify and leverage continuous symmetries of the environment results in the dimensionality reduction of the action space which in turn significantly simplifies the learning problem. However, such action reductions do not meet the conditions of the deterministic HPG theorem, as it requires the action map $g_s$ to be a \emph{local diffeomorphism}. Thus, the underlying assumptions do not account for strict collapses. In contrast, the stochastic HPG does not impose any additional structure on $g_s$, which consequently allows for effective dimensionality reduction of the action space, without risking the optimality of the policy.

\begin{figure}[t!]
    \centering
    \includegraphics[width=0.6\textwidth]{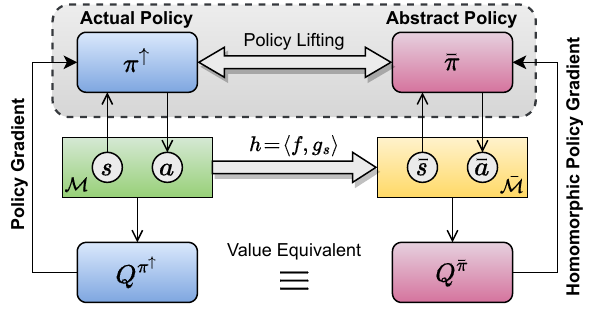}
    \caption{Schematics of HPG. The actual MDP $\mathcal{M}$ is used to train $Q^{\pi^\uparrow}$ and update $\pi^\uparrow$ with the standard PG theorem, while the abstract MDP ${\overline{\mathcal{M}}}$ is used to train $\overline{Q}^{\overline{\pi}}$ and update $\overline{\pi}$ with the homomorphic PG theorem.  ${\overline{\mathcal{M}}}$ is the MDP homomorphic image of $\mathcal{M}$ obtained by learning the homomorphism map $h \!=\! ( f, g_s )$. The policies $\pi^\uparrow$ and $\overline{\pi}$ are coupled together through the lifting procedure.}
    \label{fig:hpg_diagram}
\end{figure}

\paragraph{Maximum Entropy RL.}
Having the result for stochastic policies give theoretical guarantees when integrating MDP homomorphisms in a wider variety of algorithms. For instance, the \emph{maximum entropy RL} framework generalizes the expected return formulation to encompass the entropy of the policy, resulting in improvements in robustness and exploration \citep{ziebart2008maximum, ziebart2010modeling, haarnoja2018soft}. Of course, such methods are only applicable to stochastic policies. Thus, in contrast to the deterministic HPG, stochastic HPG is capable of benefiting from the addition of the policy's entropy. Importantly, the entropy of the pushforward measure is at most the entropy of the original measure (this is a consequence of the conditional Jensen's inequality--- see~\citet{smorodinsky2006ergodic}); hence seeking maximal entropy policies in the abstract MDP correspond to high-entropy policies in the original space as well.

\paragraph{Sample Efficiency.}
The sample efficiency of deterministic and stochastic PG methods varies significantly depending on the choice of the algorithm, network architecture, exploration strategy, and implementation \citep{henderson2018deep}; nevertheless, it is generally observed that deterministic PG methods are more sample efficient \citep{lillicrap2015continuous, fujimoto2018addressing, barth2018distributed}. One hypothesis is that since deterministic PG integrates only over the state space, in contrast to stochastic PG which integrates over both state and action spaces, the policy gradient estimation is more sample efficient, particularly in high-dimensional action spaces \citep{silver2014deterministic}. The same reasoning is applicable to HPG variants. We have carried out a thorough empirical study on a variety of environments in Section \ref{sec:experiments} to further study the characteristics of the two HPG variants.

\section{Homomorphic Actor-Critic Algorithms}
\label{sec:homomorphic_ac}
\label{sec:hom_algorithm}
In this section, we outline a practical deep reinforcement learning algorithm based on the stochastic and deterministic HPG theorems, referred to as the \emph{Deep Homomorphic Policy Gradient} (DHPG) algorithms. While the overall structure of the algorithm and learning the MDP homomorphisms map are similar in both cases of stochastic and deterministic policies, the policy lifting procedure requires additional intricate steps in the stochastic case. Algorithm \ref{alg:dhpg} describes the pseudo-code of DHPG algorithms.

Denoting pixel observations as ${o}_t$, the underlying states as ${s}_t$, and the abstract states as ${\overline{s}}_t$, the main components of the DHPG algorithm are: the MDP homomorphism map $h_{\phi, \eta} \!=\! (f_\phi({s}_t), g_\eta({s}_t, {a}_t))$, pixel encoder $E_\mu({o}_t)$, actual critic $Q_\psi({s}_t, {a}_t)$ and policy $\pi^\uparrow_\theta({a}_t | {s_t})$, abstract critic $\overline{Q}_{\overline{\psi}}({\overline{s}}_t, {\overline{a}}_t)$ and policy $\overline{\pi}_{\overline{\theta}}({\overline{a}}_t | {\overline{s}}_t)$, reward predictor $\overline{R}_\rho({\overline{s}}_t)$, and probabilistic transition dynamics $\overline{\tau}_\nu ({\overline{s}}_{t+1} | {\overline{s}}_t, {\overline{a}}_t)$ which outputs a Gaussian distribution. Finally, we leverage target critic networks ${Q}_{\psi'}$ and $\overline{Q}_{\overline{\psi}'}$ for a more stable training and use a vanilla replay buffer \citep{mnih2013playing, lillicrap2015continuous}.

\paragraph{Policy Lifting Procedure.}
In general, the lifted policy needs to satisfy the relation $\pi^\uparrow(g_s^{-1}(\beta) | s) = \overline{\pi}(\beta | f(s))$ for every Borel set $\beta \subset \mathcal{\overline{A}}$ and $s \in \mathcal{S}$. As discussed in Section \ref{sec:policy_lifting_value_eqv}, Proposition \ref{prop:lift_existence} proves the existence of the lifted policy $\pi^\uparrow$ from an abstract policy $\overline{\pi}$, however, it does not provide an explicit method for construction of the lifted policy.

If the abstract policy is deterministic, the lifted policy can be simply obtained by choosing one representative for the preimage $g_s^{-1}\big(\overline{\pi}(f(s))\big)$. If we select $g_s$ to be a bijection, as was assumed in Section \ref{sec:deterministic_pg}, the lifted policy can be uniquely defined as $\pi^\uparrow(s) \!=\! g_s^{-1}\big(\overline{\pi}(f(s))\big)$. This allows for parameterizing the two policies using the same network. In practice, we parameterize the actual policy $\pi^\uparrow_\theta$ and obtain the abstract policy as $\overline{\pi}_\theta(f(s)) \!=\! g_s(\pi^\uparrow_\theta(s))$.

The solution is not as straightforward for stochastic abstract policies; while Bayesian approaches for constructing solutions to stochastic inverse problems exist \citep{butler2018combining}, we choose a sampling-based method to derive a loss function as an approximation of the policy lifting procedure \citep{kaipio2006statistical}. Using the change of variable formula of the pushforward measure, we can show that the conditional expectations of abstract actions under the two policies are equal:
\begin{equation*}
    \mathbb{E}_{\pi^\uparrow} [g_s(a) | s] = \int_{a \in \mathcal{A}} g_s(a) \pi^\uparrow(da | s) = \int_{\overline{a} \in \mathcal{\overline{A}}} \overline{a} \; \overline{\pi}(d\overline{a} | \overline{s}) = \mathbb{E}_{\overline{\pi}} [\overline{a} | f(s)].
\end{equation*}
A similar result holds for all finite moments; in particular, the conditional variance of abstract actions under the two policies are equal--- that is:
\begin{equation*}
    \Var_{\pi^\uparrow} [g_s(a) | s] = \Var_{\overline{\pi}} [\overline{a} | f(s)].    
\end{equation*}
Therefore, we can derive a policy lifting loss as a measure of the consistency of the two policies with respect to the MDP homomorphism map and the lifting procedure. Assuming the policies $\pi^\uparrow_\theta$ and $\overline{\pi}_{\overline{\theta}}$ are parameterized by independent neural networks, the loss function is obtained by matching the conditional expectation and standard deviation (SD) of abstract actions conditioned on observations sampled from the replay buffer:
\begin{equation}
    \label{eq:policy_lifting_loss}
    \mathcal{L}_{\text{lift.}}(\theta, {\overline{\theta}}) = (\mathbb{E}_{\pi_\theta^\uparrow} [g_s(a) | s] - \mathbb{E}_{\overline{\pi}_{\overline{\theta}}} [\overline{a} | f(s)])^2 + (\text{SD}_{\pi_\theta^\uparrow} [g_s(a) | s] - \text{SD}_{\overline{\pi}_{\overline{\theta}}} [\overline{a} | f(s)])^2.
\end{equation}
As discussed, the policy lifting loss in Equation \eqref{eq:policy_lifting_loss} is not required for deterministic DHPG.

\paragraph{Training the Critic.} 
Actual and abstract critics are trained using $n$-step TD error for a faster reward propagation \citep{barth2018distributed}. The loss function for each critic is therefore defined as the expectation of the $n$-step Bellman error estimated over transitions samples from the replay buffer $\mathcal{B}$:
\begin{align}
    \label{eq:critic_loss_1}
    \mathcal{L}_\text{actual critic} (\psi) &= \mathbb{E}_{(s, a, s^\prime, r) \sim \mathcal{B}} \big[ \big( R_t^{(n)} + \gamma^n Q_{\psi^\prime}({s}_{t+n}, {a}_{t+n}) - Q_\psi({s}_t, {a}_t) \big)^2 \big] \\
    \label{eq:critic_loss_2}
    \mathcal{L}_\text{abstract critic} (\overline{\psi}, \phi, \eta) &= \mathbb{E}_{(s, a, s^\prime, r) \sim  \mathcal{B}} \big[ \big(R_t^{(n)} + \gamma^n \overline{Q}_{\overline{\psi^\prime}}({\overline{s}}_{t+n}, {\overline{a}}_{t+n}) - \overline{Q}_{\overline{\psi}}({\overline{s}}_t, {\overline{a}}_t) \big)^2 \big],
\end{align}
where ${\overline{s}_t} \!=\! f_\phi({s}_t)$ and ${\overline{a}}_t \!=\! g_\eta({s}_t, {a}_t)$ are computed using the learned MDP homomorphism, $\psi^\prime$ and $\overline{\psi^\prime}$ denote parameters of target networks, and $R_t^{(n)} \!\!=\!\! \sum_{i=0}^{n-1} \gamma^i r_{t+i}$ is the $n$-step return.

\begin{algorithm}[b!]
    \caption{Deep Homomorphic Policy Gradient (DHPG)}
    \begin{algorithmic}[1]
        \State Initialize target networks $\quad \psi' \leftarrow \psi$, $\overline{\psi'} \leftarrow \overline{\psi}$.
        \For{$t = 0$ to $T$} 
            \State Encode observation: $\quad {s}_t = E_\mu({o}_t)$
            \State Select and execute action:
            \If{stochastic policy}
                \State ${a}_t \sim \pi^\uparrow_\theta(\cdot | {s}_t)$
            \Else
                \State ${a}_t = \pi^\uparrow_\theta({s}_t) + \epsilon$, where $\epsilon \sim {N}(0, \sigma)$
            \EndIf
            \State Store transition in the replay buffer: $\quad \mathcal{B} \leftarrow \mathcal{B} \cup ({s}_t, {a}_t, {s}_{t+1}, r_{t})$
            \State Sample a mini-batch: $\quad B_i \sim \mathcal{B}$ 
            \State Permute the mini-batch: $\quad B_j = \text{permute}(B_i)$
            \State Train $h_{\phi, \eta}, E_\mu, \overline{\tau}_\nu, \overline{R}_\rho$: $\quad \mathcal{L}_\text{lax} + \mathcal{L}_\text{h}$ \Comment{Eqns. \ref{eq:lax_bisim_loss}-\ref{eq:hom_loss}}
            \State Train critics $Q_\psi, \overline{Q}_{\overline{\psi}}$: $\quad \mathcal{L}_\text{actual} + \mathcal{L}_\text{abstract}$ \Comment{Eqns. \ref{eq:critic_loss_1}-\ref{eq:critic_loss_2}}
            \If{stochastic policy}
                \State \# Lifted and abstract policies are respectively parameterized by $\theta$ and $\overline{\theta}$
                 \State Train $\pi^\uparrow_\theta$ using PG + MaxEnt   \Comment{SAC \citep{haarnoja2018soft}}
                \State Train $\overline{\pi}_{\overline{\theta}}$ using stochastic HPG \Comment{Theorem \ref{thm:stochastic_hpg}}
                \State Update policies $\pi^\uparrow_\theta$ and $\overline{\pi}_{\overline{\theta}}$ with the policy lifting loss: $\quad \mathcal{L}_\text{lift.}$ \Comment{Eqn. \ref{eq:policy_lifting_loss}} 
            \Else
                \State \# Lifted and abstract policies share the same parameters $\theta$
                \State Train $\pi^\uparrow_\theta$ using DPG   \Comment{TD3 \citep{fujimoto2018addressing}}
                \State Train $\overline{\pi}_{\theta}$ using deterministic HPG \Comment{Theorem \ref{thm:det_hpg}}
            \EndIf           
            \State Update target networks: $\quad \psi' \leftarrow \alpha \psi + (1 - \alpha) \psi', \; \overline{\psi'} \leftarrow \alpha \overline{\psi} + (1 - \alpha) \overline{\psi'}$
        \EndFor
    \end{algorithmic}
    \label{alg:dhpg}
\end{algorithm}

\paragraph{Training the Actor.}
For stochastic policies, we train the actual policy using the standard PG theorem \citep{sutton2000policy}, and train the abstract policy via the stochastic HPG (Theorem \ref{thm:stochastic_hpg}). While we can employ any stochastic actor-critic algorithm for training the actual policy, we use SAC \citep{haarnoja2018soft} to further demonstrate the applicability of our method to the maximum entropy RL framework. Notably, as discussed in Section \ref{sec:comparison_hpg}, the entropy regularizer term needs to be accounted for only during the actual policy update. 

In the case of deterministic policies, the actual policy is trained using the deterministic PG \citep{silver2014deterministic} and the abstract policy is updated using the deterministic HPG (Theorem \ref{thm:det_hpg}). Notably, since in this case the actual and abstract policies share the same set of parameters ($\theta = \overline{\theta}$), both policy updates are applied to the same set of policy parameters.

\paragraph{Learning the continuous MDP Homomorphism Map.}
We learn the continuous MDP homomorphism map using the lax bisimulation metric \citep{taylor2008bounding}, similarly to the prior work on continuous MDP homomorphisms \citep{rezaei2022continuous}. We use the lax bisimulation metric \citep{taylor2008bounding}, Equation \eqref{eq:lax_bisim}, to propose a loss function that encodes lax bisimilar states closer together in the abstract space. The lax bisimulation metric is applicable to continuous actions and as a (pseudo-)metric, it can naturally represent approximate MDP homomorphisms. The lax bisimulation metric between two state-action pairs $({s}_i, {a}_i)$ and $({s}_j, {a}_j)$ is defined as:
\begin{equation}
    \label{eq:lax_bisim}
    d_\text{lax} \big( ({s}_i, {a}_i), ({s}_j, {a}_j)  \big) = c_r |R({s}_i, {a}_i) - R({s}_j, {a}_j)| + c_t W_1 \big( \tau(\cdot | {s}_i, {a}_i), \tau(\cdot | {s}_j, {a}_j) \big),
\end{equation}
where the first term measures the distance between the reward terms and $W_1$ is the Kantorovich metric measuring the distance between transition probabilities. Following the same intuition as prior works on using bisimulations for representation learning \citep{zhang2020learning}, we define our proposed lax bisimulation loss over pairs of transition tuples sampled from the replay buffer. We permute samples to compute their pairwise distance in the abstract space and their pairwise lax bisimilarity distance. Consequently, we minimize the distance between these two terms:
{
\small
\begin{equation}
    \mathcal{L}_\text{lax}(\phi, \eta) = \mathbb{E}_{\mathcal{B}} \big[ \| f_\phi({s}_i) - f_\phi({s}_j) \|_1 - \|r_i-r_j\|_1 - \alpha W_2 \big( \overline{\tau}_\nu(\cdot | f_\phi({s}_i), g_\eta(s_i, {a}_i)), \overline{\tau}_\nu(\cdot | f_\phi({s}_j), g_\eta(s_j, a_j)) \big) \big]^2
    \label{eq:lax_bisim_loss}
\end{equation}
}
Here, the expectation is taken over two samples $(s_i, a_i, s_i', r_i), (s_j, a_j, s_j', r_j) \sim \mathcal{B}$ from the replay buffer.

Similarly to \citet{zhang2020learning}, we replaced the Kantorovich ($W_1$) metric with the $W_2$ metric as there is an explicit formula for it for Gaussian distributions. Finally, we apply the conditions of a continuous MDP homomorphism map from Definition \ref{def:cont_mdp_homo} via the loss function of:
\begin{equation}
    \label{eq:hom_loss}
    \mathcal{L}_\text{h}(\phi, \eta, \nu, \rho) = \mathbb{E}_{(s_i, a_i, s_i^\prime, r_i) \sim \mathcal{B}} \big[ \big( f_\phi({s}_i^\prime) - {\overline{s}}_i^\prime \big)^2 + \big( r_i \!-\! \overline{R}_\rho(f_\phi({s}_i)) \big)^2  \big],
\end{equation}
where ${\overline{s}}_i^\prime \!\sim\!  \overline{\tau}_\nu(\cdot | f_\phi({s}_i), g_\eta({s}_i, {a}_i))$. The final loss function is obtained as  $\mathcal{L}_\text{lax}(\phi, \eta) + \mathcal{L}_\text{h}(\phi, \eta, \rho, \nu)$.

\section{Experiments}
\label{sec:experiments}
\label{sec:exp_results}
\begin{figure}[t!]
    \centering
    \begin{subfigure}[b]{0.32\textwidth}
         \centering
         \includegraphics[width=\textwidth]{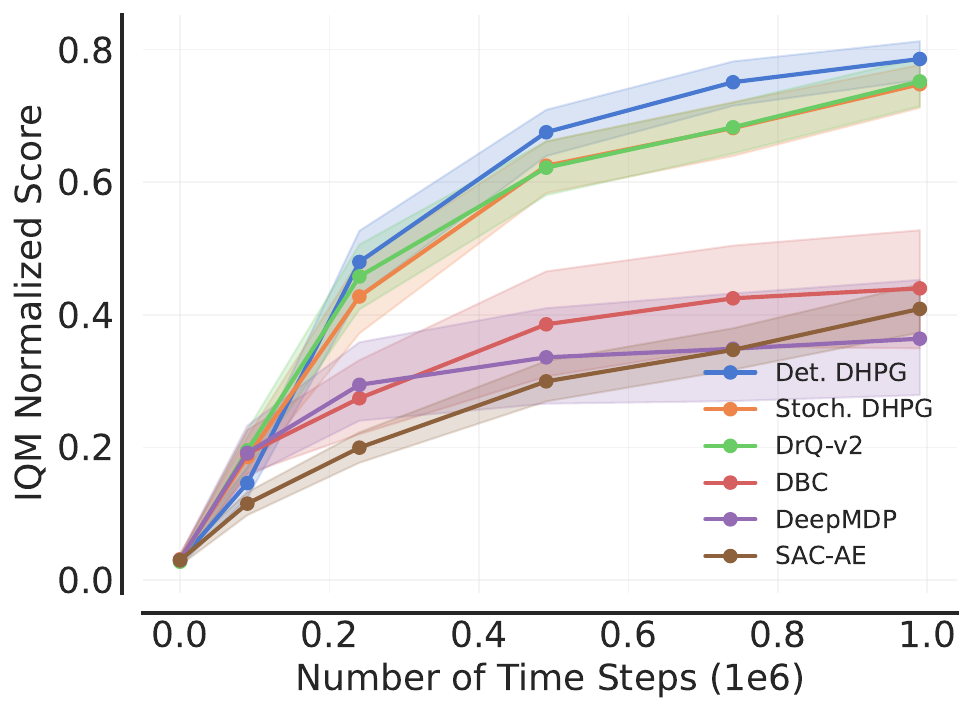}
         \caption{Sample efficiency.}
         \label{fig:pixels_sample_efficiency}
    \end{subfigure}
    \begin{subfigure}[b]{0.32\textwidth}
         \centering
         \includegraphics[width=\textwidth]{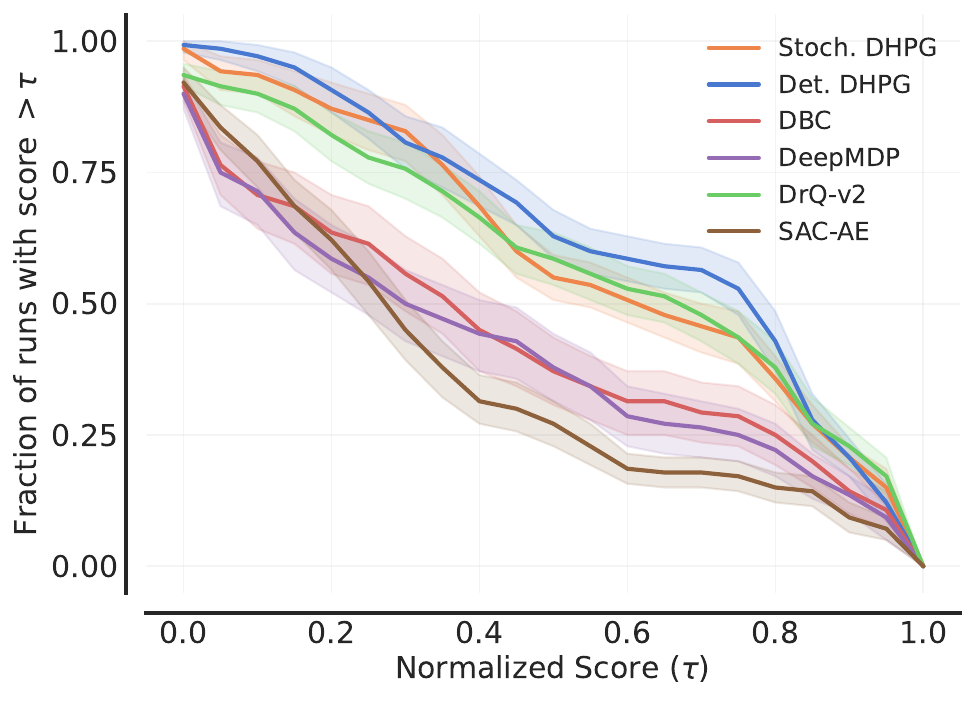}
         \caption{Performance profiles.}
         \label{fig:spixels_perf_profile_main}
    \end{subfigure}
    \begin{subfigure}[b]{0.32\textwidth}
         \centering
         \includegraphics[width=\textwidth]{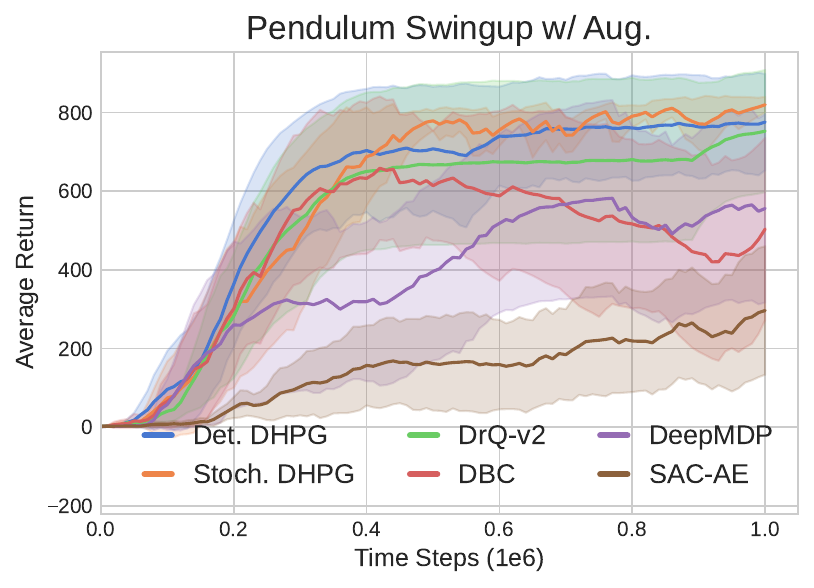}
         \caption{Learning curves.}
         \label{fig:pixels_res_1}
    \end{subfigure}    
    \caption{Results of DM Control tasks with pixel observations obtained on 10 seeds. RLiable metrics are aggregated over 14 tasks.  All methods are \textbf{with} image augmentation. \textbf{(a)} RLiable IQM scores as a function of number of steps for comparing sample efficiency, \textbf{(b)} RLiable performance profiles at 500k steps, \textbf{(c)} learning curves on the pendulum swingup task. Full results are in Appendix \ref{sec:additional_results_pixels}. Shaded regions represent $95\%$ confidence intervals.}
    \label{fig:pixels_results}
\end{figure}

In our experiments, we aim to answer the following key questions:
\begin{enumerate}[noitemsep,nosep]
    \item Does the homomorphic policy gradient improve policy optimization?
    \item What are the qualitative properties of the learned representations and the abstract MDP?
    \item Can DHPG learn and recover the minimal MDP image from raw pixel observations?
\end{enumerate}

We evaluate DHPG on continuous control tasks from DM Control on pixel observations. Importantly, to reliably evaluate our algorithm against the baselines and to correctly capture the distribution of results, we follow the best practices proposed by \citet{agarwal2021deep} and report the interquartile mean (IQM) and performance profiles aggregated on all tasks over 10 random seeds. While our baseline results are obtained using the official code, when possible\footnote{We use the official implementations of DrQv2, DBC, and SAC-AE, while we re-implement DeepMDP due to the unavailability of the official code. See Appendix \ref{sec:baseline_impl} for full details.}, some of the results may differ from the originally reported ones due to the difference in the seed numbers and our goal to present a faithful representation of the true performance distribution \citep{agarwal2021deep}. 

\subsection{DeepMind Control Suite}
\label{sec:results_pixels}
We compare the effectiveness of DHPG on pixel observations against DBC \citep{zhang2020learning}, DeepMDP \citep{gelada2019deepmdp}, SAC-AE \citep{yarats2021improving}, and state-of-the-art performing DrQ-v2 \citep{yarats2021mastering}. All methods use $n$-step returns, share the same hyperparameters in Appendix \ref{sec:hyperparams} and all hyperparameters are adapted from DrQ-v2 \emph{without any further tuning}. We acknowledge that since DrQ-v2 is based upon DDPG, the hyperparameters we used may be more advantageous to det erministic DHPG in comparison with stochastic DHPG. Importantly, for a fair comparison with DrQ-v2 which uses image augmentation, we present two variations of DHPG and other baselines, \emph{with and without image augmentation}.

\begin{figure}[t!]
    \centering
    \begin{subfigure}[b]{0.36\textwidth}
        \centering
        \includegraphics[width=\textwidth]{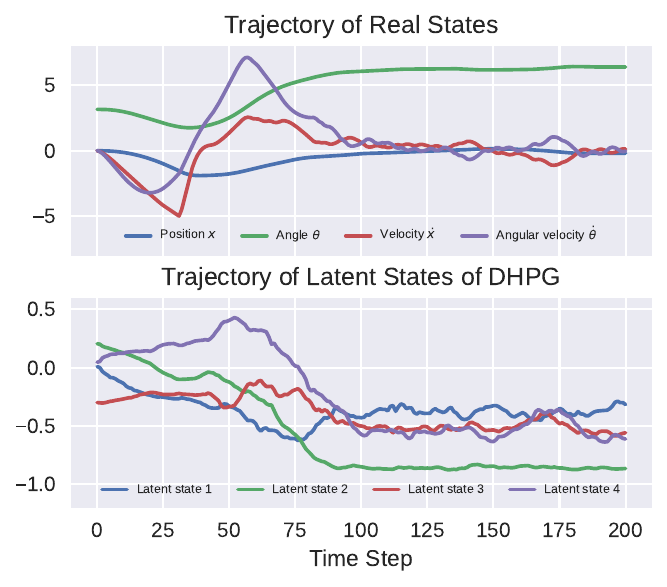}
        \caption{State trajectories over time.}
        \label{fig:low_dim_results_traj}
    \end{subfigure}
    \hfill
    \begin{subfigure}[b]{0.31\textwidth}
        \centering
        \includegraphics[width=\textwidth]{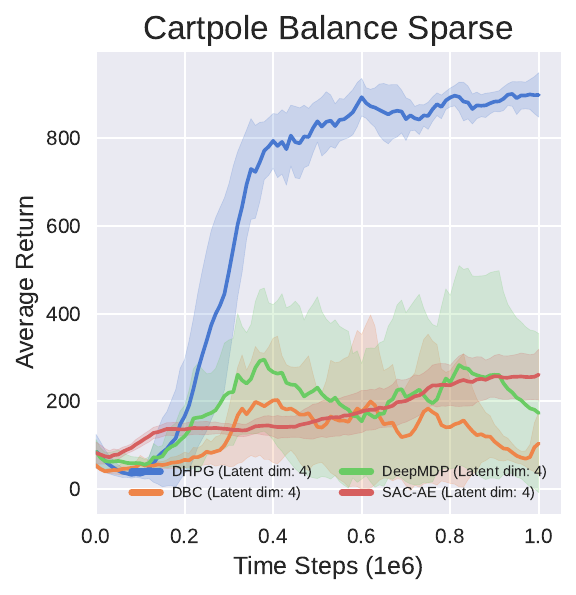}
        \caption{Learning curves.}
    \end{subfigure}
    \hfill
    \begin{subfigure}[b]{0.31\textwidth}
        \centering
        \includegraphics[width=\textwidth]{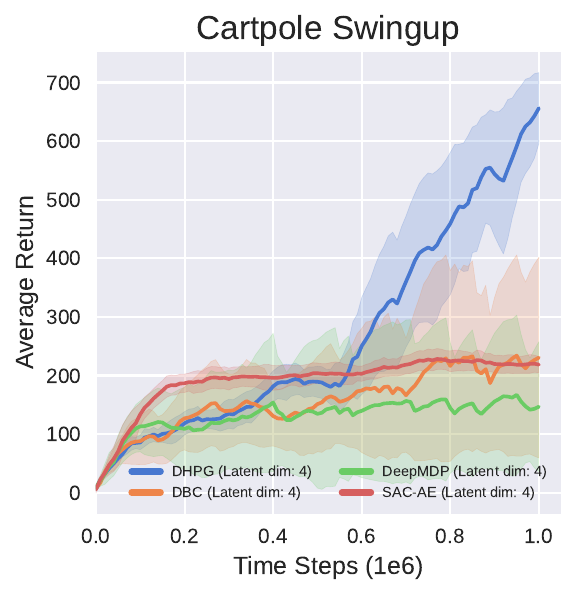}
        \caption{Learning curves.}
    \end{subfigure}
    \caption{Effectiveness of DHPG in recovering the minimal MDP from pixels. All methods are limited to a 4-dimensional latent space which is equal to the dimensions of the real state space of cartpole. \textbf{(a)} Trajectories of real states obtained from Mujoco and trajectories of latent states of DHPG. \textbf{(b, c)} Learning curves averaged on 10 seeds.}
    \label{fig:low_dim_results}
\end{figure}

\paragraph{DHPG outperforms or matches other algorithms on pixel observations, demonstrating its effectiveness in representation learning.} Results with image augmentation are presented in Figure \ref{fig:pixels_results} and \emph{full results are in Appendix \ref{sec:additional_results_pixels}}. Aggregated over 14 tasks, deterministic DHPG outperforms state-of-the-art DrQ-v2 and stochastic DHPG is as performant. Although these yield slight performance gains overall, the comparison in performance between DHPG and DrQ-v2 is highly task-dependent (see Figure~\ref{fig:pixel_results_supp_aug} in Appendix \ref{sec:additional_results_pixels}). For complex tasks such as Walker Run or Cheetah Run, DHPG obtains equal or slightly worse performance compared to DrQ-v2; however, on domains with clear symmetries---and hence easily learnable MDP homomorphism maps--- DHPG outperforms DrQ-v2. In particular, DHPG without image augmentation outperforms DrQ-v2 on domains such as Cartpole and Pendulum, demonstrating its capability of representation learning. 

\paragraph{Deterministic DHPG and stochastic DHPG have approximately similar sample efficiency, with deterministic DHPG being slightly better.} As discussed in Section \ref{sec:comparison_hpg}, deterministic policy gradient in theory is more sample efficient than the stochastic policy gradient, since it does not need to integrate over the action space \citep{silver2014deterministic}. Additionally, due to the complications of lifting a stochastic policy, stochastic DHPG has more components which can negatively impact the learning performance. As a result, deterministic DHPG is slightly more sample efficient than stochastic DHPG.

\paragraph{DHPG can learn and recover a low-dimensional MDP image.} A key strength of MDP homomorphisms is their ability to represent the minimal MDP image \citep{ravindran2001symmetries}, which is particularly important when learning from pixel observations. To demonstrate this ability, we have limited the latent space dimensions to the dimension of the real system and compared DHPG (without image augmentation) with baselines in Figure \ref{fig:low_dim_results}. While other methods are not able to learn the tasks, DHPG can successfully learn the policy and the minimal low-dimensional latent space. Surprisingly, trajectories of the latent states resemble that of the real states as shown in Figure \ref{fig:low_dim_results_traj}.

\begin{figure}[t!]
    \centering    
    \begin{subfigure}[b]{0.4\textwidth}
         \centering
         \includegraphics[width=\textwidth]{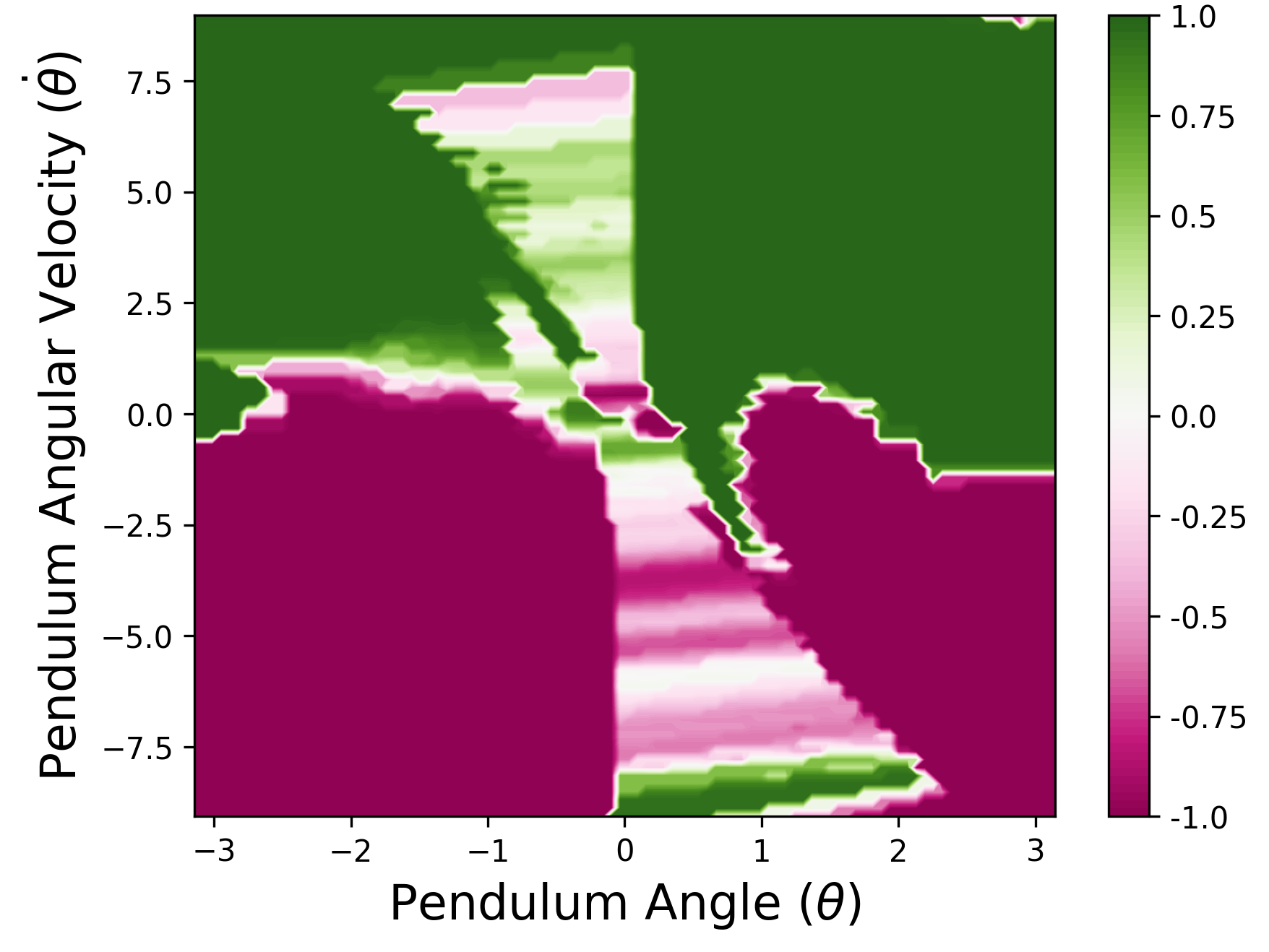}
         \caption{Actual optimal policy.}
         \label{fig:states_pendulum_action_contour_a}
    \end{subfigure}    
    \hspace{1em}
    \begin{subfigure}[b]{0.4\textwidth}
         \centering
         \includegraphics[width=\textwidth]{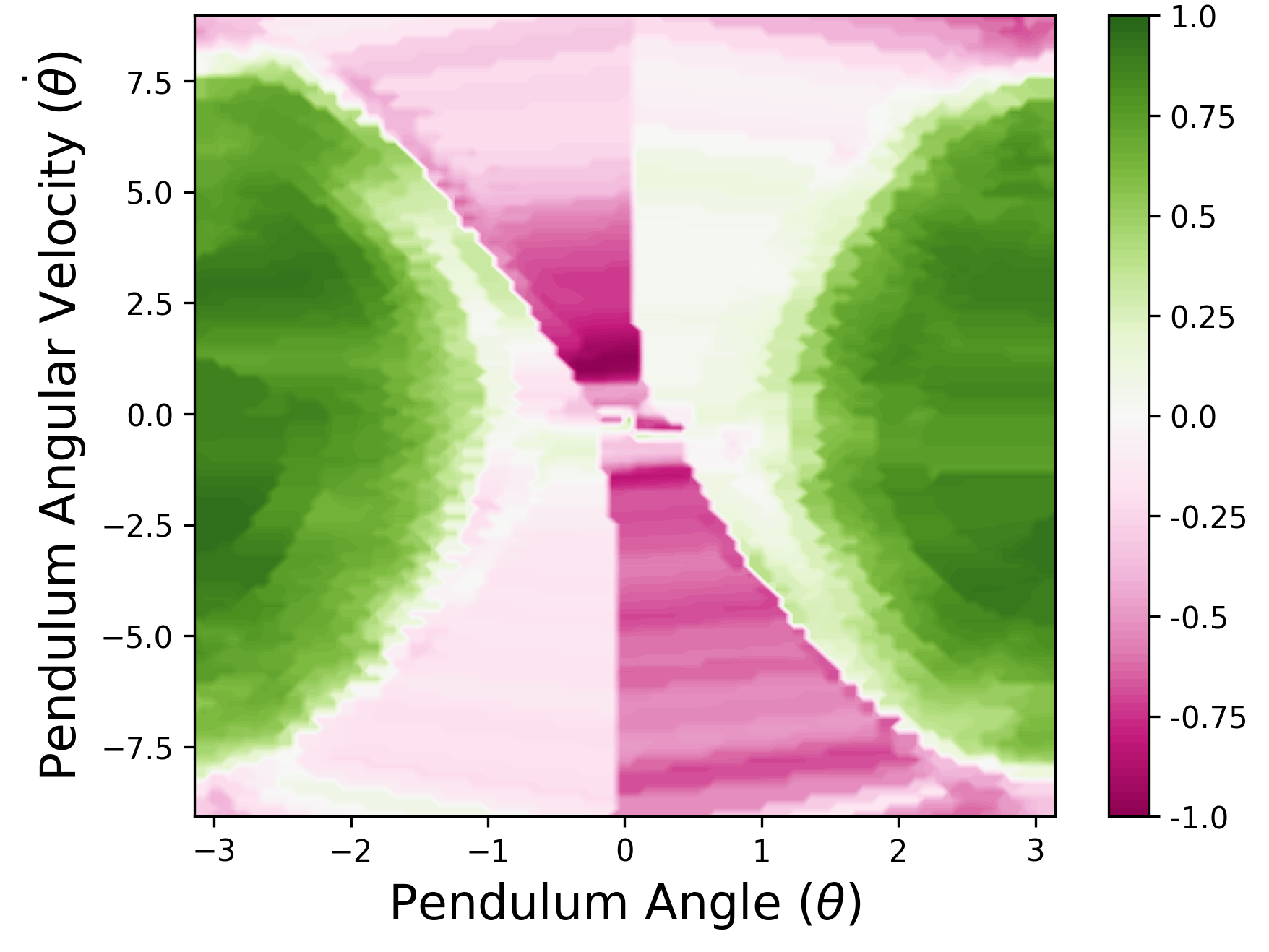}
         \caption{Abstract optimal policy.}
         \label{fig:states_pendulum_action_contour_b}
    \end{subfigure}    
    \caption{Contours of actual and abstract optimal actions over the state space of the pendulum-swingup task. Colors represent action values, and states are $s \!=\! (\theta, \dot{\theta})$. \textbf{(a)} Actual optimal policy;  contours of optimal actions $a^* \!=\! \pi^{\uparrow^*}\!(s)$. \textbf{(b)} Abstract optimal policy; contours of abstract optimal actions $\overline{a}^* \!=\! g_s(a^*) \!=\! \overline{\pi}^*(\overline{s})$. The relation $g_{s_1}\!(a_1) \!=\! g_{s_2}\!(a_2)$ holds for equivalent state-action pairs, and the abstract optimal policy is symmetric.}
    \label{fig:states_pendulum_action_contour}    
\end{figure}

\begin{figure}[b!]
    \centering
    \hfill
    \begin{subfigure}[b]{0.24\textwidth}
        \centering
        \includegraphics[width=\textwidth]{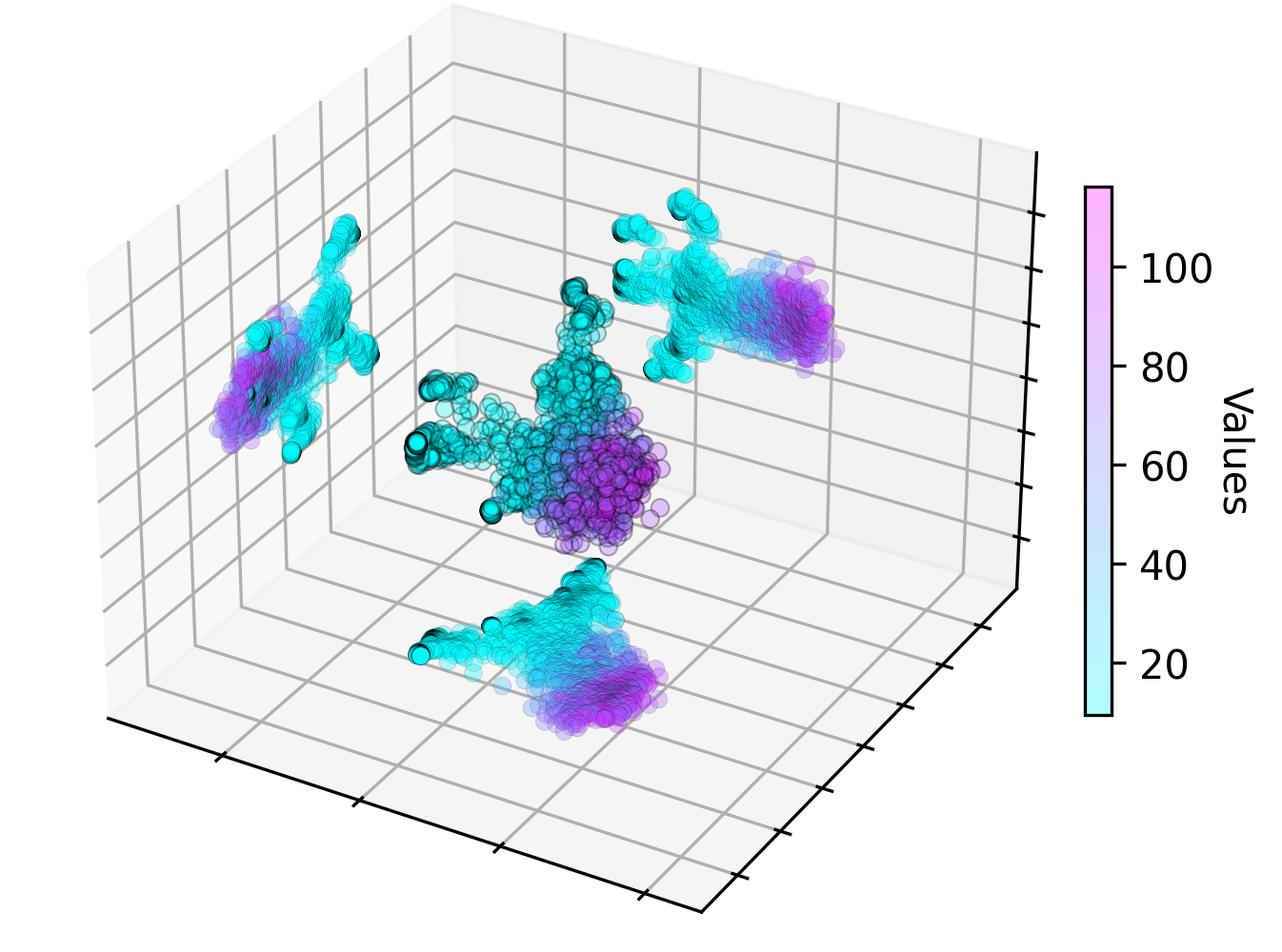}
        \caption{\scriptsize DHPG latent states.}
        \label{fig:pixels_quadruped_mdp_vis_a}
    \end{subfigure}
    \hfill
    \begin{subfigure}[b]{0.24\textwidth}
        \centering
        \includegraphics[width=\textwidth]{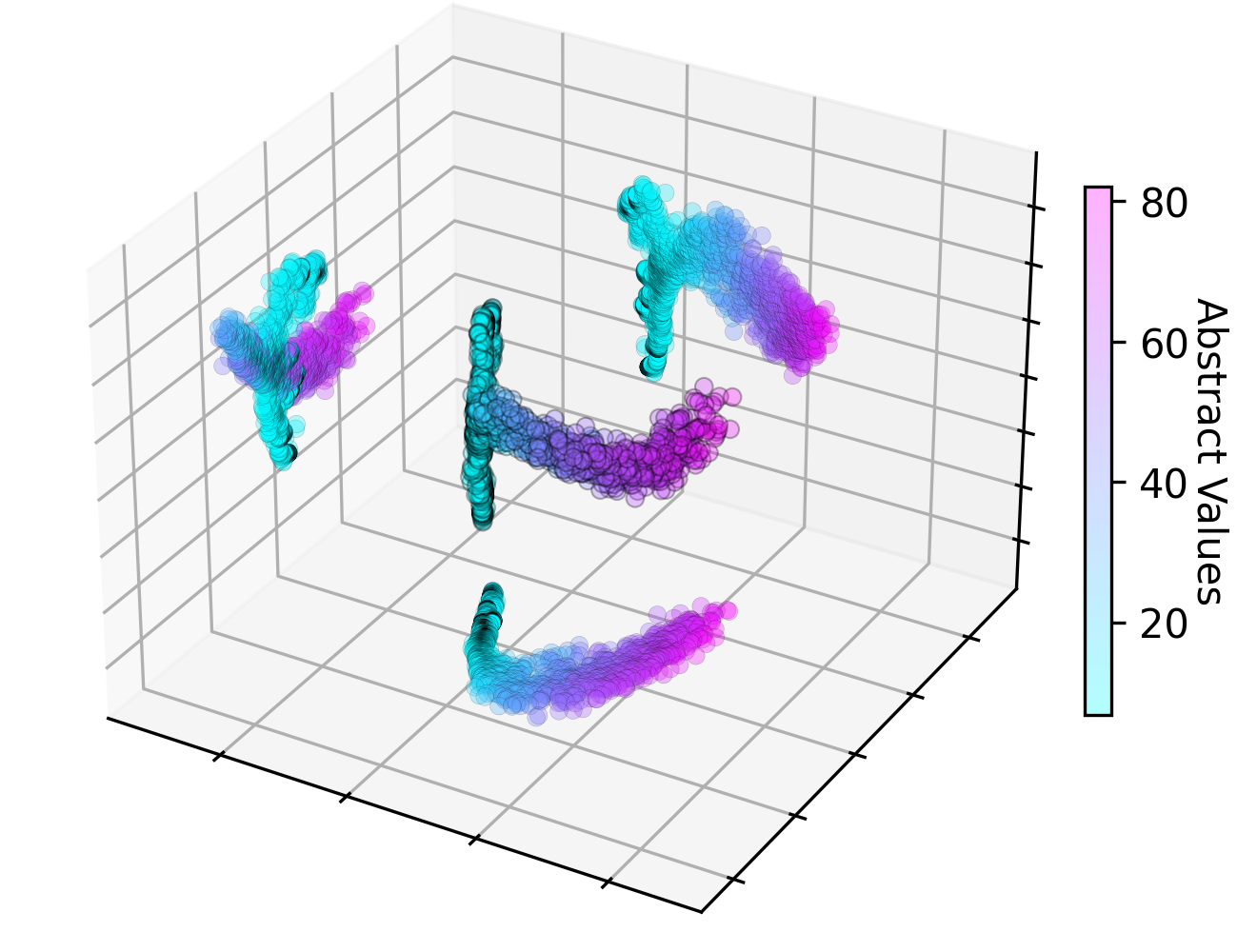}
        \caption{\scriptsize DHPG abstract states.}
        \label{fig:pixels_quadruped_mdp_vis_b}
    \end{subfigure}
    \hfill
    \begin{subfigure}[b]{0.24\textwidth}
        \centering
        \includegraphics[width=\textwidth]{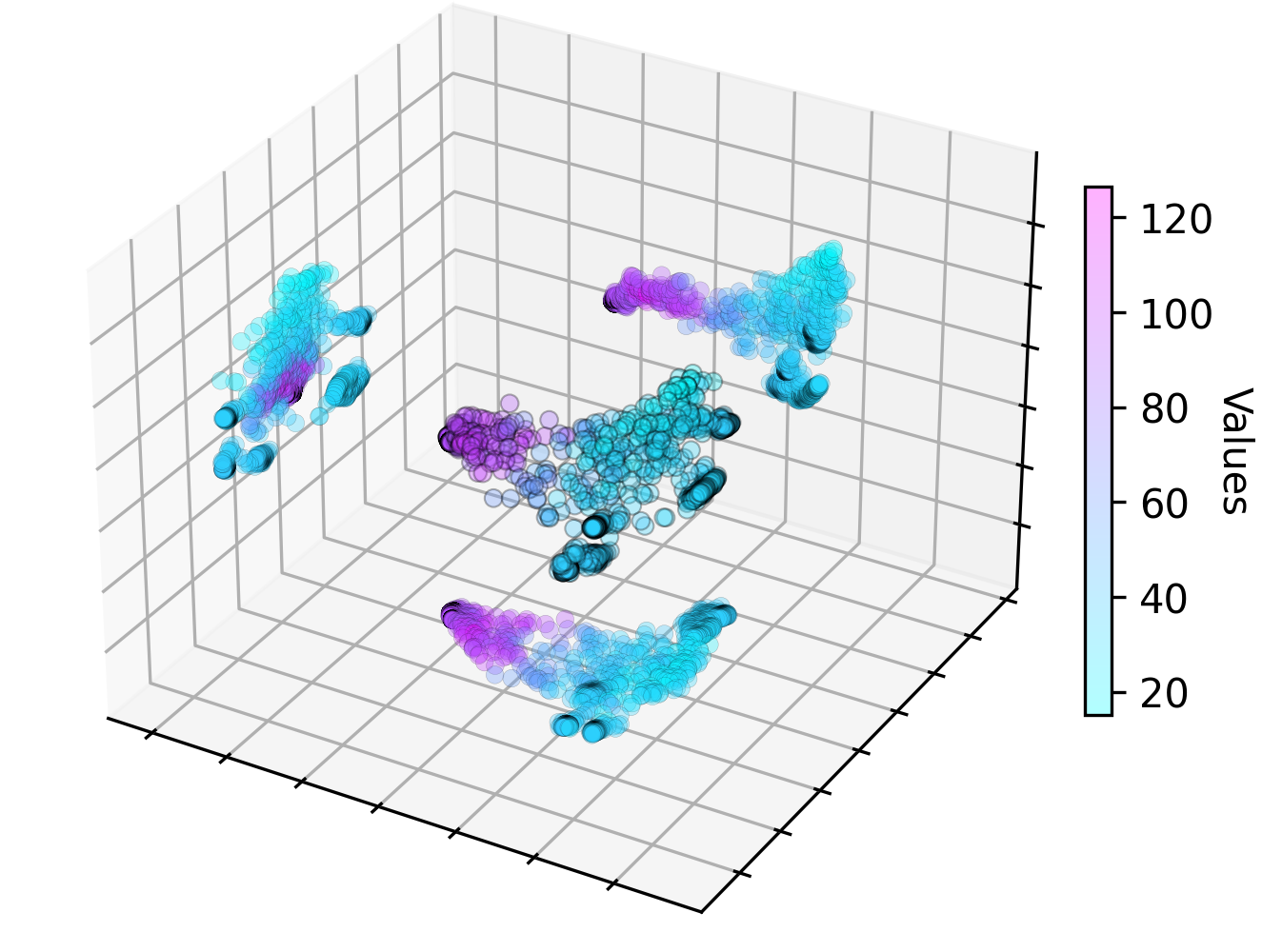}
        \caption{\scriptsize DBC latent states.}
    \end{subfigure}
    \hfill
    \begin{subfigure}[b]{0.24\textwidth}
        \centering
        \includegraphics[width=\textwidth]{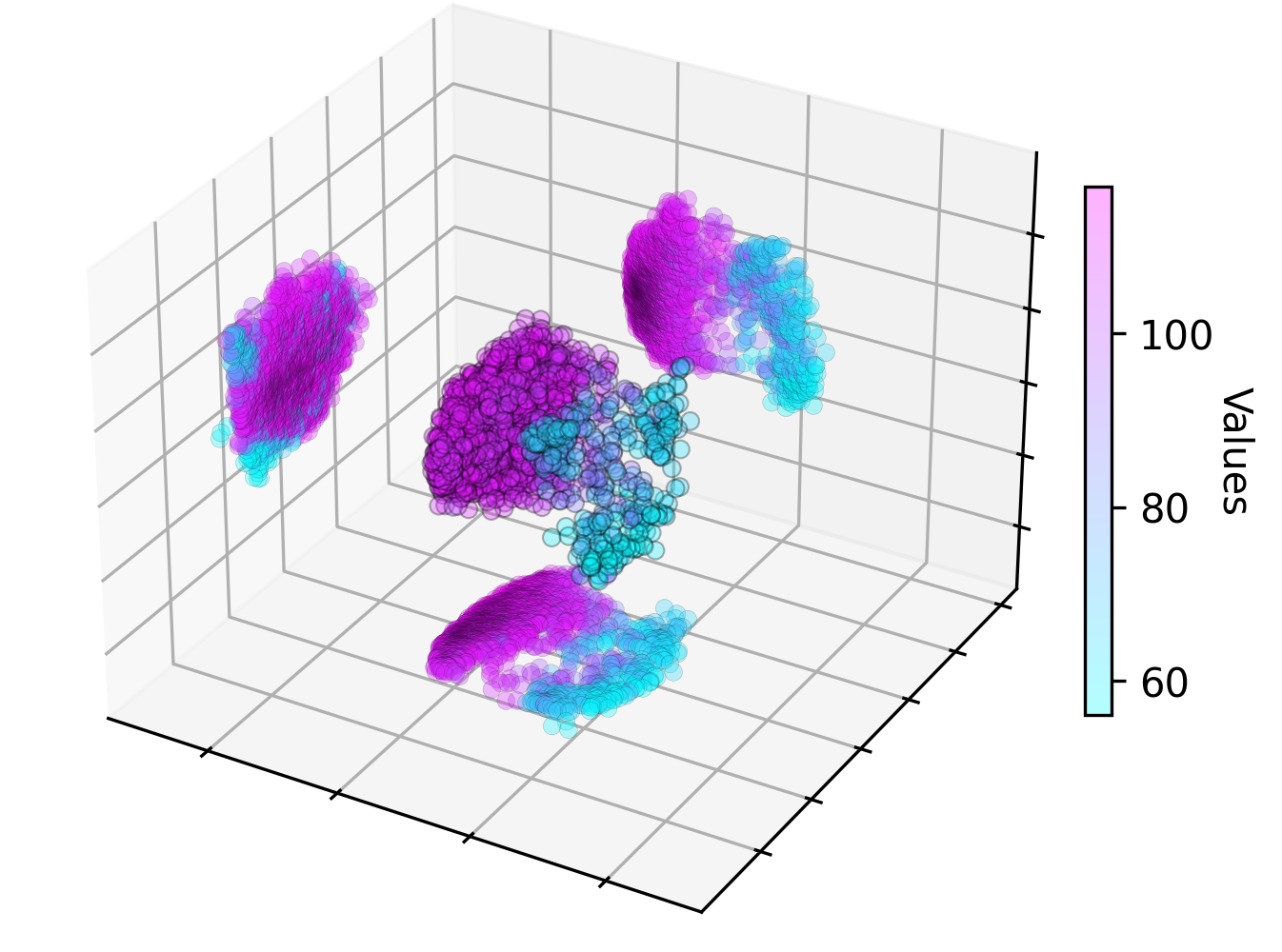}
        \caption{\scriptsize DrQ-v2 latent states.}
    \end{subfigure}
    \hfill
    \caption{PCA projection of learned representations for quadruped-walk with pixel observations. \textbf{(a)} Latent states $s \!=\! E_\mu(o)$, \textbf{(b)} abstract latent states $\overline{s} \!=\! f_\phi(E_\mu(o))$ for DHPG, \textbf{(c)} latent states $s \!=\! E_\mu(o)$ for DBC, and \textbf{(d)} DrQ-v2. Color of each point denotes its value learned by $Q(s, a)$ or $\overline{Q}(\overline{s}, \overline{a})$. Points are also projected onto each main plane. The homomorphism map of DHPG has mapped the latent states of corresponding legs (e.g., left forward leg and right backward leg) \textbf{(a)} on to the same abstract latent states \textbf{(b)}, indicating a clear structure in ${\overline{\mathcal{S}}}$.}
    \label{fig:pixels_quadruped_mdp_vis}
\end{figure}

\paragraph{The learned mapping ${h \!\!=\!\! (f, g_s\!)}$ demonstrates properties of an MDP homomorphism.} We use the pendulum swingup task to visualize its learned MDP homomorphism, as its symmetries are perfectly intelligible. Two state-action pairs $(s_1\!=\!(\theta_1, \dot{\theta}_1), a_1)$ and $(s_2\!=\! (\theta_2, \dot{\theta}_2), a_2)$ are equivalent if $a_1 \!=\! -a_2$, $\theta_1 \!=\! -\theta_2$, and $\dot{\theta}_1 \!=\! -\dot{\theta}_2$. Therefore, the learned action representations are expected to reflect this by setting $g_{s_1}(a_1) \!=\! g_{s_2}(a_2)$. Figure \ref{fig:states_pendulum_action_contour_a} shows contours of optimal actions over $\mathcal{S}$, while Figure \ref{fig:states_pendulum_action_contour_b} shows action representations $\overline{a} \!=\! g_s(a)$ of optimal actions over $\mathcal{S}$. Clearly, abstract actions adhere to the aforementioned relation for equivalent state-action pairs, indicating $g_s(a)$ is in fact representing the action encoder of an MDP homomorphism mapping.

\paragraph{The abstract MDP demonstrates properties of an MDP homomorphic image.} To qualitatively demonstrate the significance of learning joint state-action representations, Figure \ref{fig:pixels_quadruped_mdp_vis} shows visualizations of latent states for quadruped-walk, a task with symmetries around movements of its four legs. Interestingly, while the latent space of DHPG (Figure \ref{fig:pixels_quadruped_mdp_vis_a}) shows distinct states for each leg, abstract state encoder $f_\phi$ has mapped corresponding legs (e.g., left forward leg and right backward leg) to the same abstract latent state (Figure \ref{fig:pixels_quadruped_mdp_vis_b}) as they are some homomorphic image of one another. Clearly, DBC and DrQ-v2 are not able to achieve this.

\begin{figure}[t!]
     \centering
     \begin{subfigure}[b]{0.75\textwidth}
         \centering
         \includegraphics[width=\textwidth]{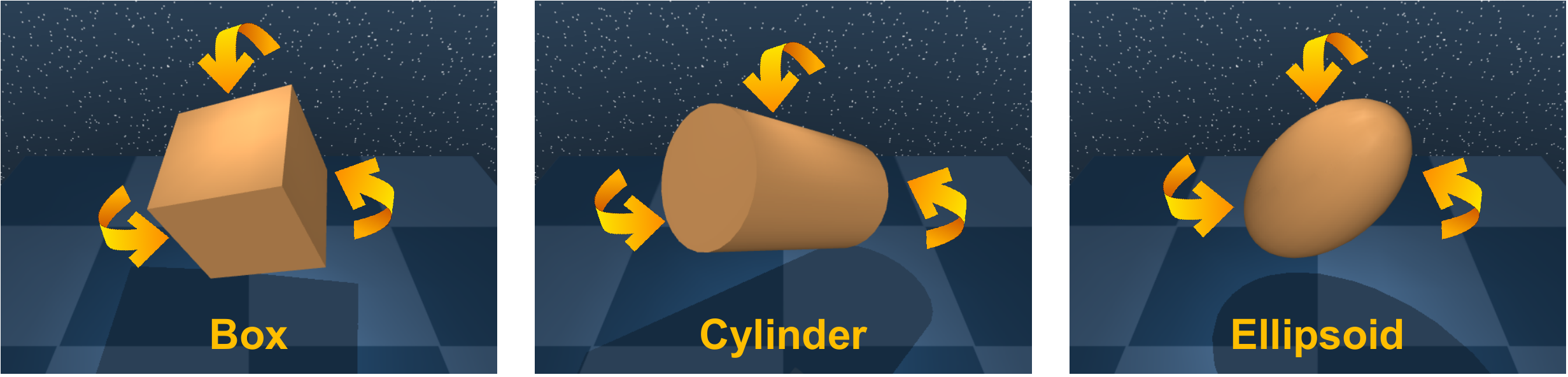}
         \caption{Rotate Suite.}
         \label{fig:rotate_suite_env}
     \end{subfigure}
     \hfill
     \begin{subfigure}[b]{0.22\textwidth}
         \centering
         \includegraphics[width=\textwidth]{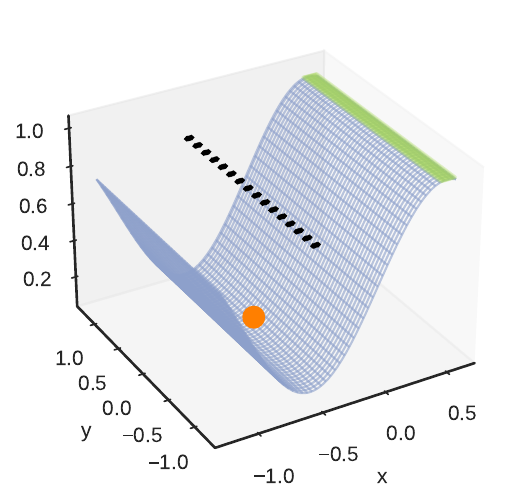}
         \caption{3D mountain car.}
         \label{fig:mountain_car_env}
     \end{subfigure}
     \hfill
    \caption{Novel environments with symmetries. \textbf{(a)} Rotate Suite is a series of visual control tasks with the goal of rotating a 3D object along its axes to achieve a goal orientation. Symmetries of the environment are declared by symmetries of the object. \textbf{(b)} 3D mountain car has a continuous translational symmetry along the $y$-axis, shown as a dotted black line. The orange point represents the car and the green line represents the goal position.}
    \label{fig:cont_symm_environments}
\end{figure}

\paragraph{The learned representations and the MDP homomorphism map transfer to new tasks within the same domain.} Importantly, one consideration with representation learning methods relying on rewards is the transferability of the learned representations to a new reward setting within the same domain. To ensure that our method does not hinder such transfer, we have carried out experiments in which the actor, critics, and the learned MDP homomorphism map are transferred to another task from the same domain. Results, given in Appendix \ref{sec:transfer_supp} show that our method has not compromised transfer abilities.

\paragraph{Additional Experiments.}
We study the value equivalence property as a measure for the quality of the learned MDP homomorphisms in Appendix \ref{sec:value_equiv_supp}, and we present ablation studies on DHPG variants, and the impact of $n$-step return on our method in Appendices \ref{sec:ablation_dhpg_variants} and \ref{sec:ablation_n_step}, respectively. Finally, we compare the computation time of our method against the baselines in Appendix \ref{sec:computation_supp}.

\subsection{Environments with Continuous Symmetries}
As discussed in Section \ref{sec:comparison_hpg}, the key difference between the deterministic and stochastic HPG theorems is that the former requires a bijective action encoder, whereas the latter does not impose any structure on it. The implication of such requirement is that deterministic DHPG is not capable of abstracting actions beyond relabling them. While relabling actions is sufficient for discrete symmetries, environments with continuous symmetries can in principle have their action dimensions reduced. To showcase the superiority of stochastic DHPG in action abstraction, we carry out experiments on a suite of novel environments with continuous symmetries. These environments are publicly available\footnote{\href{https://github.com/sahandrez/rotate_suite}{\texttt{https://github.com/sahandrez/rotate\_suite}}}.
\begin{itemize}[noitemsep]
    \item \textbf{Rotate Suite} is a series of visual control tasks developed based on the DeepMind Control Suite. The goal in each environment is to rotate a 3D object along its axes to achieve a goal orientation. Symmetries of the environment are declared by symmetries of the object. Thus, the box rotation has discrete symmetries, while cylinder and ellipsoid rotation have continuous rotational symmetries. Figure \ref{fig:rotate_suite_env} shows examples of these interactive environments.
    \item \textbf{3D Mountain Car} is an extension of the 2D mountain car problem \citep{moore1990efficient} in which the mountain curve is extended along the $y$-axis, creating a mountain surface. The agent has a 2D action on the mountain surface, in contrast to the 1D action space of the 2D mountain car. As a result of this extension, the problem has a continuous translational symmetry along the $y$-axis and the action along side this axis is redundant. Figure \ref{fig:mountain_car_env} shows the surface of the 3D mountain car.
\end{itemize}

\begin{figure}[t!]
     \centering
     \begin{subfigure}[b]{0.32\textwidth}
         \centering
         \includegraphics[width=\textwidth]{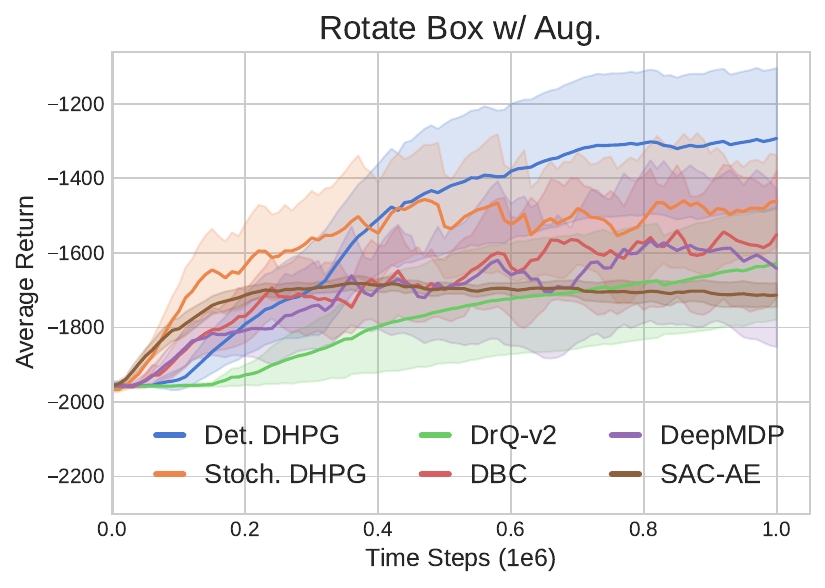}
     \end{subfigure}
     \begin{subfigure}[b]{0.32\textwidth}
         \centering
         \includegraphics[width=\textwidth]{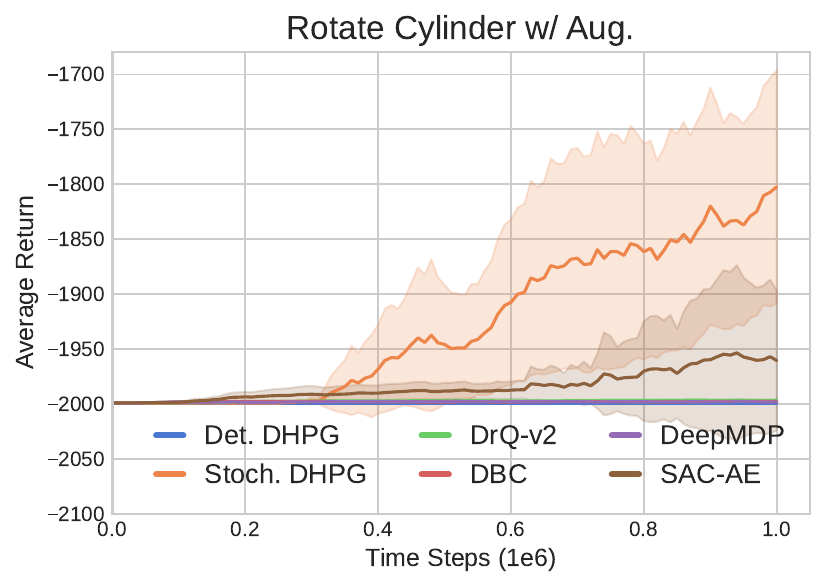}
     \end{subfigure}
     \begin{subfigure}[b]{0.32\textwidth}
         \centering
         \includegraphics[width=\textwidth]{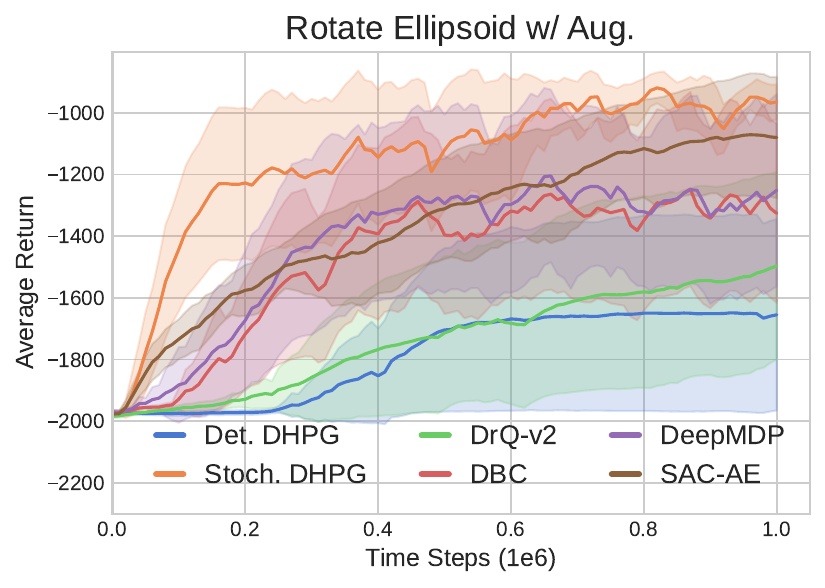}
     \end{subfigure}     
    \caption{Learning curves for Rotate Suite environments obtained on 10 seeds. \textbf{(a)} Box rotation. \textbf{(b)} Cylinder rotation. \textbf{(c)} Ellipsoid rotation. Shaded regions represent $95\%$ confidence intervals. Stochastic DHPG outperforms deterministic DHPG on environments with continuous symmetries (cylinder and ellipsoid rotation).}
    \label{fig:rotate_suite_res}
\end{figure}

\begin{figure}[b!]
    \centering
    \includegraphics[width=0.7\textwidth]{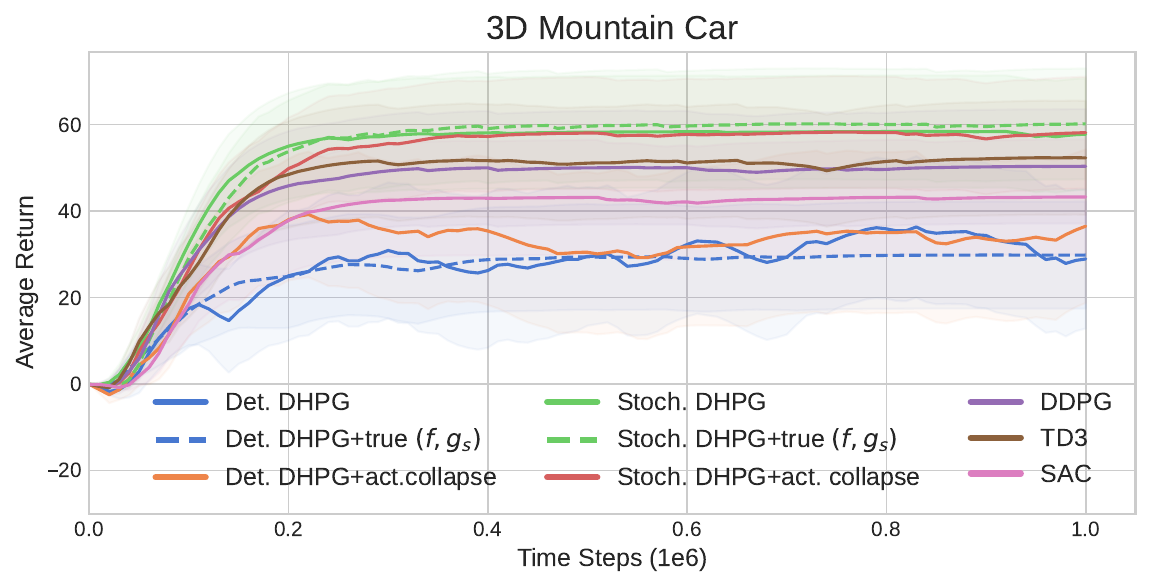}
    \caption{Learning curves for the 3D mountain car environment obtained on 50 seeds. Shaded regions represent $95\%$ confidence intervals. Stochastic DHPG outperforms deterministic DHPG due to the continuous translational symmetry of the environment. DHPG+true $(f, g_s)$ uses the ground truth homomorphism map, and DHPG+act. collapse removes has 1D abstract action space, as opposed to the original 2D action space.}
    \label{fig:mountain_car_res}
\end{figure}

\paragraph{Stochastic DHPG outperforms deterministic DHPG in the environments with continuous symmetries.} Results are presented in Figures \ref{fig:rotate_suite_res} and \ref{fig:mountain_car_res}. Notably, stochastic DHPG outperforms other baselines, as well as deterministic DHPG on environments with continuous symmetries. This is due to the fact that in theory, stochastic DHPG does not impose any structures on the action encoder and is therefore able to achieve higher levels of action abstraction, compared to deterministic DHPG.

\paragraph{DHPG is able to learn a structured latent space that adheres to the properties of a 3D rotation.} Visualizations of latent state trajectories in the cylinder rotation task are presented in Figure \ref{fig:rotate_cylinder_vis}. Each trajectory, color-coded by the action dimension, is collected by applying a constant rotation around one of the main axes (pitch, roll, and yaw). Interestingly, the latent trajectories of DHPG are fully disentangled and resemble 3D rotations in the latent space. However, none of the other baselines exhibits such structure in their latent representation. 

\begin{figure}[t!]
     \centering
     \begin{subfigure}[b]{0.19\textwidth}
         \centering
         \includegraphics[width=\textwidth]{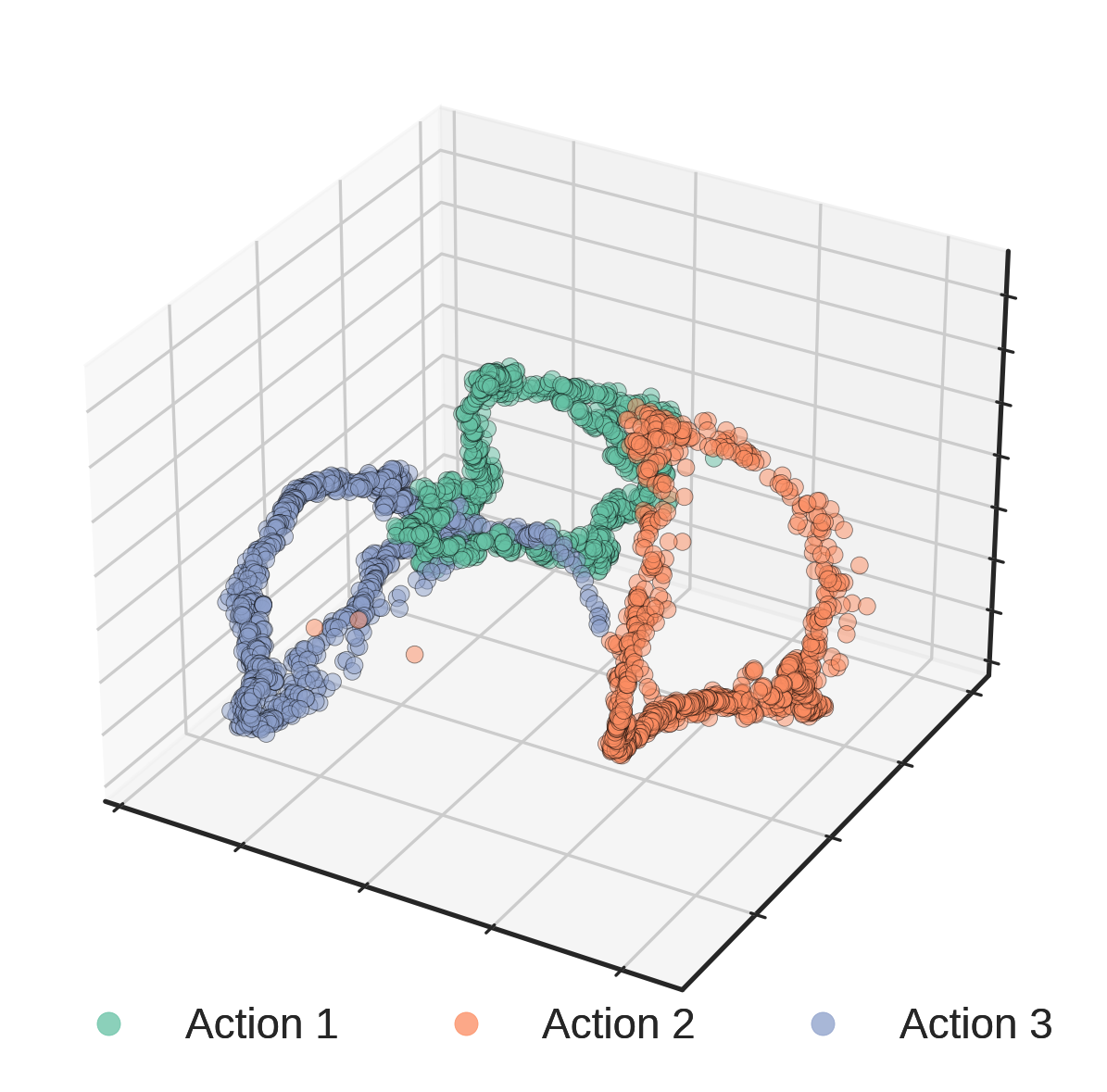}
         \caption{Stoch. DHPG.}
     \end{subfigure}
     \begin{subfigure}[b]{0.19\textwidth}
         \centering
         \includegraphics[width=\textwidth]{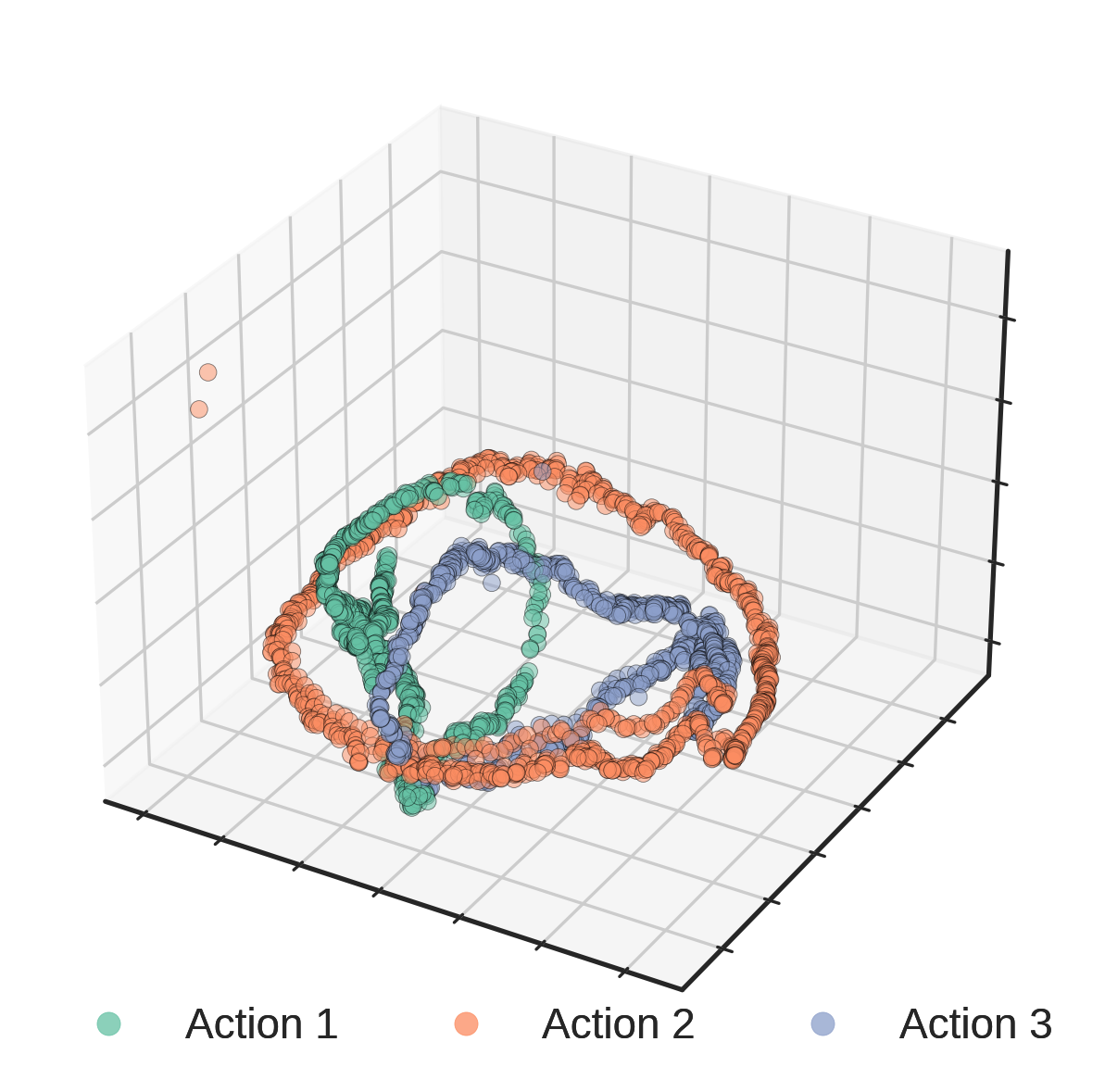}
         \caption{Det. DHPG.}
     \end{subfigure}
     \begin{subfigure}[b]{0.19\textwidth}
         \centering
         \includegraphics[width=\textwidth]{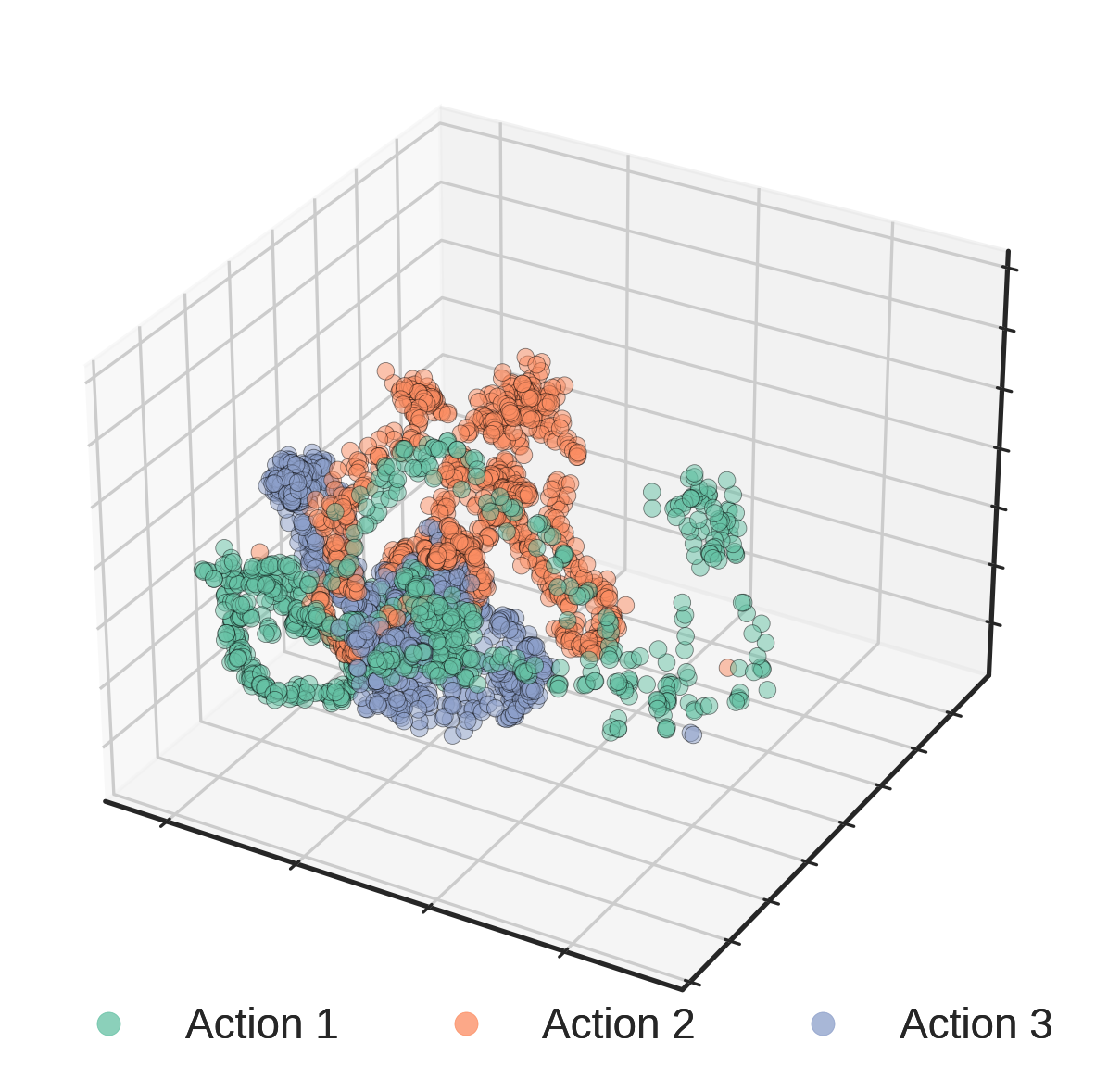}
         \caption{DrQ-v2.}
     \end{subfigure}
     \begin{subfigure}[b]{0.19\textwidth}
         \centering
         \includegraphics[width=\textwidth]{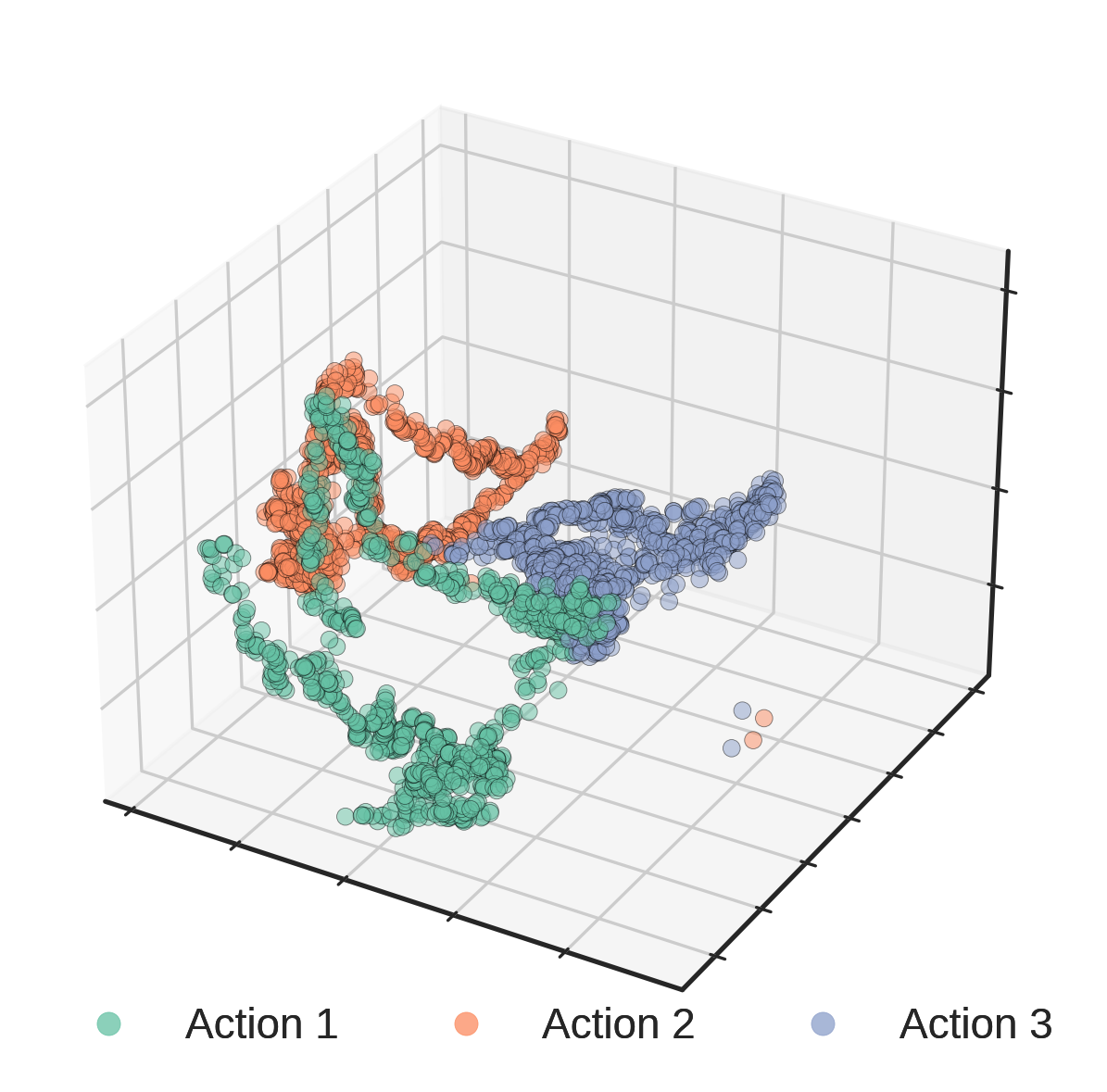}
         \caption{DBC.}
     \end{subfigure}
     \begin{subfigure}[b]{0.19\textwidth}
         \centering
         \includegraphics[width=\textwidth]{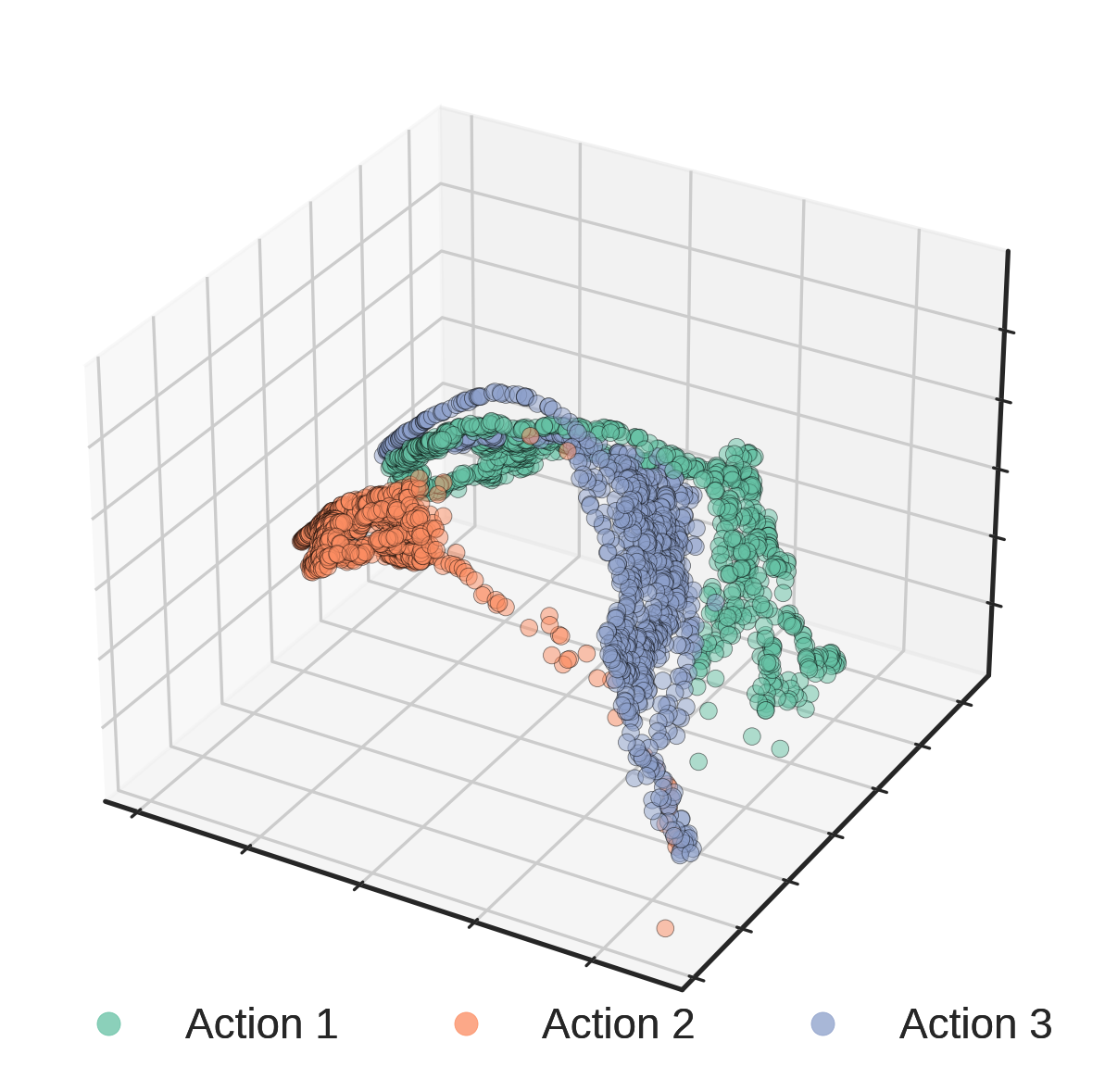}
         \caption{DeepMDP.}
     \end{subfigure}
    \caption{Visualization of latent state trajectories in the cylinder rotation task, color-coded by the action dimension. Action dimensions respectively correspond to \emph{pitch}, \emph{roll}, and \emph{yaw} and each trajectory is collected by applying a constant rotation around a specific axis. Latent states are projected into a 3D space using PCA.}
    \label{fig:rotate_cylinder_vis}
\end{figure}

\section{Conclusion}
\label{sec:conclusion}
In this paper, we developed the novel theory of continuous MDP homomorphisms using measure theory, and we rigorously proved their value and optimal value equivalence properties. We derived the homomorphic policy gradient for both stochastic and deterministic policies, in order to directly use a joint state-action abstraction for policy optimization. Importantly, we rigorously proved that applying our homomorphic policy gradient on the abstract MDP is equivalent to applying the standard policy gradient on the actual MDP. Based on our novel theoretical results, we developed a family of deep actor-critic algorithms, with either stochastic or deterministic policies, that can simultaneously learn the policy and the MDP homomorphism map using the lax bisimulation metric. Our algorithm, referred to as Deep Homomorphic Policy Gradient (DHPG) improves upon strong baselines in challenging visual control problems. The visualization of the latent space demonstrates the strong potential of MDP homomorphisms in learning structured representations that can  preserve value functions. Finally, we introduced a series of environments with continuous symmetries to further demonstrate the ability of our algorithm for action abstraction in the presence of continuous symmetries.

We believe that our work will open-up future possibilities for the application of MDP homomorphisms in challenging continuous control problems by enabling other RL algorithms to benefit from the abstraction power of MDP homomorphisms and the homomorphic policy gradient theorems.

\acks{SRS is supported by an NSERC CGS-D scholarship. RZ was supported by an NSERC CGS-M scholarship at the time this work was completed. PP is supported by a research grant from NSERC.  The computing resources for this research were provided by Calcul Quebec and the Digital Research Alliance of Canada.}

\clearpage  


\bibliography{refs}

\newpage
\appendix
\section{Assumptions and Conditions}
\label{sec:assumptions}
The derivation of our homomorphic policy gradient theorem is for continuous state and action spaces. Therefore, we have assumed the following regularity conditions on the actual MDP ${\mathcal{M}}$ and its MDP homomorphic image ${\overline{\mathcal{M}}}$ under the MDP homomorphism map $h$. The conditions are largely based on the regularity conditions of the deterministic policy gradient theorem \citep{silver2014deterministic}:

\paragraph{Regularity conditions 1:} $\tau_a(s' | s)$, $\nabla_a \tau_a(s' | s)$, $\overline{\tau}_{\overline{a}}(\overline{s}' | \overline{s})$, $\nabla_{\overline{a}} \overline{\tau}_{\overline{a}}(\overline{s}' | \overline{s})$, $R(s, a), \nabla_a R(s, a)$, $\overline{R}(\overline{s}, \overline{a}), \nabla_{\overline{a}} \overline{R}(\overline{s},  \overline{a})$, $\pi^\uparrow_\theta(s), \nabla_\theta \pi^\uparrow_\theta(s), \overline{\pi}_\theta(\overline{s})$, $\nabla_{\theta} \overline{\pi}_\theta(\overline{s})$, $p_1(s)$, and $\overline{p}_1(\overline{s})$ are continuous with respect to all parameters and variables $s, \overline{s}, a, \overline{a}, s'$, and $\overline{s}'$.

\paragraph{Regularity conditions 2:} There exists a $b$ and $L$ such that $\sup_s p_1(s) \!<\! b$, $\sup_{\overline{s}} \overline{p}_1(\overline{s}) < b$, $\sup_{a, s, s'} \tau_a(s'|s) < b$, $\quad\sup_{\overline{a}, \overline{s}, \overline{s}'} \overline{\tau}_{\overline{a}}(\overline{s}'|\overline{s}) < b$, $\quad\sup_{a, s} R(s, a) < b$, $\sup_{\overline{a}, \overline{s}} \overline{R}(\overline{s}, \overline{a}) < b$, $\sup_{a, s, s'}\|\nabla_a \tau_a(s'|s) \| < L$, $\quad\sup_{\overline{a}, \overline{s}, \overline{s}'}\|\nabla_{\overline{a}} \overline{\tau}_{\overline{a}}(\overline{s}'|\overline{s}) \| < L$, $\qquad\sup_{s, a} \| \nabla_a R(s, a)\| < L\qquad$, $\sup_{\overline{s}, \overline{a}} \| \nabla_{\overline{a}} \overline{R}(\overline{s}, \overline{a})\| < L$.

\paragraph{Regularity conditions 3:} The action mapping $g_s(a)$ is a local diffeomorphism (Definition \ref{def:local_diffeo}). Hence it is continuous with respect to $a$ and locally bijective with respect to $a$. Additionally, $\nabla_a g_s(a)$ is continuous with respect to the parameter $a$, and there exists a $L$ such that $\sup_{s, a} \| \nabla_a g_{s}(a)\| < L$.

\section{Mathematical tools}
\label{supp:math_tools}
Various mathematical concepts from measure theory and differential geometry are
briefly presented in this section.  We only explicitly introduce concepts which are
directly mentioned or relevant to the proofs presented in Sections
\ref{sec:cont_mdp_homomorphism} and \ref{sec:homomorphic_pg}; for a more
comprehensive overview, we direct the reader to textbooks such as
\citet{bogachev2007measure, lang2012differential, spivak2018calculus}. 

\subsection{Metric spaces and topology}

A set equipped with a metric is called a metric space and is usually written as a pair,
typically \((X,d)\).  Given a metric one can define standard notions from basic analysis:
convergent sequence, limit of a sequence, Cauchy sequence and continuous function.  If
every Cauchy sequence converges we say the metric space is \emph{complete}.  Limits of
convergent sequences are unique in proper metric spaces but not in pseudometric spaces.  A
function $f:(X,d)\to (Y,d')$ between metric spaces is said to be \emph{nonexpansive} if:
\[ \forall x,x'\in X, d'(f(x),f(x'))\leq d(x,x'). \]
A function is said to be \emph{contractive} if there is some number \(c\in (0,1)\) such
that:
\[ \forall x,x'\in X, d'(f(x),f(x')) < c\cdot d(x,x').\]
The fundamental theorem about metric spaces, called the \emph{Banach fixed-point theorem},
states the following:
\begin{theorem}
If $f$ is a contractive function from a complete metric space $X$ to itself then there is
a \emph{unique} fixed point for $f$, i.e.\ a unique point $x_0\in X$ such that \(f(x_0)=x_0\).
\end{theorem}

We assume that the readers are familiar with basic concepts of topology: open
and closed sets, base and subbase, convergence of sequences and continuity of
functions.

A topological space is \emph{completely metrizable} if it can be equipped with a metric which generates its topology and the resulting metric space is complete. A topological space is \emph{separable} if it contains a countable, dense subset--- that is, every nonempty subset in the topological space contains at least one element of this subset. A \emph{Polish space} is a topological space that is separable and completely metrizable. Polish spaces have ``desirable properties'' and are used primarily in areas of descriptive set theory and measure theory.

Another fundamental concept of topological spaces is \emph{compactness}.  An \emph{open cover} of a
topological space $X$ is a family of open subsets whose union includes all of
$X$.  A \emph{subcover} of a cover is a subcollection of the open sets in the
cover that also covers $X$.  A topological space is said to be \emph{compact} if
every open cover has a finite subcover.  In metric spaces, the compact sets are
exactly the closed and bounded sets.  A space is said to be \emph{locally
  compact} if every point is contained in an open set that is contained in a
compact set.

\subsection{Measure theory}
Measure theory attempts to generalize notions of length, area and volume or
mass to more complicated subsets than the simple ones that one first encounters
in geometry.  In common situations, like the real line, it is not possible to
define a measure on \emph{all} subsets in such a way that one's normal
intuitions of length survive.  In the real line there is no measure that is
defined on all subsets and which assigns to all intervals its length.  One
needs, therefore, to choose well behaved families of subsets on which one can
define measures.

\begin{definition}[$\sigma$-algebra]
    Given a set $X$, a $\sigma$-algebra on $X$ is a family $\Sigma$ of subsets
    of $X$ such that 1) $X \in \Sigma$, 2) $A \in \Sigma$ implies $A^c \in
    \Sigma$ (closure under complements), and 3) if $(A_i)_{i \in \N}$ satisfies
    $A_i \in \Sigma$ for all $i \in \N$, then $\cup_{i \in \N}A_i \in \Sigma$
    (closure under countable union). The tuple $(X, \Sigma)$ is a measurable
    space. 
\end{definition}
A set equipped with a $\sigma$-algebra is called a \emph{measurable} space.
The $\sigma$-algebra of a space specifies the sets for which a measure can be
defined; in probability theory---and in our use case---a $\sigma$-algebra
represents a collection of events which can be assigned probabilities.

Given any family of subsets there is a \emph{smallest} $\sigma$-algebra that
includes the given family: this is called the $\sigma$-algebra generated by the
family.  In metric spaces the $\sigma$-algebra generated by the open sets is
called the \emph{Borel} $\sigma$-algebra.

Given a $\sigma$-algebra one can define a measure.
\begin{definition}
  A (probability) \emph{measure} $\mu$ on a $\sigma$-algebra $\Sigma$ defined on $X$ is a
  function \(\mu:\Sigma\to [0,\infty]\) (\(\mu:\Sigma\to [0,1]\)) satisfying:
  \begin{itemize}
  \item \(\mu(\emptyset) =0\)
  \item If $\{A_i\}_{i\in\N}$ is a countable family of pairwise disjoint subsets
    in $\Sigma$, then
    \(\mu(\cup_iA_i) = \sum_{i\in\N}\mu(A_i)\).
\end{itemize}
For a probability measure we require $\mu(X) = 1$.
\end{definition}

The functions that play a key role are called measurable functions.
\begin{definition}
A \emph{measurable} function or map between two measurable spaces $(X\Sigma)$ and
$(Y,\Lambda)$ is a function $f:X\to Y$ such that for any $B\in\Lambda$,
\(f^{-1}(B)\in\Sigma\). 
\end{definition}

We can define measures on the image of a measurable function based on a measure in the preimage. This yields the definition of the pushforward measure and the change of variables formula which is crucial for our proofs in switching the domain of integration across a measurable function.

\begin{definition}[Pushforward measure]
Let $(X_1, \Sigma_1)$ and $(X_2, \Sigma_2)$ be two measurable spaces, $f: X_1
\to X_2$ a measurable map and $\mu: \Sigma_1 \to [0, \infty]$ a measure on
$X_1$. Then the pushforward measure of $\mu$ with respect to $f$, denoted
$f_*(\mu): \Sigma_2 \to [0, \infty]$ is defined as: 
$$
    (f_*(\mu))(B) = \mu(f^{-1}(B)) \; \forall \; B \in \Sigma_2.
$$
\end{definition}

\begin{theorem}[Change of variables]
\label{thm:cov}
A measurable function $g$ on $X_2$ is integrable with respect to $f_*(\mu)$ if
and only if the function $g \circ f$ is integrable with respect to $\mu$, in
which case the integrals are equal: 
$$
    \int_{X_2} g d(f_*(\mu)) = \int_{X_1}g \circ f d\mu.
$$
\end{theorem}

\subsection{Manifolds}

Differential manifolds formalize doing differential calculus on
curved surfaces.  Unlike vector spaces, we cannot define addition of
points in an arbitrary topological space; we need additional structure.  Hence the
strategy is to define ``patches'' of the topological space that ``look
like'' patches of a vector space and then glue them together.  This
motivates the definition of a differential or smooth manifold.  The word
``smooth'' is a synonym for \emph{infinitely differentiable} or $C^{\infty}$.
\begin{definition}
An $n$-dimensional \textbf{smooth (or differential) manifold} is a topological
space $M$\footnote{Assumed to be paracompact and Hausdorff.} equipped with a family of pairs, called \textbf{charts},
\(\{(U_i,\phi_i)|i\in \cA\}\) where:  
\begin{itemize}
\item Each $U_i$ is an open subset of $M$,
\item each $\phi_i:U_i\to \reals^n$ is a \emph{homeomorphism} between $U_i$ and the image
  $V_i := \phi_i(U_i)$,
\item the \(\{U_i\}\) form an open cover of $M$.
\end{itemize}
In addition, the following \emph{compatibility condition} must be satisfied:\\
if \(U_i \cap U_j\not=\emptyset\) then the map
\[\phi_j\circ\phi_i^{-1}\bigg|_{\phi_i(U_i\cap U_j)}:\phi_i(U_i\cap U_j)\to\phi_j(U_i\cap
  U_j)\]
is infinitely differentiable, written as $C^{\infty}$.  A collection of compatible charts
is called an \textbf{atlas}.
\end{definition}
\begin{figure}[H]
\begin{center}
\includegraphics[height=2.0in,width=3.0in]{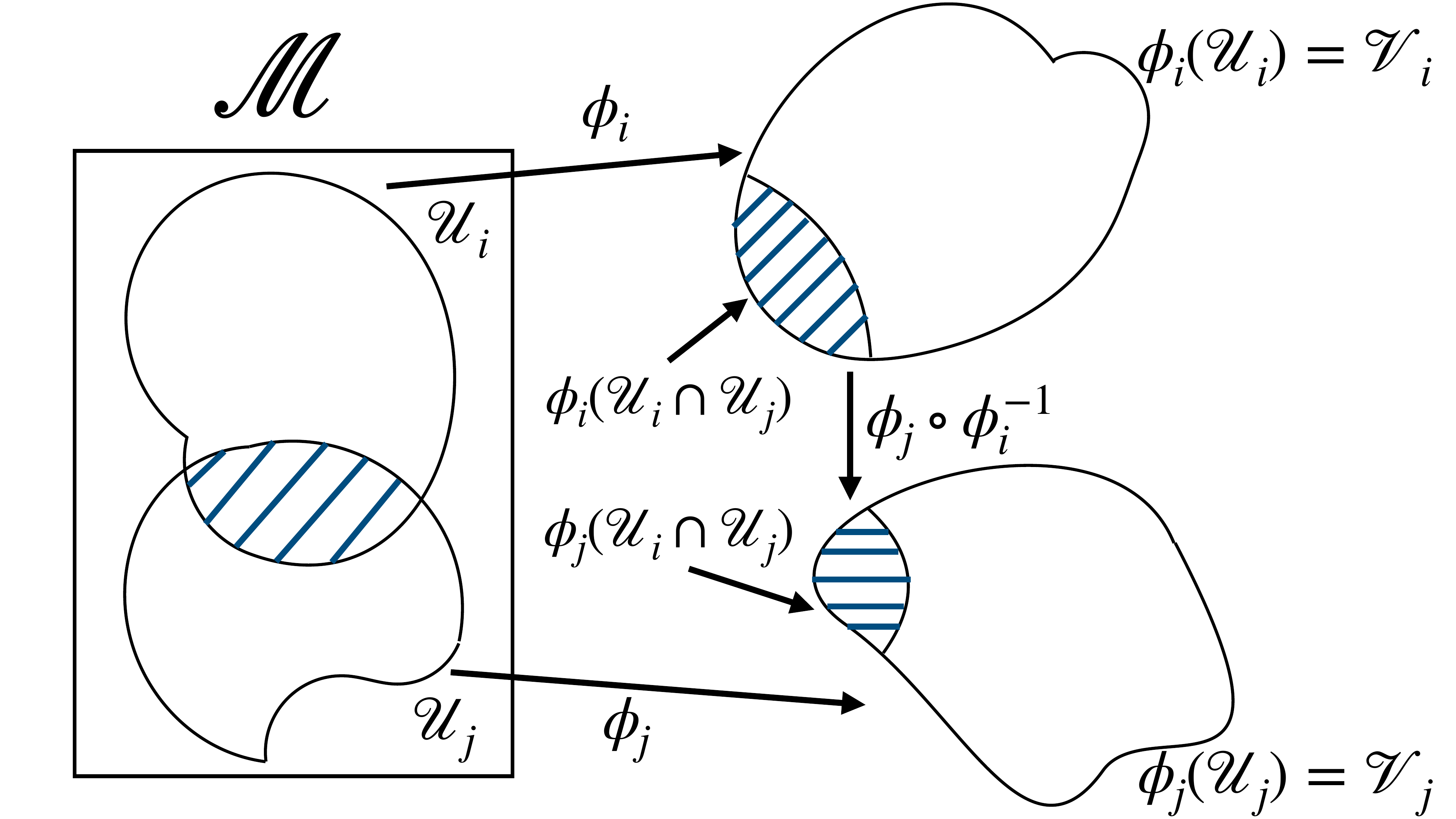}
\caption{The compatibility condition on charts.}
\label{chart}
\end{center}
\end{figure}

Note that the very last condition refers to a map between an open set in
$\reals^n$ and another open set in $\reals^n$; hence its meaning is clear since
$\reals^n$ is a vector space.  The picture in Fig.~\ref{chart} illustrates the
meaning of this condition.  Now everything can be defined in terms of charts.
\begin{definition}
A \emph{smooth function}
$f:M\to\reals$ is a function such that for any chart \((U,\phi)\) the map
\(f\circ\phi^{-1}:\reals^n\to\reals\)\footnote{Actually this is only defined on
  $\phi(U)$ not all of $\reals^n$ but it would clutter the notation too much to
  constantly put in the correct restrictions.} is smooth.
\end{definition}
A smooth map between manifolds $M$ and $M'$ can be defined similarly.
\begin{definition}
A \emph{smooth map}
$f:M\to M'$ is a function such that for any chart \((U,\phi)\) of $M$ and
\((U',\phi')\) of $M'$ the map
\(\phi'\circ\psi\circ\phi^{-1}:\reals^n\to\reals^{n'}\) is smooth.
\end{definition}
Smooth maps and smooth functions are automatically continuous.

Differential manifolds come with a notion of isomorphism called a
\emph{diffeomorphism}.
\begin{definition}
A smooth map between two manifolds is called a \emph{diffeomorphism} if it is a
bijection and the inverse map is also smooth.
\end{definition}

\begin{definition}[Local diffeomorphism]
\label{def:local_diffeo}
Let $M$ and $N$ be differentiable manifolds. A function $f : M \to N$ is a
\emph{local diffeomorphism}, if for each point $x \in M$ there exists an open
set $U$ containing $x$ such that $f(U)$ is open in $N$ and $f|_U : U \to f(U)$
is a diffeomorphism. 
\end{definition}

Once the structure of a smooth manifold is in place one can define the notion of
derivative operator.  The tangent to a curve can also be defined in terms of
differentiation.  A tangent vector $t$ at a point $x$ should be thought of as a
directional derivative.  Standard results from multivariable calculus can be invoked to
show that the set of tangent vectors at a point form an $n$-dimensional vector space.  One
writes $T_x$ for this vector space, which is called \emph{the tangent space} at $x$.

The cleanest way to axiomatize the concept of tangent vector at a point $x$ is as
follows.  Let $\cF$ be the set of smooth real-valued functions defined on $M$.
\begin{definition}
  Given a point $x$ of $M$ we define a \textbf{tangent vector at} $x$ to be a map
  $t:\cF\to\reals$ such that, \(\forall a,b\in\reals\) and \(\forall f,g\in\cF\):
\begin{enumerate}
\item \(t(af+bg) = at(f) + bt(g)\),
\item \(t(fg) = f(x)t(g) + g(x)t(f)\).
\end{enumerate}
\end{definition}
It follows immediately that if $f$ is a constant function then $t(f) = 0$.  Note how the
second condition makes specific reference to the point $x$.

A smooth map $\psi:M\to N$ induces a map between tangent spaces at the corresponding
points.  The \emph{differential} of the map $\psi$ at $x$ is the linear map
\(\mathrm{d}\psi:T_xM\to T_{\psi(x)}N\) defined as follows.  Let $g$ be a smooth function
on a neighbourhood of $\psi(x)$ and let $t$ be a tangent vector at $x$.  We have to define
a tangent vector at $\psi(x)$ so it should be able to act on $g$.  We define
\[ \mathrm{d}\psi(t)(g) := t(g\circ\psi). \]

The following theorems are fundamental and used in the proofs of the policy gradient theorems.
\begin{theorem}[Inverse function theorem for manifolds]
\label{thm:inv_function}
If $f : M \to N$ is a smooth map whose differential $df_x : T_x M \to T_{f(x)}N$ is an
isomorphism at a point $x \in M$.  Then $f$ is a local diffeomorphism at $x$.
\end{theorem}

\begin{theorem}[Chain rule for manifolds]
If $f : M \to N$ and $g : N \to O$ are smooth maps of manifolds, then:
$$d
    (g \circ f)_{x} = dg_{f(x)} \circ df_x.
$$
\end{theorem}

\clearpage
\section{Full Results}
\label{sec:additional_results}
As discussed in Section \ref{sec:experiments}, we evaluate DHPG on continuous control tasks from DM Control on pixel observations, as well as custom designed environments. Importantly, to reliably evaluate our algorithm against the baselines and to correctly capture the distribution of results, we follow the best practices proposed by \citet{agarwal2021deep} and report the interquartile mean (IQM) and performance profiles aggregated on all tasks over 10 random seeds. While our baseline results are obtained using the official code, when possible, some of the results may differ from the originally reported ones due to the difference in the seed numbers and our goal to present a faithful representation of the true performance distribution \citep{agarwal2021deep}. 

We use the official implementations of DrQv2, DBC, and SAC-AE, while we re-implement DeepMDP due to the unavailability of the official code; See Appendix \ref{sec:baseline_impl} for more details on the baselines.

\subsection{DeepMind Control Suite}
\label{sec:additional_results_pixels}
Figures \ref{fig:pixel_results_supp_aug}-\ref{fig:pixel_results_supp_no_aug} show full results obtained on 16 DeepMind Control Suite tasks with pixel observations to supplement the results of Section \ref{sec:results_pixels}. Domains that require excessive exploration and large number of time steps (e.g., acrobat, swimmer, and humanoid) and domains with visually small targets (e.g., reacher hard and finger turn hard) are not included in this benchmark.

Figures \ref{fig:pixel_aug_results_performance_profiles}-\ref{fig:pixel_no_aug_results_performance_profiles} and \ref{fig:pixel_aug_results_aggregate_metrics}-\ref{fig:pixel_no_aug_results_aggregate_metrics} respectively show performance profiles and aggregate metrics \citep{agarwal2021deep} on 14 tasks; hopper hop and walker run are removed from RLiable evaluation as none of the algorithms have acquired reasonable performance within 1 million time-steps. 

In Figures \ref{fig:pixel_results_supp_aug}, \ref{fig:pixel_aug_results_performance_profiles}, and \ref{fig:pixel_aug_results_aggregate_metrics} all methods are \emph{with} image augmentation, while in Figures \ref{fig:pixel_results_supp_no_aug}, \ref{fig:pixel_no_aug_results_performance_profiles}, and \ref{fig:pixel_no_aug_results_aggregate_metrics} all methods are \emph{without} image augmentation.

\clearpage

\vspace{5em}
\begin{figure}[h!]
     \centering
     \begin{subfigure}[b]{0.24\textwidth}
         \centering
         \includegraphics[width=\textwidth]{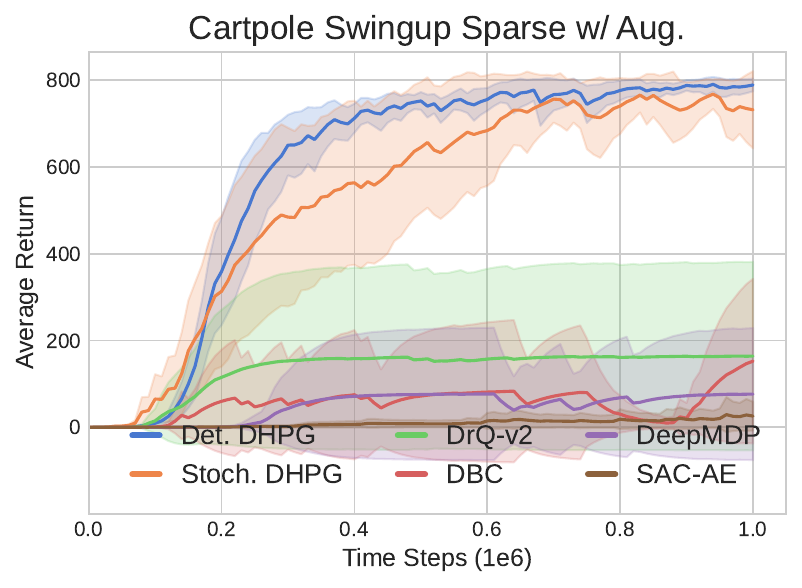}
     \end{subfigure}
     \hfill
     \begin{subfigure}[b]{0.24\textwidth}
         \centering
         \includegraphics[width=\textwidth]{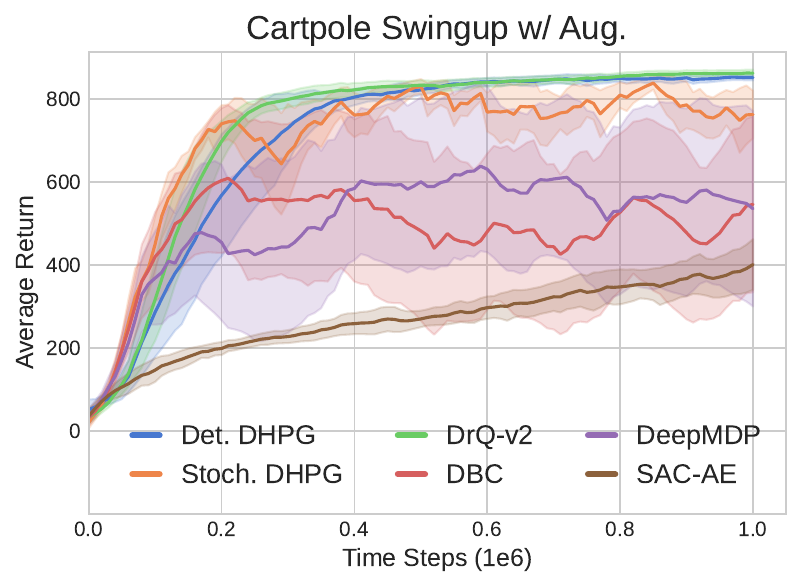}
     \end{subfigure}
     \hfill
     \begin{subfigure}[b]{0.24\textwidth}
         \centering
         \includegraphics[width=\textwidth]{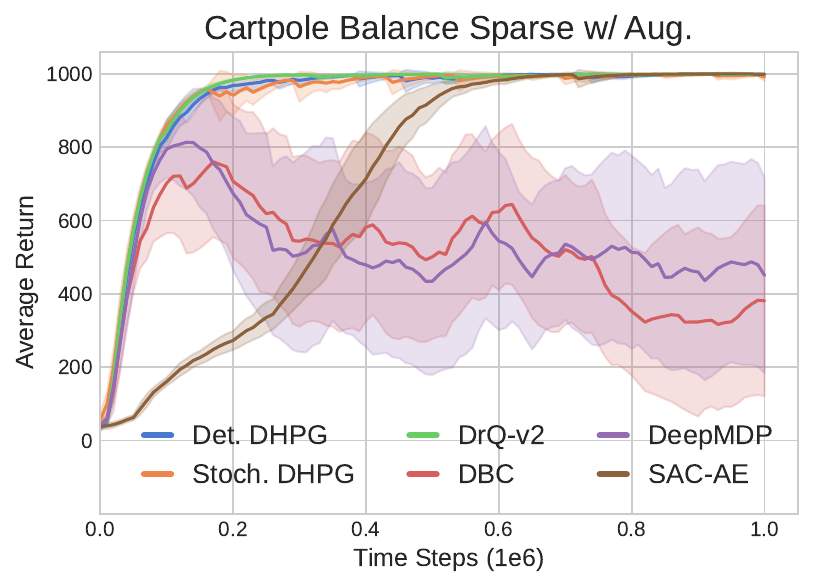}
     \end{subfigure}
     \hfill     
     \begin{subfigure}[b]{0.24\textwidth}
         \centering
         \includegraphics[width=\textwidth]{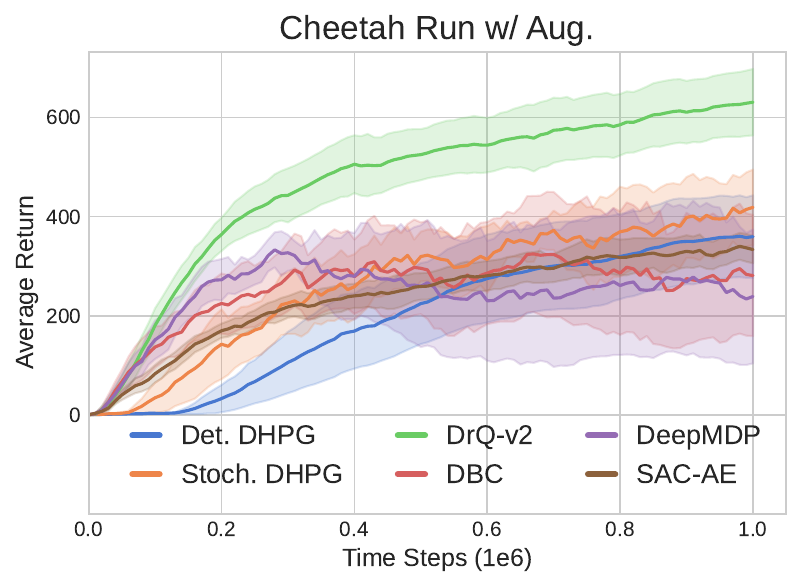}
     \end{subfigure}
     \hfill
     
     \begin{subfigure}[b]{0.24\textwidth}
         \centering
         \includegraphics[width=\textwidth]{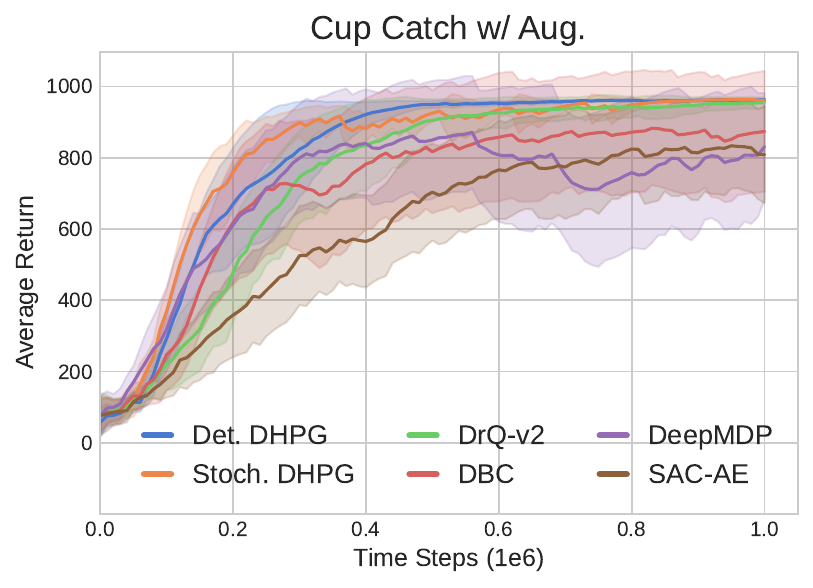}
     \end{subfigure}
     \hfill
     \begin{subfigure}[b]{0.24\textwidth}
         \centering
         \includegraphics[width=\textwidth]{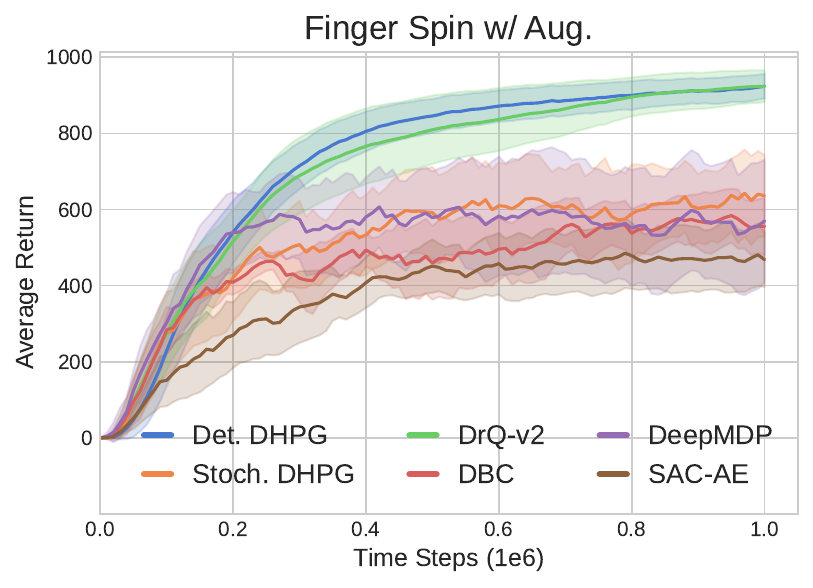}
     \end{subfigure}
     \hfill
     \begin{subfigure}[b]{0.24\textwidth}
         \centering
         \includegraphics[width=\textwidth]{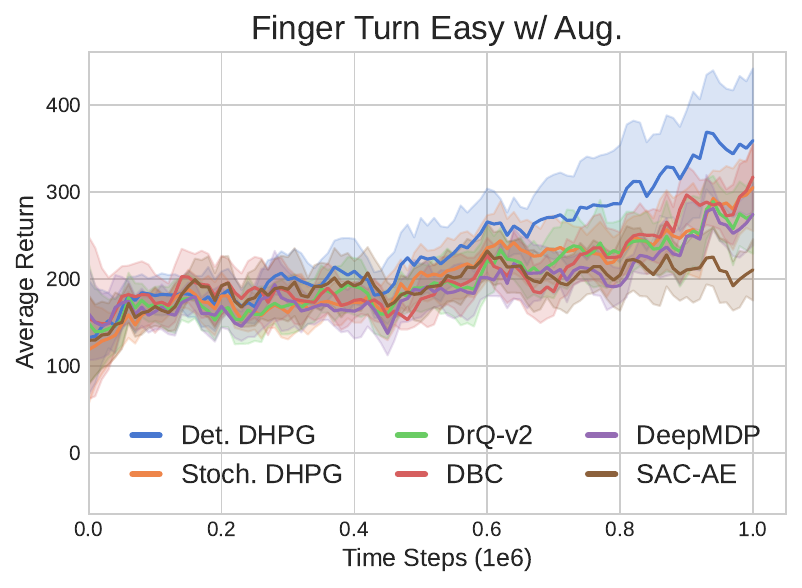}
     \end{subfigure}     
     \hfill
     \begin{subfigure}[b]{0.24\textwidth}
         \centering
         \includegraphics[width=\textwidth]{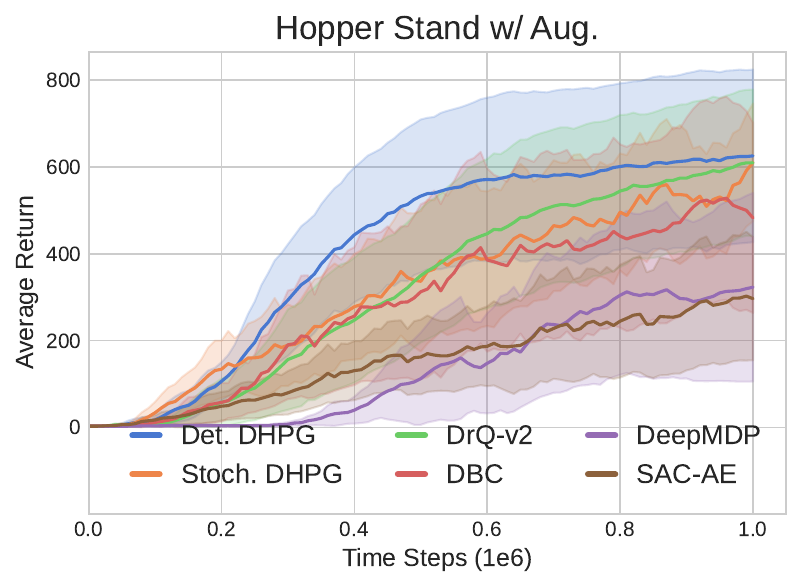}
     \end{subfigure}
     \hfill     
     
     \begin{subfigure}[b]{0.24\textwidth}
         \centering
         \includegraphics[width=\textwidth]{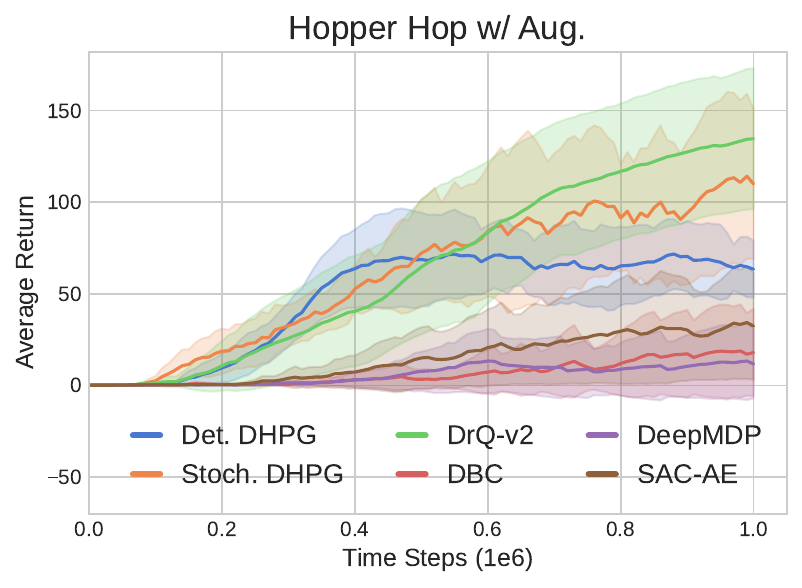}
     \end{subfigure}    
     \hfill
     \begin{subfigure}[b]{0.24\textwidth}
         \centering
         \includegraphics[width=\textwidth]{figures/pixels_aug_main/pixels_aug_pendulum_swingup_episode_reward_eval.pdf}
     \end{subfigure}
     \hfill
     \begin{subfigure}[b]{0.24\textwidth}
         \centering
         \includegraphics[width=\textwidth]{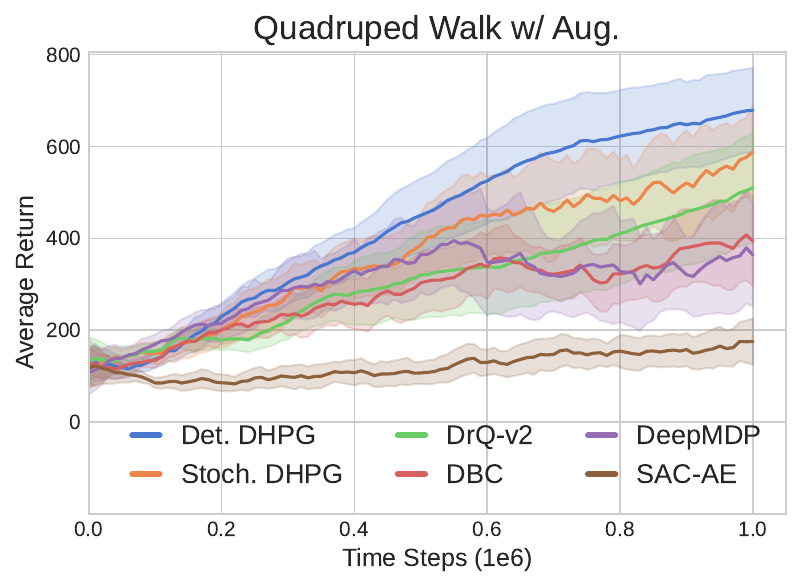}
     \end{subfigure}
     \hfill
     \begin{subfigure}[b]{0.24\textwidth}
         \centering
         \includegraphics[width=\textwidth]{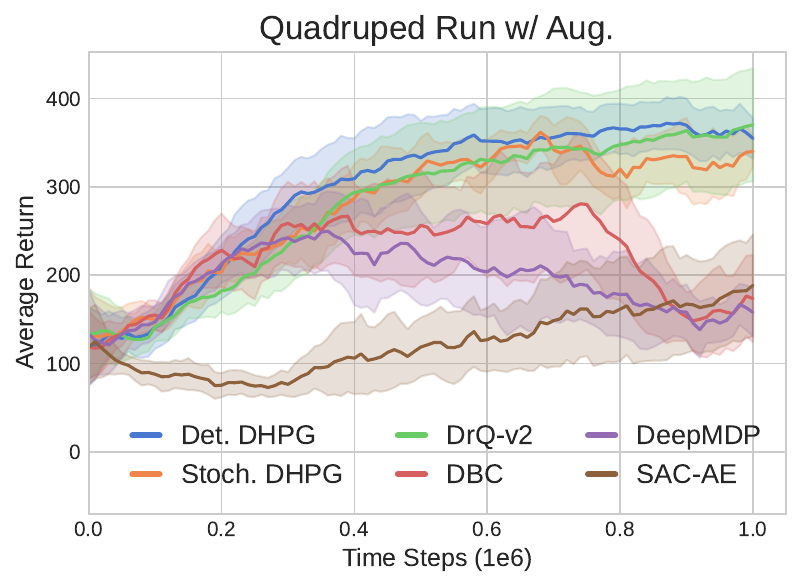}
     \end{subfigure}
     \hfill
     
     \begin{subfigure}[b]{0.24\textwidth}
         \centering
         \includegraphics[width=\textwidth]{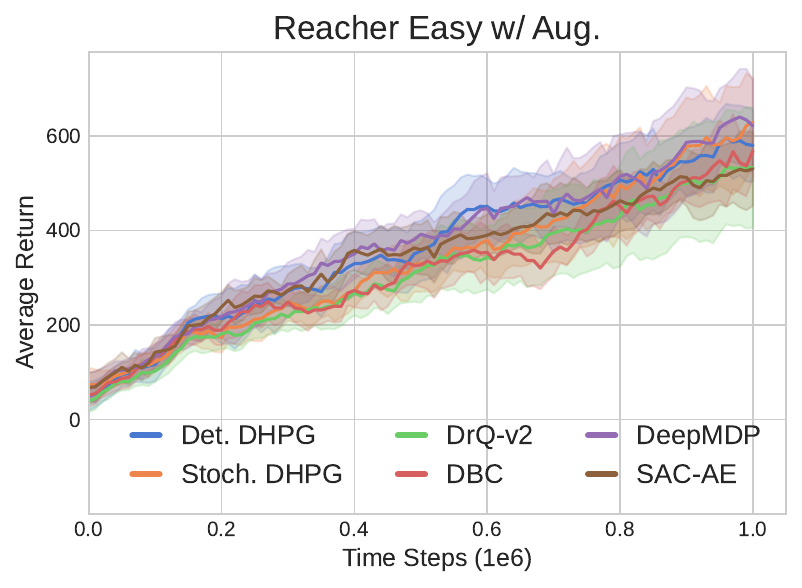}
     \end{subfigure}
     \hfill
     \begin{subfigure}[b]{0.24\textwidth}
         \centering
         \includegraphics[width=\textwidth]{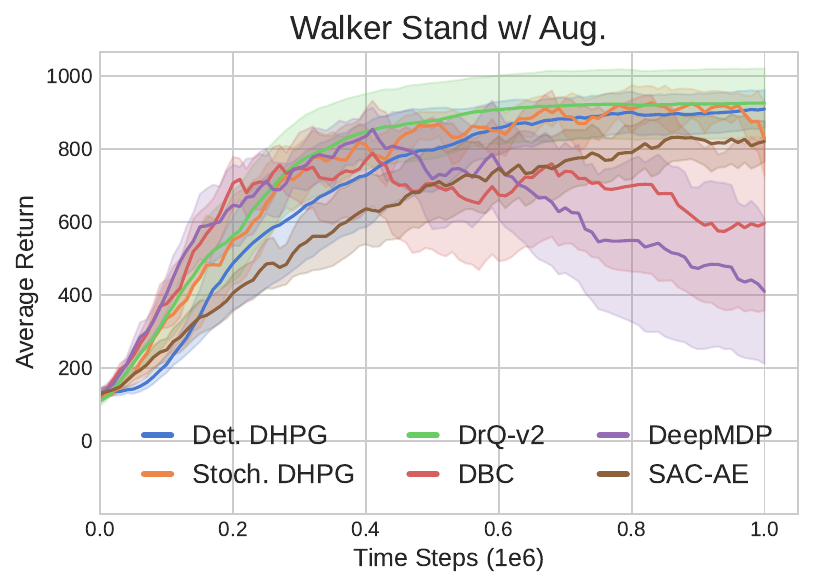}
     \end{subfigure}
     \hfill
     \begin{subfigure}[b]{0.24\textwidth}
         \centering
         \includegraphics[width=\textwidth]{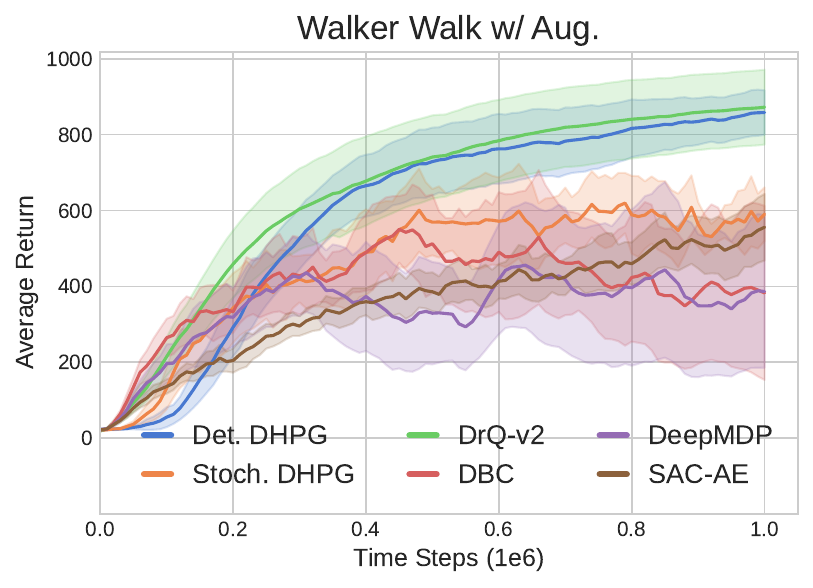}
     \end{subfigure}
     \hfill
     \begin{subfigure}[b]{0.24\textwidth}
         \centering
         \includegraphics[width=\textwidth]{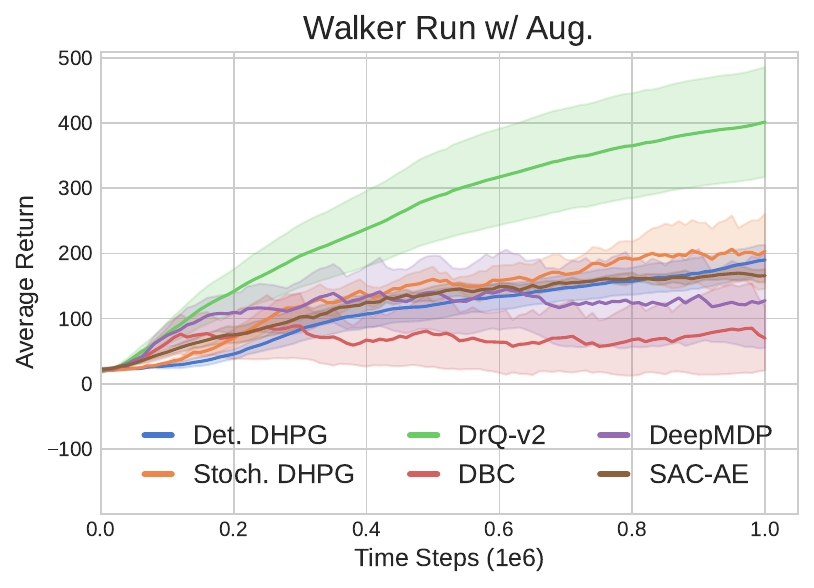}
     \end{subfigure}     
    \caption{Learning curves for 16 DM control tasks with pixel observations. All methods are \textbf{with} image augmentation. Mean performance is obtained over 10 seeds and shaded regions represent $95\%$ confidence intervals. Plots are smoothed uniformly for visual clarity.}
    \label{fig:pixel_results_supp_aug}
\end{figure}

\clearpage

\begin{figure}[h!]
     \centering
     \begin{subfigure}[b]{0.24\textwidth}
         \centering
         \includegraphics[width=\textwidth]{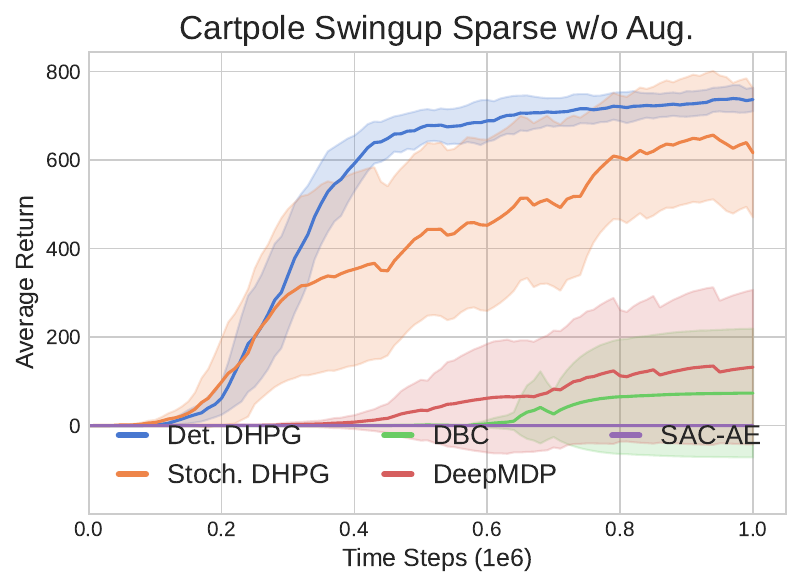}
     \end{subfigure}
     \hfill
     \begin{subfigure}[b]{0.24\textwidth}
         \centering
         \includegraphics[width=\textwidth]{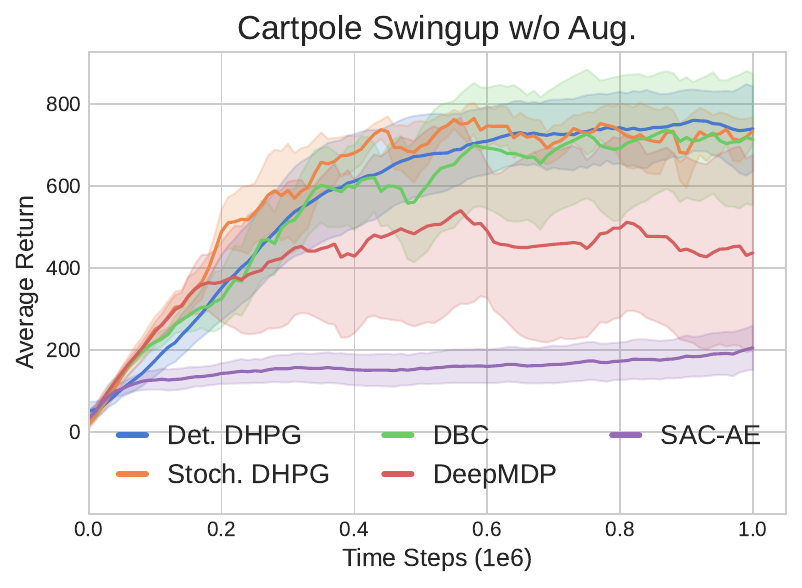}
     \end{subfigure}
     \hfill
     \begin{subfigure}[b]{0.24\textwidth}
         \centering
         \includegraphics[width=\textwidth]{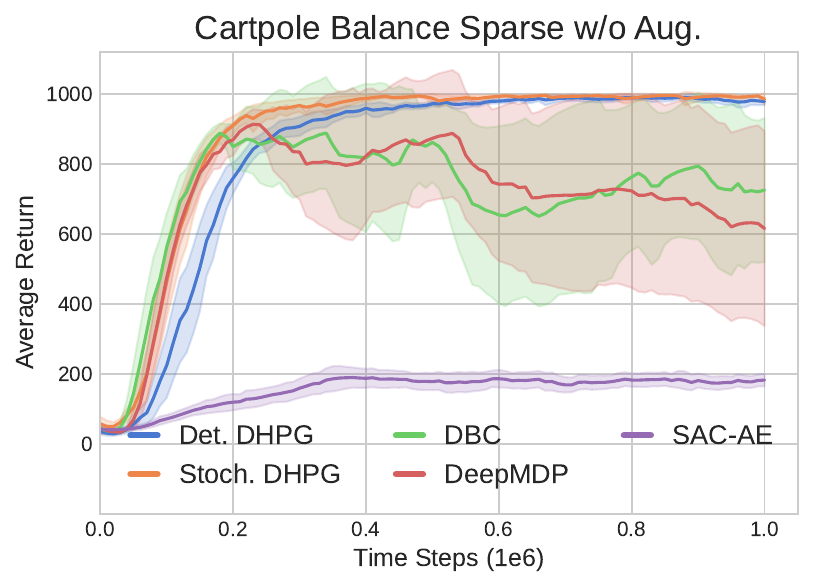}
     \end{subfigure}
     \hfill     
     \begin{subfigure}[b]{0.24\textwidth}
         \centering
         \includegraphics[width=\textwidth]{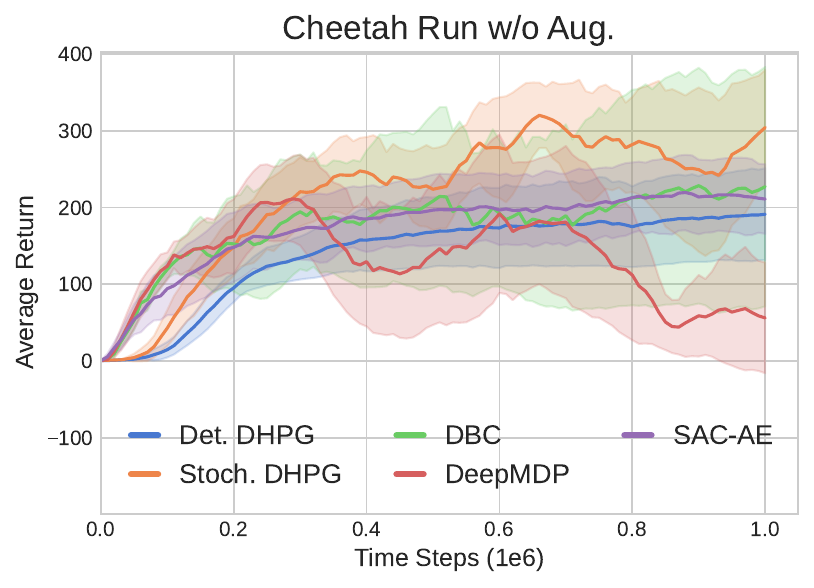}
     \end{subfigure}
     \hfill
     
     \begin{subfigure}[b]{0.24\textwidth}
         \centering
         \includegraphics[width=\textwidth]{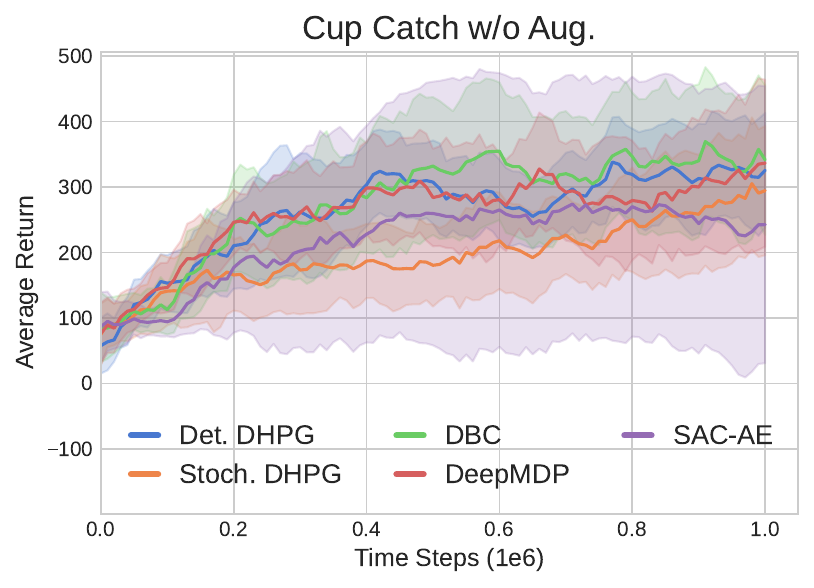}
     \end{subfigure}
     \hfill
     \begin{subfigure}[b]{0.24\textwidth}
         \centering
         \includegraphics[width=\textwidth]{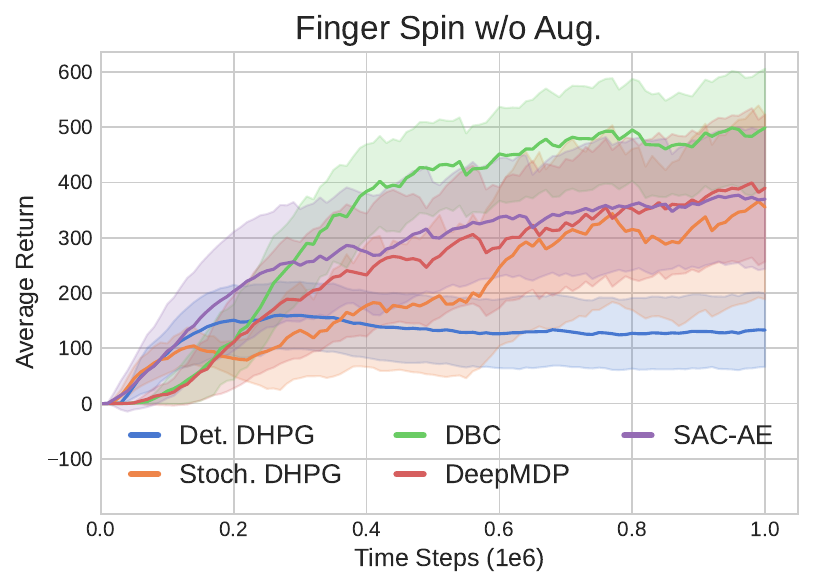}
     \end{subfigure}
     \hfill
     \begin{subfigure}[b]{0.24\textwidth}
         \centering
         \includegraphics[width=\textwidth]{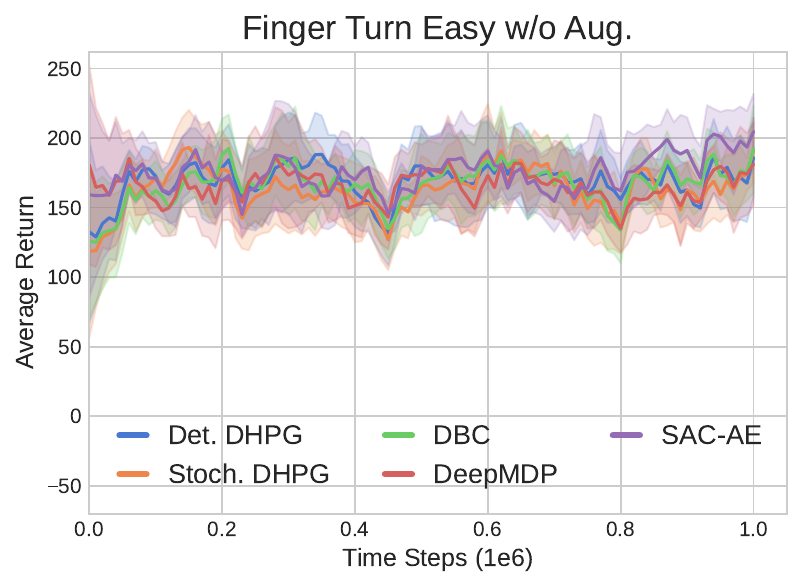}
     \end{subfigure}     
     \hfill
     \begin{subfigure}[b]{0.24\textwidth}
         \centering
         \includegraphics[width=\textwidth]{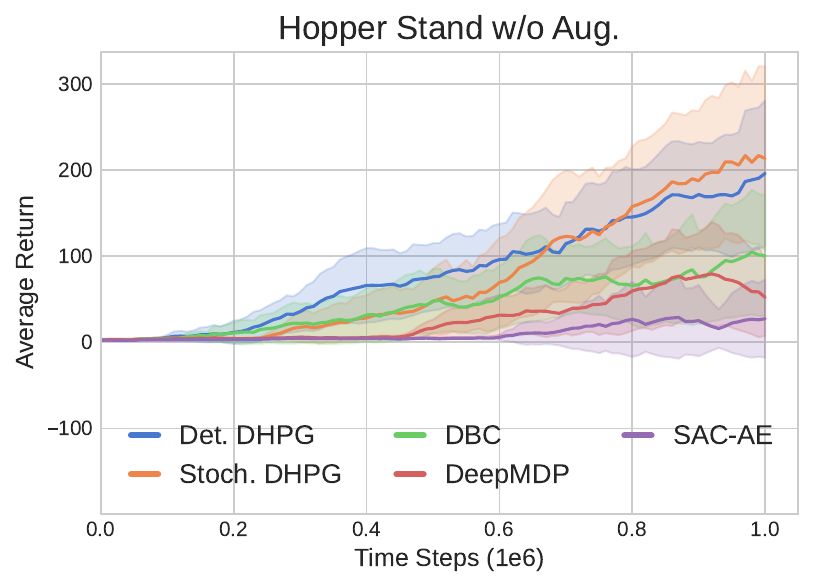}
     \end{subfigure}
     \hfill     
     
     \begin{subfigure}[b]{0.24\textwidth}
         \centering
         \includegraphics[width=\textwidth]{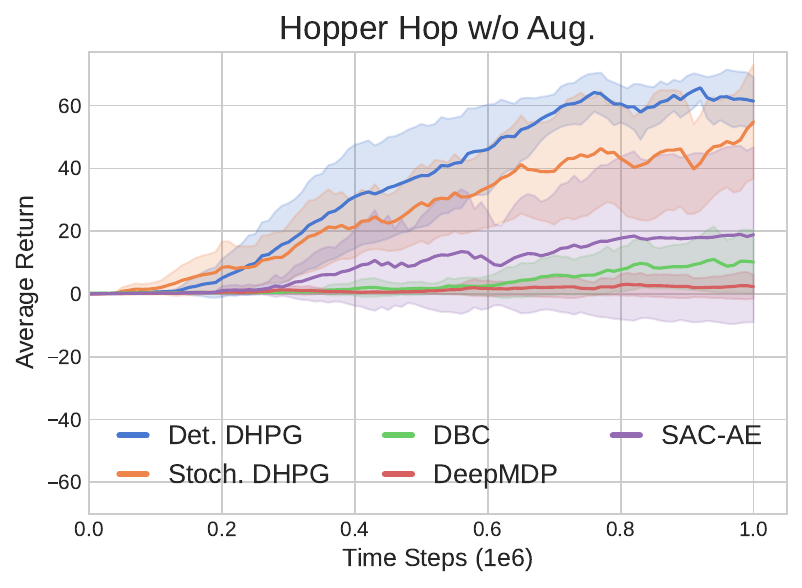}
     \end{subfigure}    
     \hfill
     \begin{subfigure}[b]{0.24\textwidth}
         \centering
         \includegraphics[width=\textwidth]{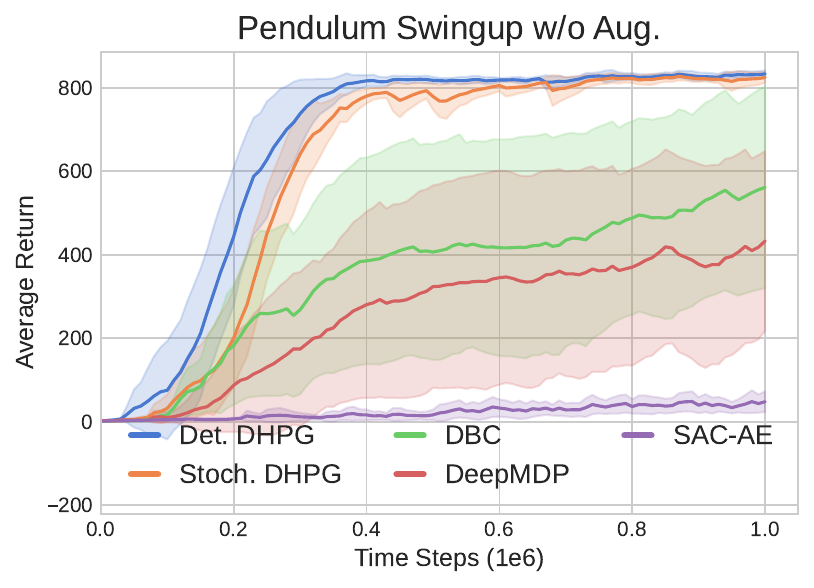}
     \end{subfigure}
     \hfill
     \begin{subfigure}[b]{0.24\textwidth}
         \centering
         \includegraphics[width=\textwidth]{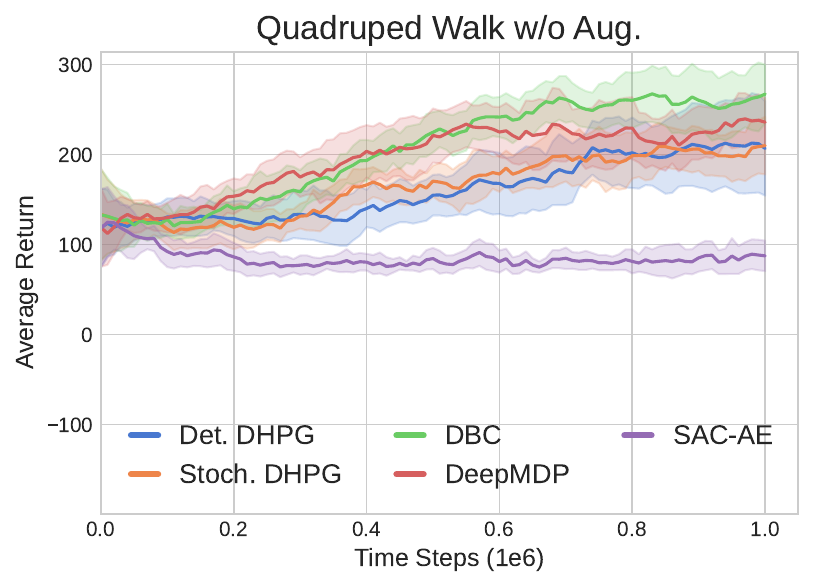}
     \end{subfigure}
     \hfill
     \begin{subfigure}[b]{0.24\textwidth}
         \centering
         \includegraphics[width=\textwidth]{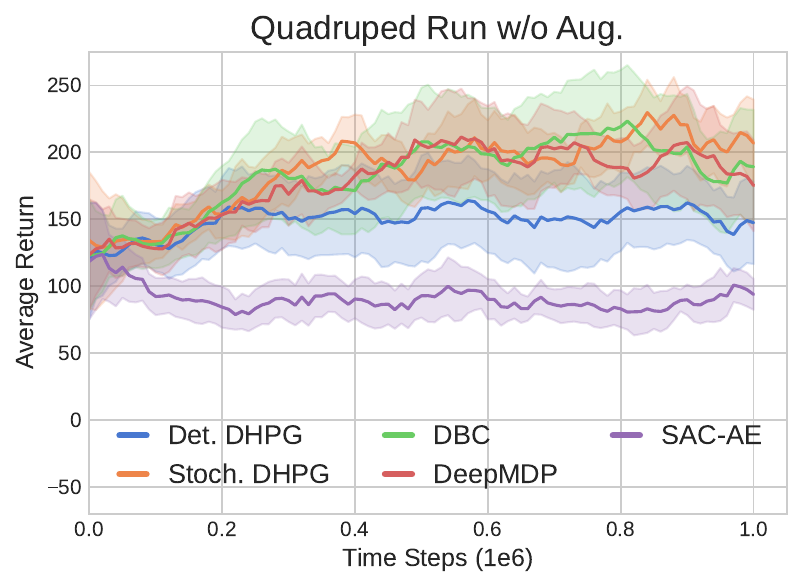}
     \end{subfigure}
     \hfill
     
     \begin{subfigure}[b]{0.24\textwidth}
         \centering
         \includegraphics[width=\textwidth]{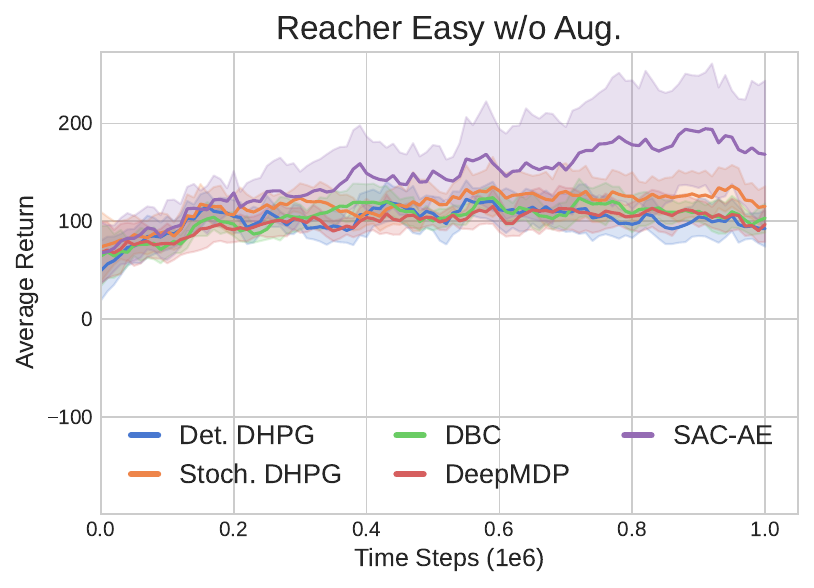}
     \end{subfigure}
     \hfill
     \begin{subfigure}[b]{0.24\textwidth}
         \centering
         \includegraphics[width=\textwidth]{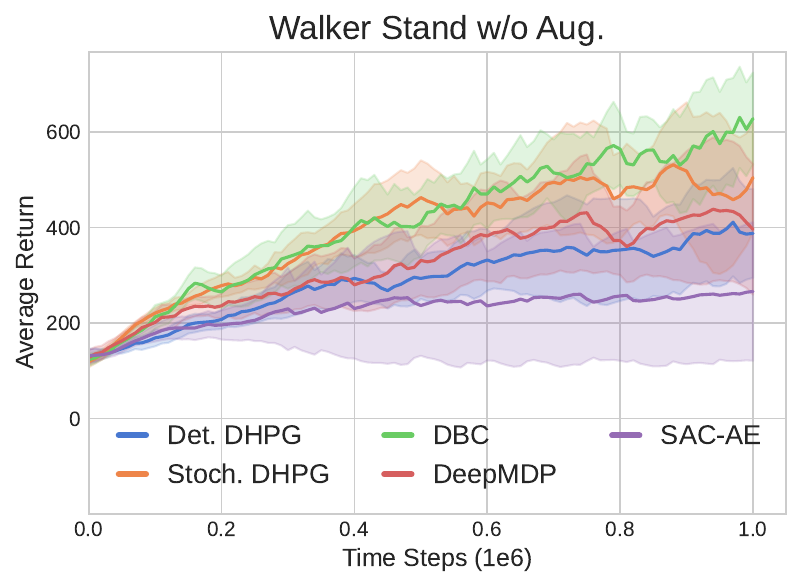}
     \end{subfigure}
     \hfill
     \begin{subfigure}[b]{0.24\textwidth}
         \centering
         \includegraphics[width=\textwidth]{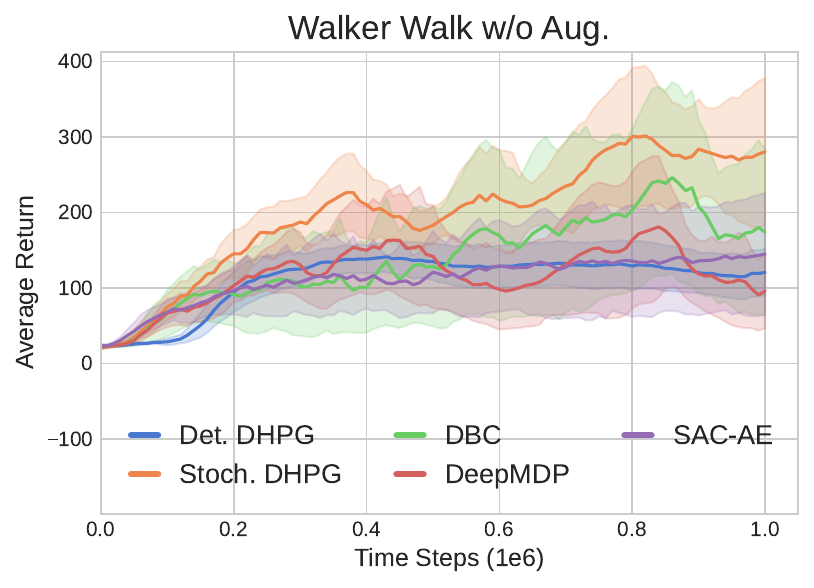}
     \end{subfigure}
     \hfill
     \begin{subfigure}[b]{0.24\textwidth}
         \centering
         \includegraphics[width=\textwidth]{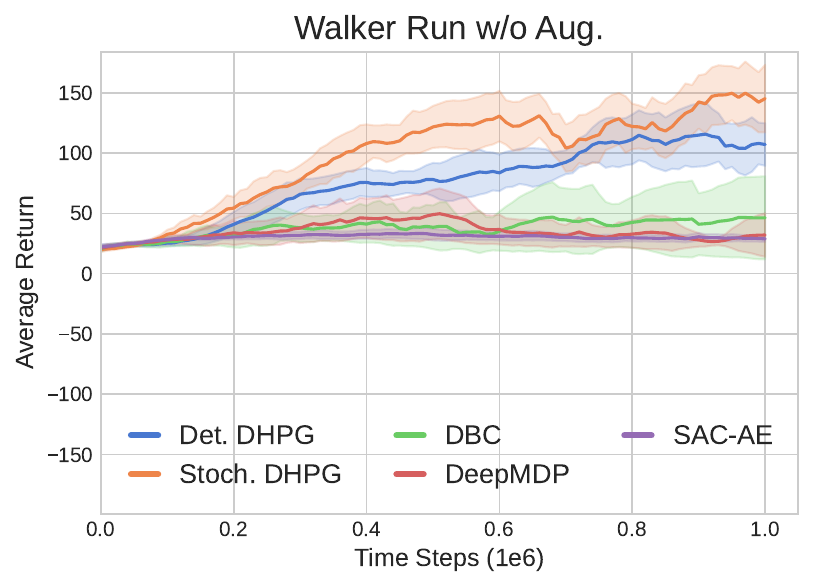}
     \end{subfigure}
     
    \caption{Learning curves for 16 DM control tasks with pixel observations. All methods are \textbf{without} image augmentation. Mean performance is obtained over 10 seeds and shaded regions represent $95\%$ confidence intervals. Plots are smoothed uniformly for visual clarity.}
    \label{fig:pixel_results_supp_no_aug}
\end{figure}

\clearpage

\begin{figure}[h!]
    \centering
    \begin{subfigure}[b]{0.45\textwidth}
        \includegraphics[width=\textwidth]{figures/pixels_aug_rliable/pixels_aug_performance_profiles_500k.pdf}
        \caption{500k step benchmark.}
    \end{subfigure}
    \hfill
    \begin{subfigure}[b]{0.45\textwidth}
        \includegraphics[width=\textwidth]{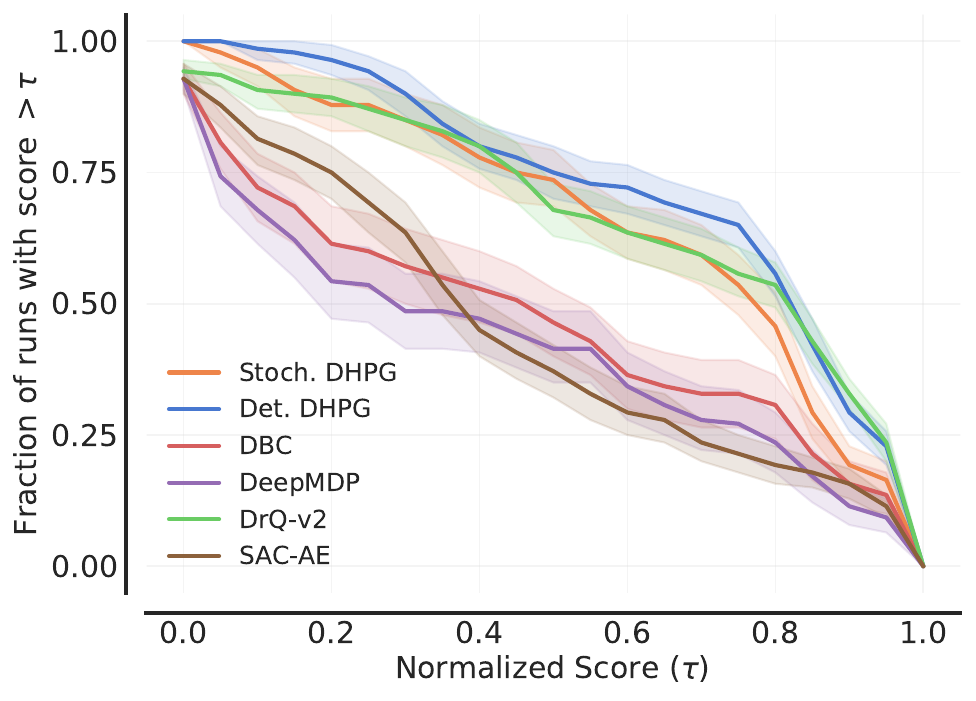}
        \caption{1m step benchmark.}
    \end{subfigure}
    \caption{Performance profiles for pixel observations based on 14 tasks over 10 seeds, at 500k steps \textbf{(a)}, and at 1m steps \textbf{(b)}. All methods are \textbf{with} image augmentation. Shaded regions represent $95\%$ confidence intervals.}    
    \label{fig:pixel_aug_results_performance_profiles}
\end{figure}

\vspace{5em}

\begin{figure}[h!]
    \centering
    \begin{subfigure}[b]{0.45\textwidth}
        \includegraphics[width=\textwidth]{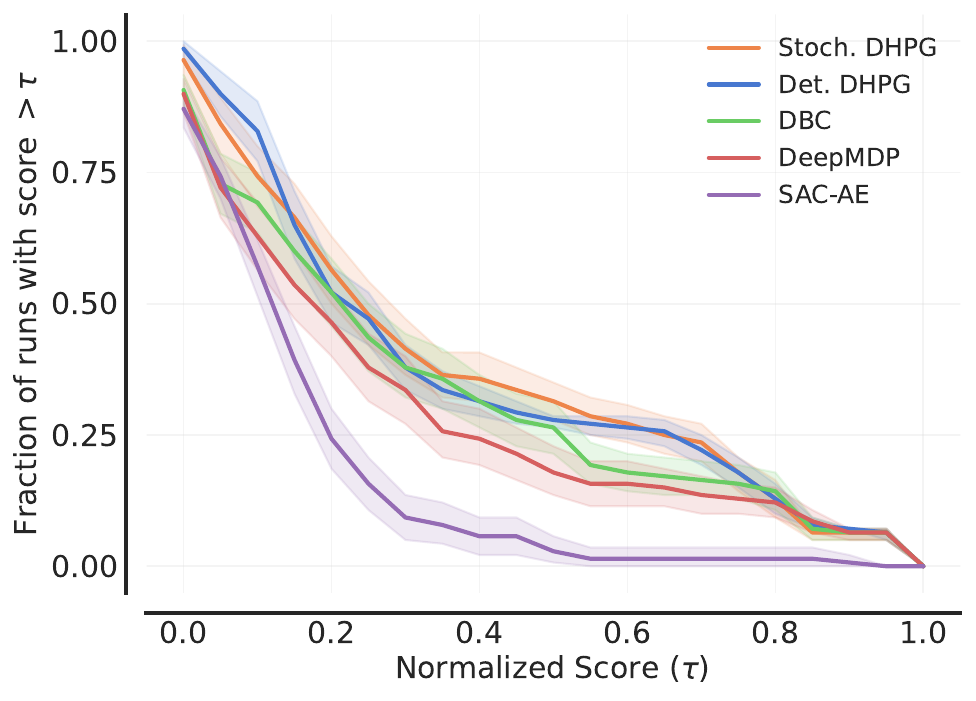}
        \caption{500k step benchmark.}
    \end{subfigure}
    \hfill
    \begin{subfigure}[b]{0.45\textwidth}
        \includegraphics[width=\textwidth]{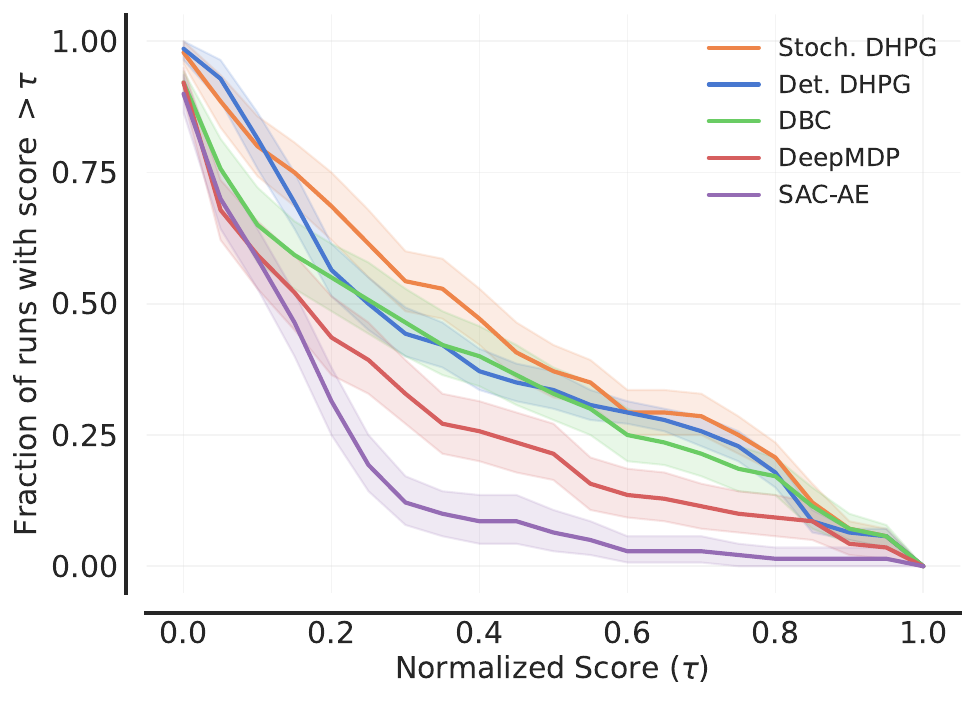}
        \caption{1m step benchmark.}
    \end{subfigure}
    \caption{Performance profiles for pixel observations based on 14 tasks over 10 seeds, at 500k steps \textbf{(a)}, and at 1m steps \textbf{(b)}. All methods are \textbf{without} image augmentation. Shaded regions represent $95\%$ confidence intervals.}    
    \label{fig:pixel_no_aug_results_performance_profiles}
\end{figure}
\clearpage

\begin{figure}[h!]
    \centering    
    \begin{subfigure}[b]{0.95\textwidth}
        \includegraphics[width=\textwidth]{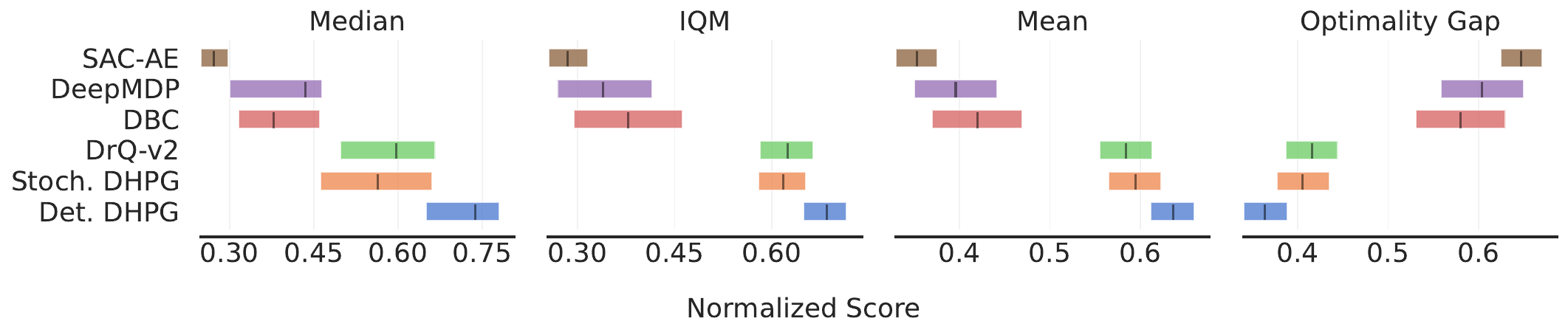}
        \caption{500k step benchmark.}
    \end{subfigure}
    \vspace{2em}
    
    \begin{subfigure}[b]{0.95\textwidth}
        \includegraphics[width=\textwidth]{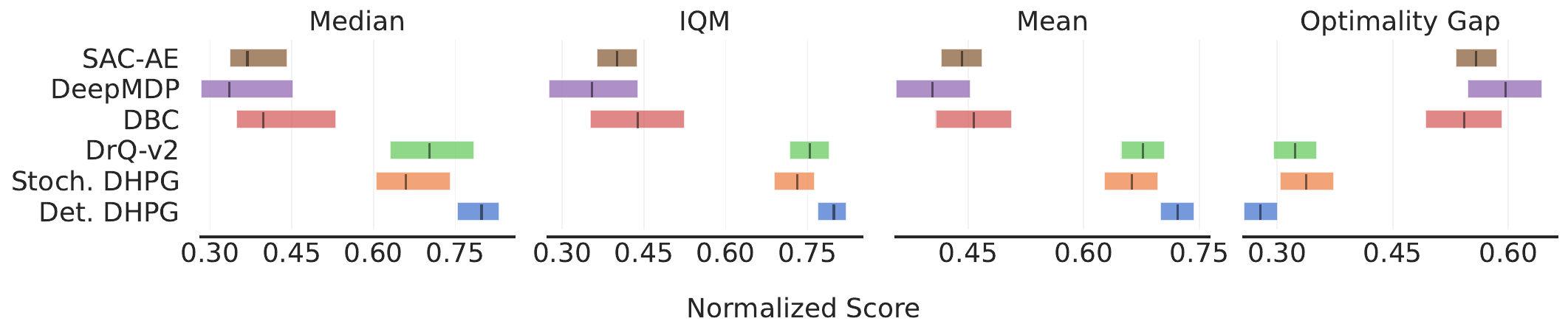}
        \caption{1m step benchmark.}
    \end{subfigure}
    \caption{Aggregate metrics for pixel observations with $95\%$ confidence intervals based on 14 tasks over 10 seeds, at 500k steps \textbf{(a)}, and at 1m steps \textbf{(b)}. All methods are \textbf{with} image augmentation.}    
    \label{fig:pixel_aug_results_aggregate_metrics}
\end{figure}

\vfill

\begin{figure}[h!]
    \centering    
    \begin{subfigure}[b]{0.95\textwidth}
        \includegraphics[width=\textwidth]{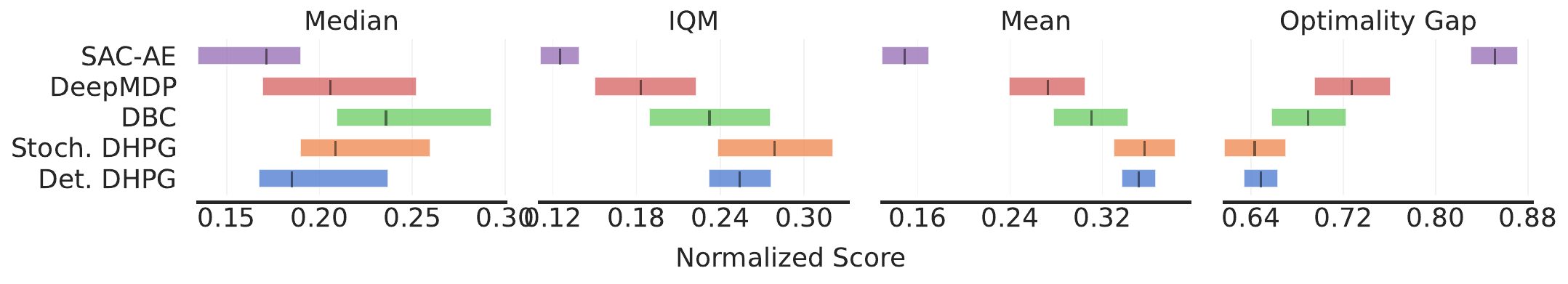}
        \caption{500k step benchmark.}
    \end{subfigure}
    \vspace{2em}
    
    \begin{subfigure}[b]{0.95\textwidth}
        \includegraphics[width=\textwidth]{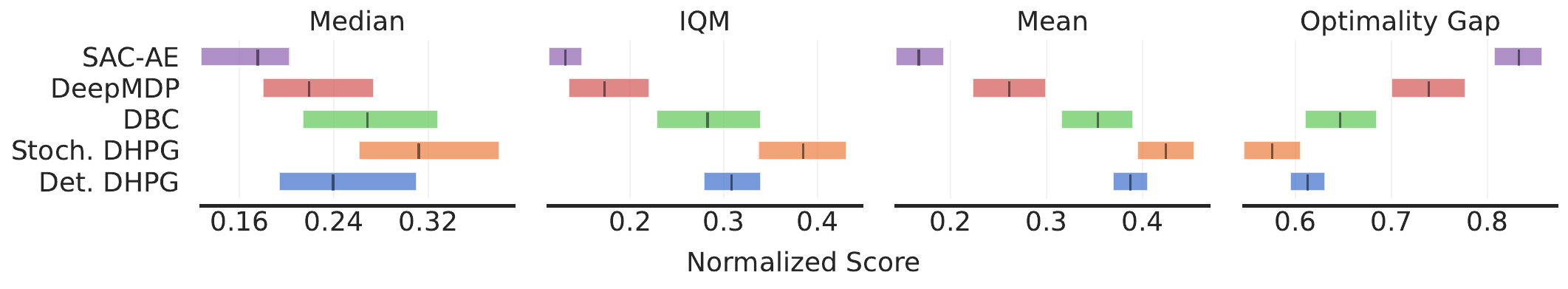}
        \caption{1m step benchmark.}
    \end{subfigure}
    \caption{Aggregate metrics for pixel observations with $95\%$ confidence intervals based on 14 tasks over 10 seeds, at 500k steps \textbf{(a)}, and at 1m steps \textbf{(b)}. All methods are \textbf{without} image augmentation.}    
    \label{fig:pixel_no_aug_results_aggregate_metrics}
\end{figure}

\clearpage

\subsection{Value Equivalence Property in Practice}
\label{sec:value_equiv_supp}
We can use the value equivalence between the critics of the actual and abstract MDPs as a measure for the quality of learned MDP homomorphisms, since the two critics are not directly trained to minimize this distance, instead they have equivalent values through the learned MDP homomorphism map. Figure \ref{fig:value_equivalence} shows the normalized mean absolute error of $|Q(s, a) \!-\! \overline{Q}(\overline{s}, \overline{a})|$ during training, indicating the property is holding in practice. Expectedly, for lower-dimensional tasks with easily learnable homomorphism maps (e.g., cartpole) the error is reduced earlier than more complicated tasks (e.g., quadruped). But importantly, in all cases the error decreases over time and is at a reasonable range towards the end of the training, meaning the continuous MDP homomorphisms is adhering to conditions of Definition \ref{def:cont_mdp_homo}.
\begin{figure}[h!]
    \centering    
    \begin{subfigure}[b]{0.7\textwidth}
        \includegraphics[width=\textwidth]{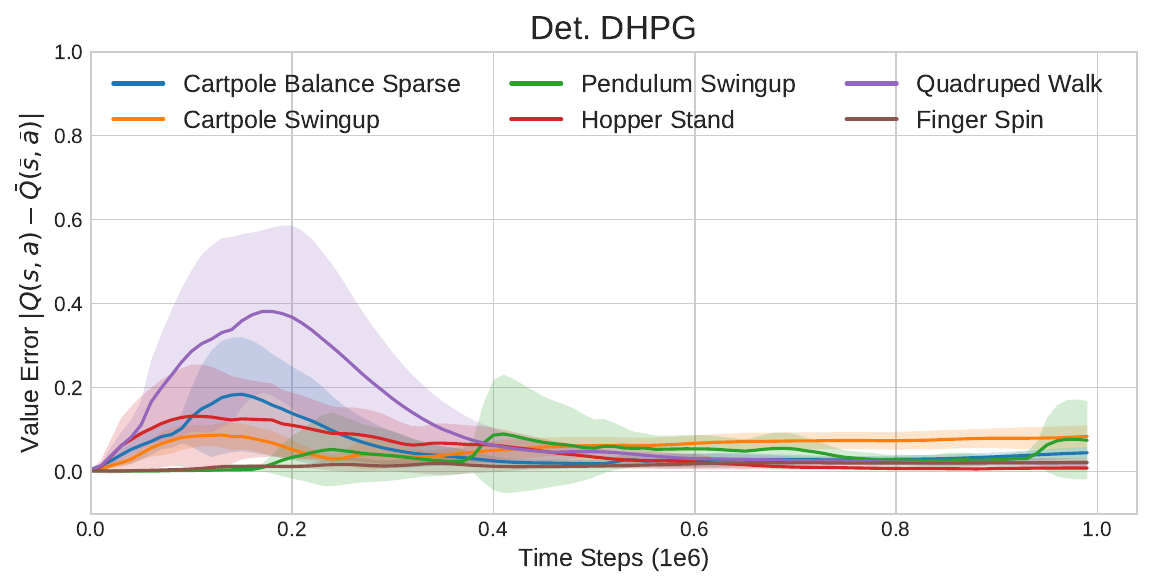}
        \caption{Value equivalence for deterministic DHPG.}
    \end{subfigure}
    \vspace{1em}
    
    \begin{subfigure}[b]{0.7\textwidth}
        \includegraphics[width=\textwidth]{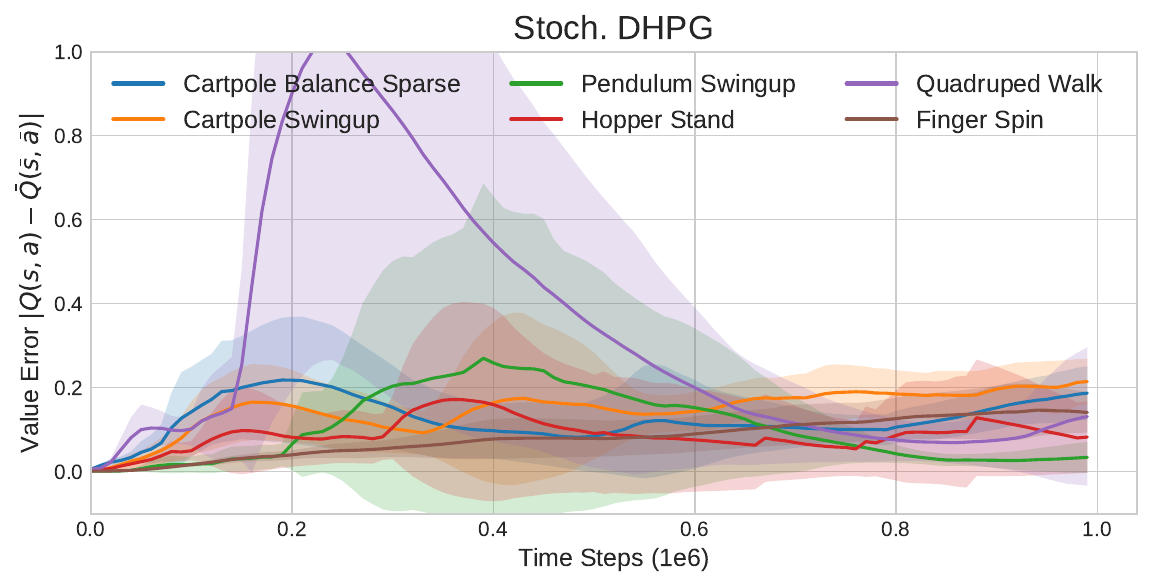}
        \caption{Value equivalence for stochastic DHPG.}
    \end{subfigure}
    \caption{Normalized mean absolute error $|Q(s, a) - \overline{Q}(\overline{s}, \overline{a})|$ as a measure for the value equivalence property during training of different tasks from pixel observations. The error is measured on samples from the replay buffer and is normalized by the range of the value function. The error is averaged over 10 seeds and shaded regions represent $95\%$ confidence intervals.}    
    \label{fig:value_equivalence}
\end{figure}
\clearpage

\clearpage
\subsection{Transfer Learning Experiments}
\label{sec:transfer_supp}
As discussed in Section \ref{sec:experiments}, the purpose of transfer experiments is to ensure that using MDP homomorphisms does not compromise transfer abilities. We select the deterministic DHPG for these experiments. Figure \ref{fig:pixels_transfer_supp} shows learning curves for a series of transfer scenarios in which the critic, actor, and representations are transferred to a new task within the same domain. DHPG matches the same transfer abilities of other methods. 
\begin{figure}[h!]
     \centering
     \begin{subfigure}[b]{0.325\textwidth}
         \centering
         \includegraphics[width=\textwidth]{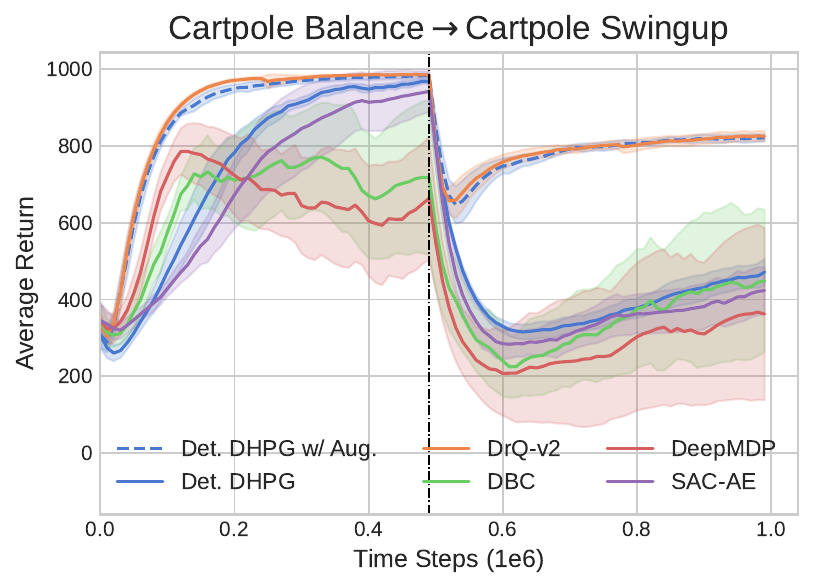}
         \caption{Cartpole.}
     \end{subfigure}
     \hfill
     \begin{subfigure}[b]{0.325\textwidth}
         \centering
         \includegraphics[width=\textwidth]{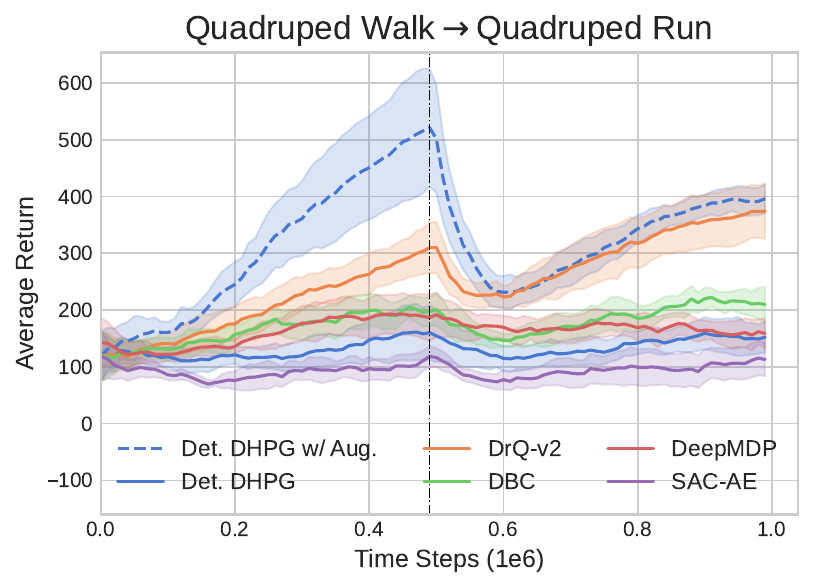}
         \caption{Quadruped.}
     \end{subfigure}
     \hfill
     \begin{subfigure}[b]{0.325\textwidth}
         \centering
         \includegraphics[width=\textwidth]{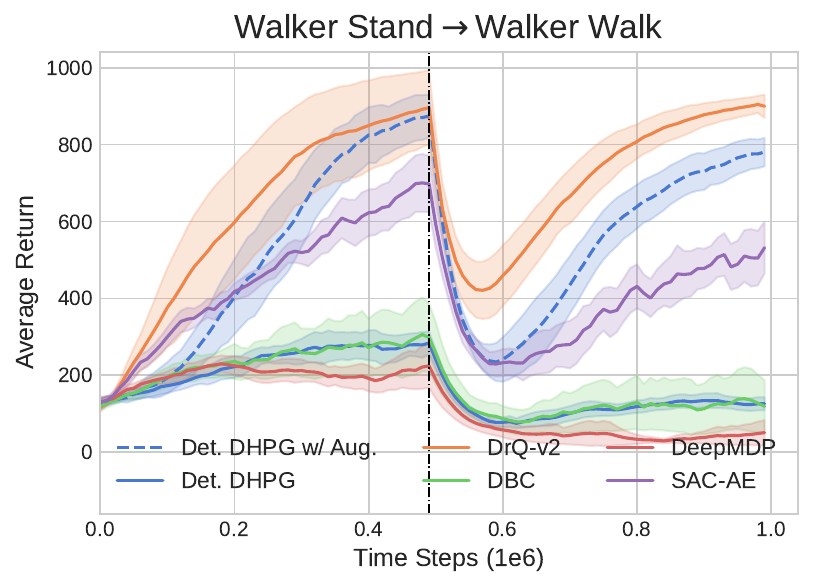}
         \caption{Walker.}
     \end{subfigure}        
     \caption{Learning curves for transfer experiments with pixel observations. At 500k time step mark, all components are transferred to a new task on the same domain. Mean performance is obtained over 10 seeds and shaded regions represent $95\%$ confidence intervals. Plots are smoothed uniformly for visual clarity.}
    \label{fig:pixels_transfer_supp}
\end{figure}

\subsection{Ablation Study on the Combination of HPG with the Standard Policy Gradient}
\label{sec:ablation_dhpg_variants}
We carry out an ablation study on the combination of HPG with the standard policy gradient for actor updates. Since the deterministic DHPG algorithm is generally simpler as the lifted policy can be analytically obtained, we evaluate the performance of four variants of the deterministic DHPG (all variants use image augmentation):
\begin{enumerate}[noitemsep,nosep]
    \item \textbf{DHPG:} Gradients of HPG and DPG are added together and a single actor update is done based on the sum of gradients. This is the deterministic DHPG algorithm that is used throughout the paper.
    \item \textbf{DHPG with independent DPG update:} Gradients of HPG and DPG are independently used to update the actor. 
    \item \textbf{DHPG without DPG update:} Only HPG is used to update the actor. 
    \item \textbf{DHPG with single critic:} A single critic network is trained for learning values of both the actual and abstract MDP. Consequently, HPG and DPG are used to update the actor. 
\end{enumerate}
Figure \ref{fig:ablation_dhpg_variants} shows learning curves obtained on 16 DeepMind Control Suite tasks with pixel observations, and Figure \ref{fig:ablation_dhpg_variants_rliable} shows RLiable \citep{agarwal2021deep} evaluation metrics. In general, summing the gradients of HPG and DPG (variant 1) results in lower variance of gradient estimates compared to independent HPG and DPG updates (variant 2).
Interestingly, the variant of DHPG without DPG (variant 3) performs reasonably well or even outperforms other variants in simple tasks where learning MDP homomorphisms is easy (e.g., cartpole and pendulum), indicating the effectiveness of our method in using \textbf{only} the abstract MDP to update the policy of the actual MDP. However, in the case of more complicated tasks (e.g., walker), DPG is required to additionally use the actual MDP for policy optimization. Finally, using a single critic for both the actual and abstract MDPs (variant 4) can improve sample efficiency in symmetrical MDPs, but may result in performance drops in non-symmetrical MDPs due to the large error bound between the two MDPs, $\| Q^{\pi^\uparrow}\!(s,a) \!-\! Q^{\overline{\pi}}\!(\overline{s}, \overline{a}) \|$ \citep{taylor2008bounding}.

\begin{figure}[h!]
    \centering
    \begin{subfigure}[b]{0.95\textwidth}
         \centering
         \includegraphics[width=\textwidth]{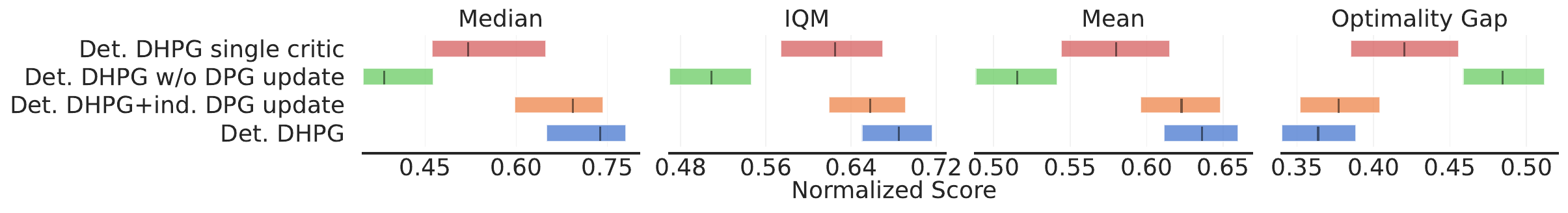}
         \caption{Aggregate metrics at 500k.}
    \end{subfigure}
    \vspace{2em}
    
    \begin{subfigure}[b]{0.325\textwidth}
         \centering
         \includegraphics[width=\textwidth]{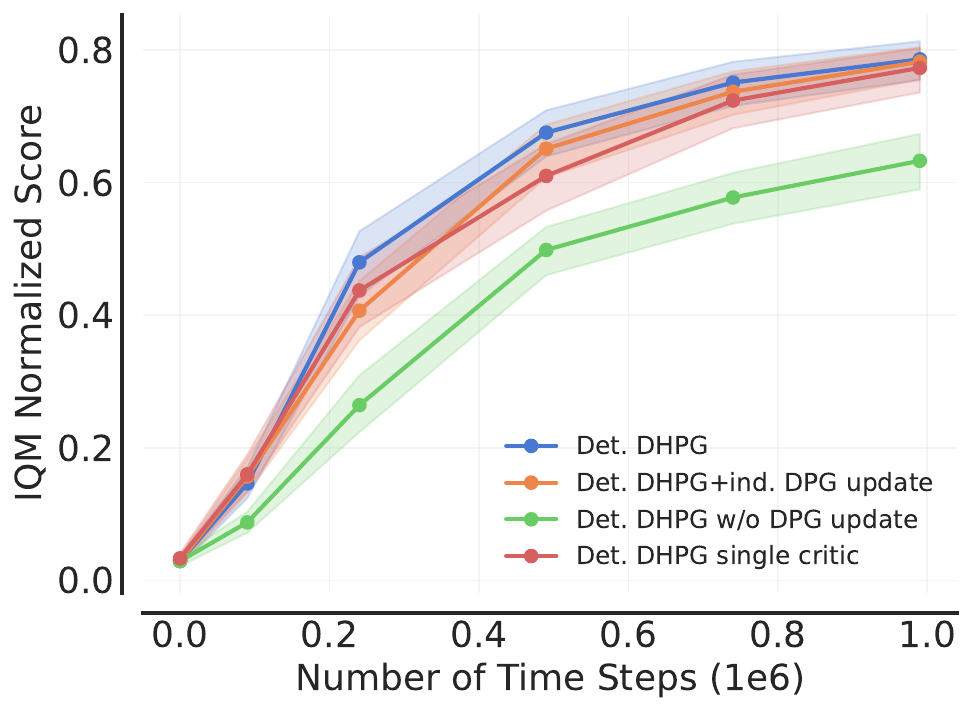}
         \caption{Sample efficiency.}
    \end{subfigure}
    \hfill
    \begin{subfigure}[b]{0.325\textwidth}
         \centering
         \includegraphics[width=\textwidth]{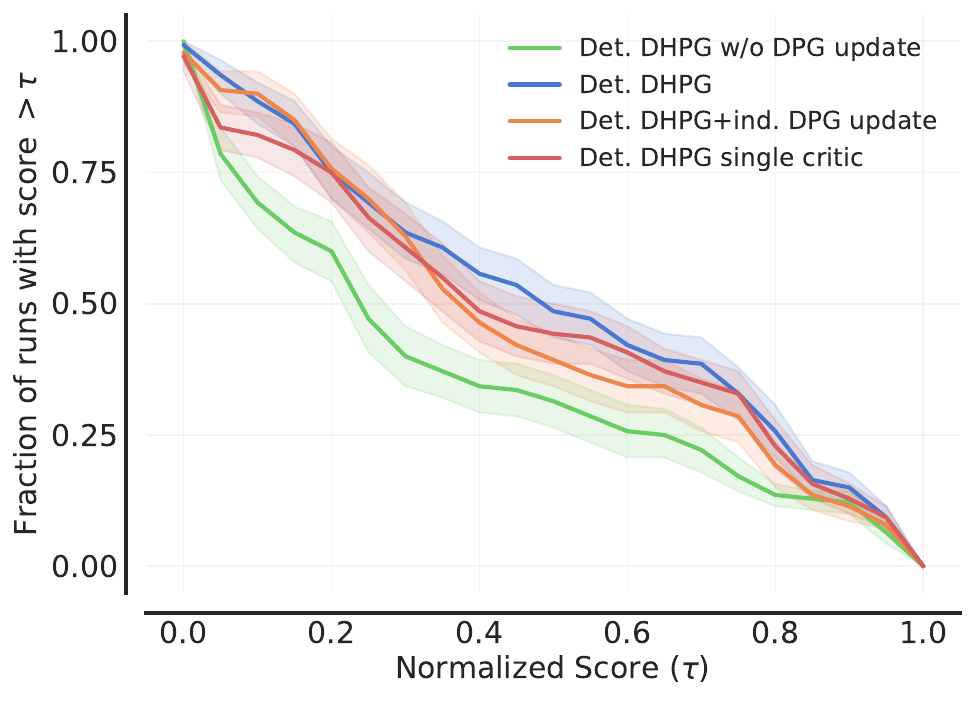}
         \caption{Performance profiles at 250k.}
    \end{subfigure}
    \hfill
    \begin{subfigure}[b]{0.325\textwidth}
         \centering
         \includegraphics[width=\textwidth]{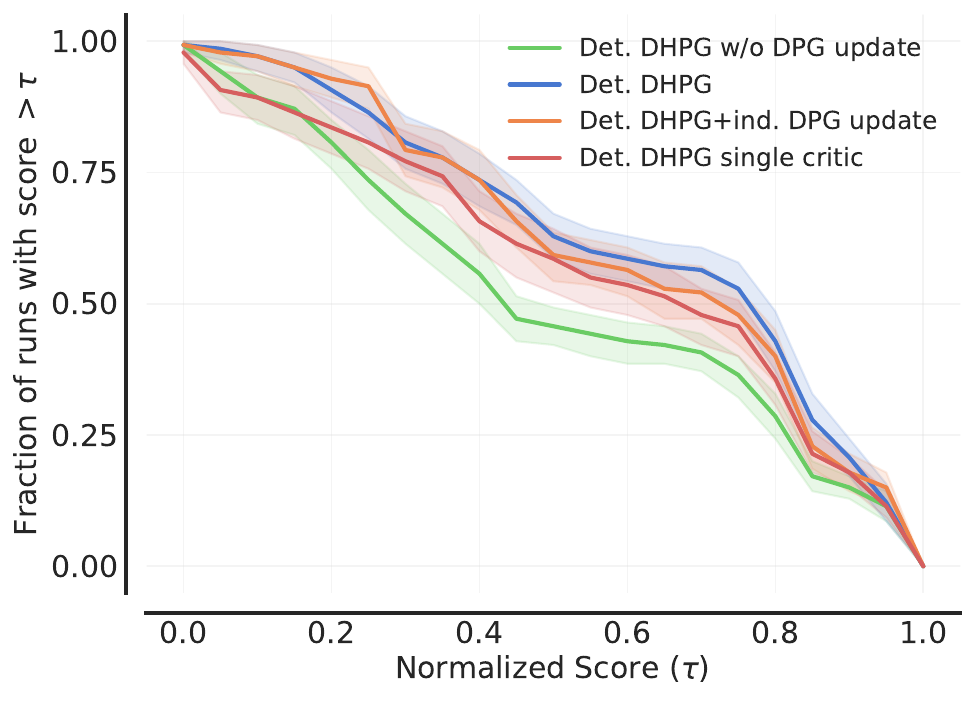}
         \caption{Performance profiles at 500k.}
    \end{subfigure}
    \caption{Ablation study on the combination of HPG and DPG. RLiable evaluation metrics for pixel observations averaged on 14 tasks over 10 seeds. Aggregate metrics at 500k steps \textbf{(a)}, IQM scores as a function of number of steps for comparing sample efficiency \textbf{(b)}, performance profiles at 250k steps \textbf{(c)}, performance profiles at 500k steps \textbf{(d)}. Shaded regions represent $95\%$ confidence intervals.}
    \label{fig:ablation_dhpg_variants_rliable}
\end{figure}

\begin{figure}[ht!]
     \centering
     \begin{subfigure}[b]{0.24\textwidth}
         \centering
         \includegraphics[width=\textwidth]{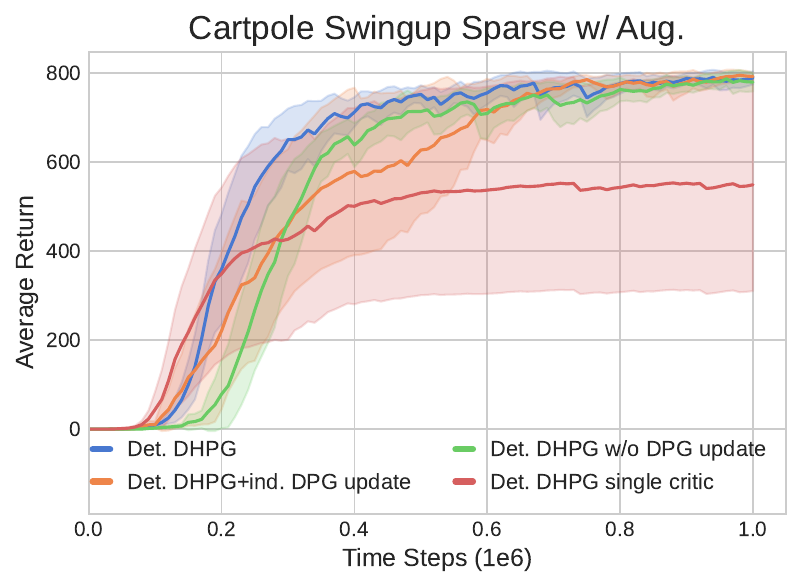}
     \end{subfigure}
     \hfill
     \begin{subfigure}[b]{0.24\textwidth}
         \centering
         \includegraphics[width=\textwidth]{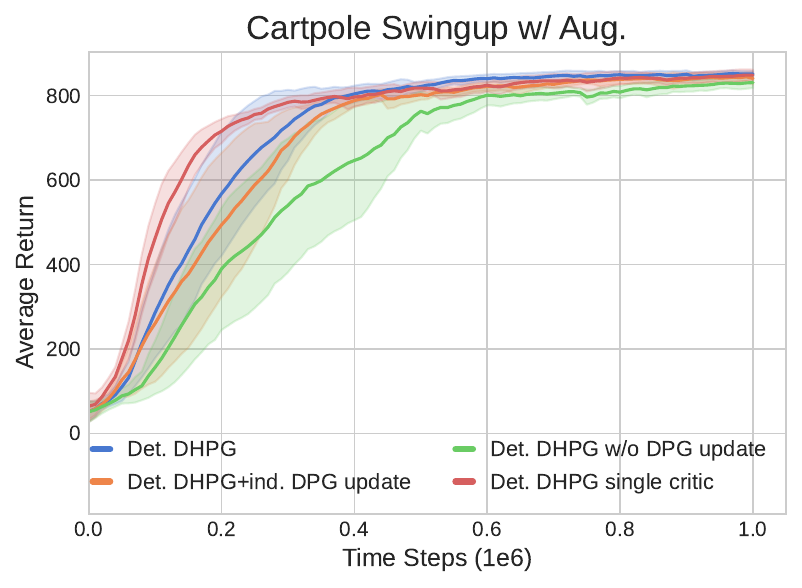}
     \end{subfigure}
     \hfill
     \begin{subfigure}[b]{0.24\textwidth}
         \centering
         \includegraphics[width=\textwidth]{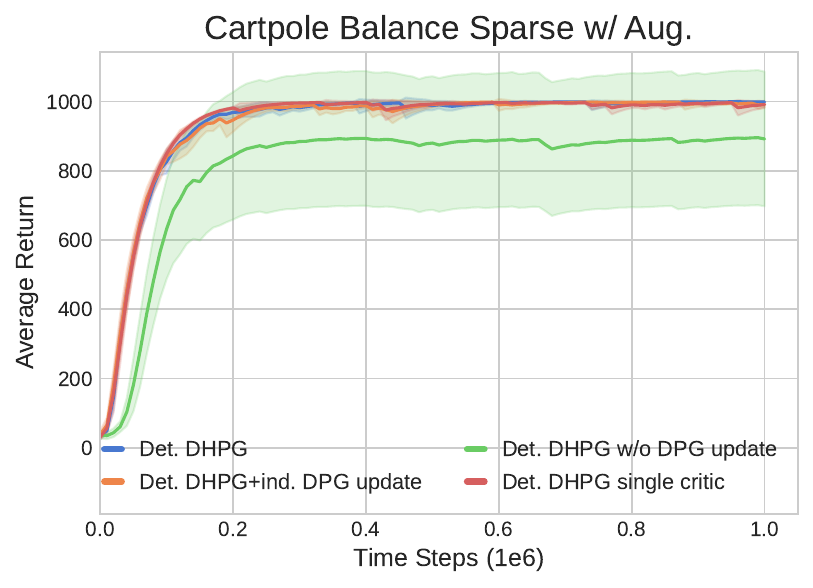}
     \end{subfigure}
     \hfill     
     \begin{subfigure}[b]{0.24\textwidth}
         \centering
         \includegraphics[width=\textwidth]{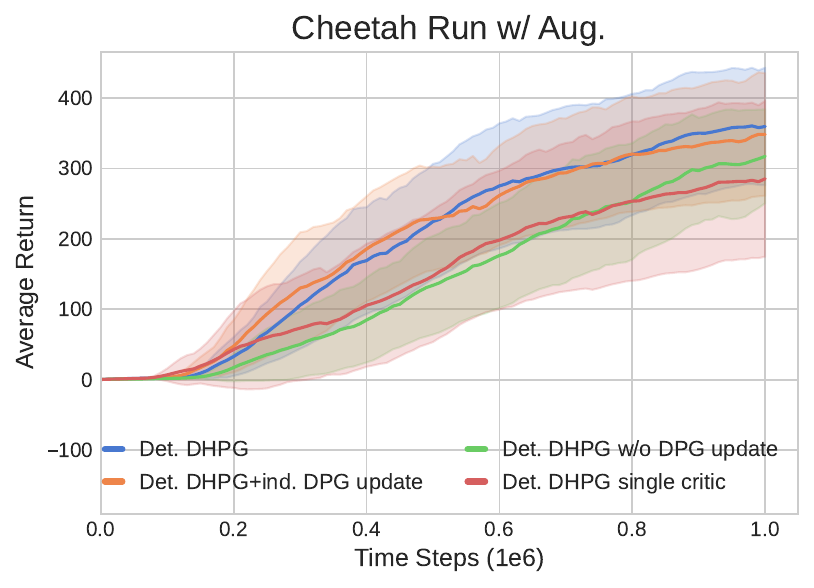}
     \end{subfigure}
     
     \hfill     
     \begin{subfigure}[b]{0.24\textwidth}
         \centering
         \includegraphics[width=\textwidth]{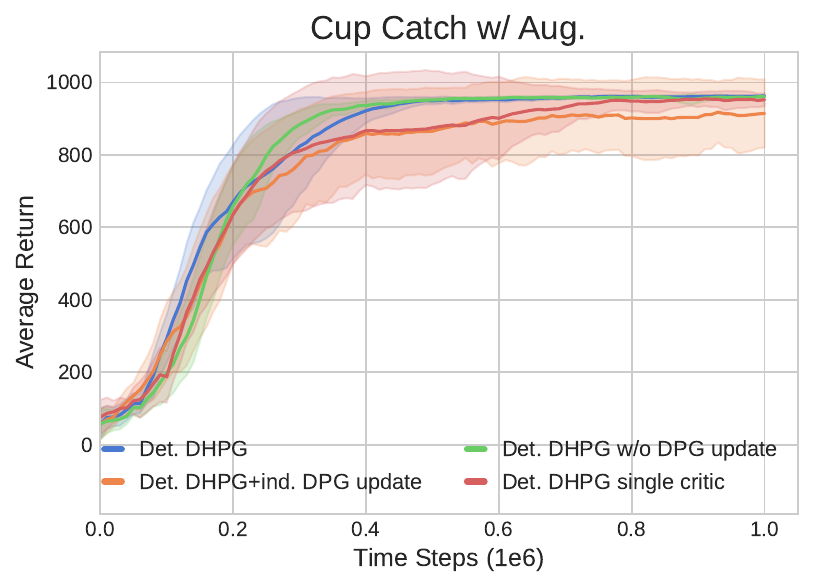}
     \end{subfigure}
     \hfill
     \begin{subfigure}[b]{0.24\textwidth}
         \centering
         \includegraphics[width=\textwidth]{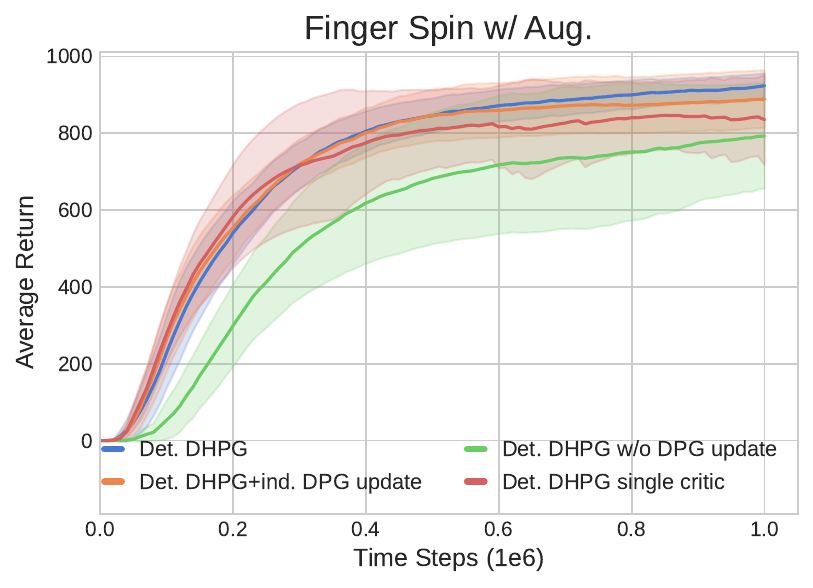}
     \end{subfigure}
     \hfill     
     \begin{subfigure}[b]{0.24\textwidth}
         \centering
         \includegraphics[width=\textwidth]{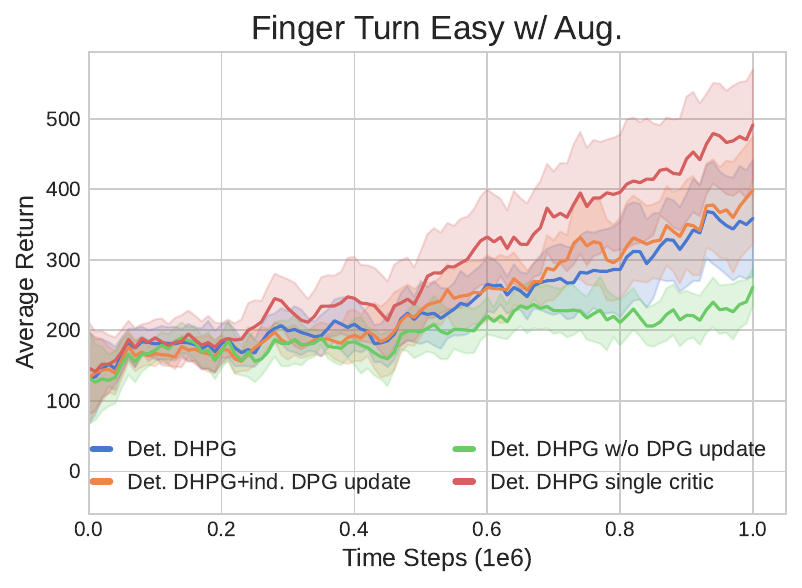}
     \end{subfigure}
     \hfill
     \begin{subfigure}[b]{0.24\textwidth}
         \centering
         \includegraphics[width=\textwidth]{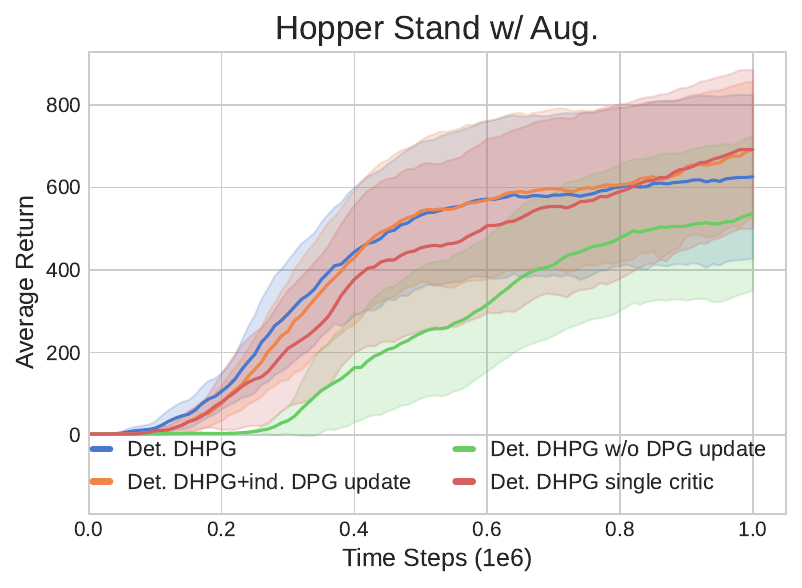}
     \end{subfigure}
     
     \hfill     
     \begin{subfigure}[b]{0.24\textwidth}
         \centering
         \includegraphics[width=\textwidth]{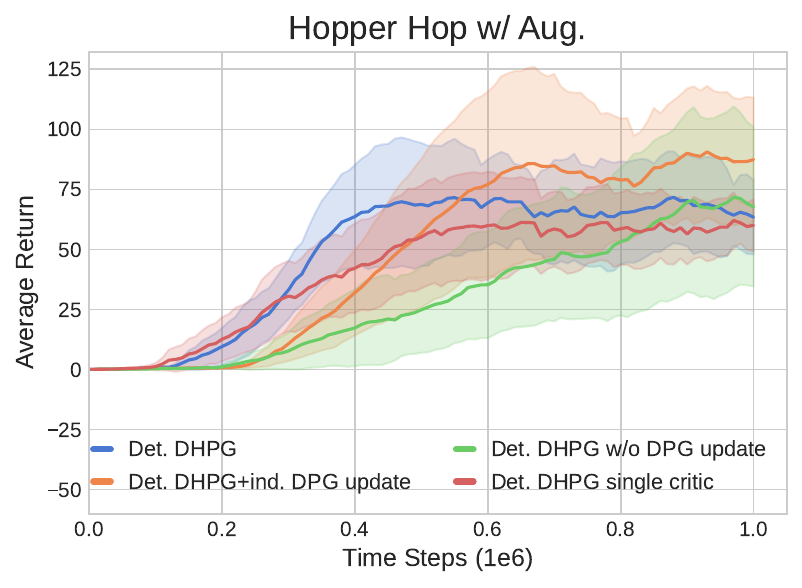}
     \end{subfigure}
     \hfill     
     \begin{subfigure}[b]{0.24\textwidth}
         \centering
         \includegraphics[width=\textwidth]{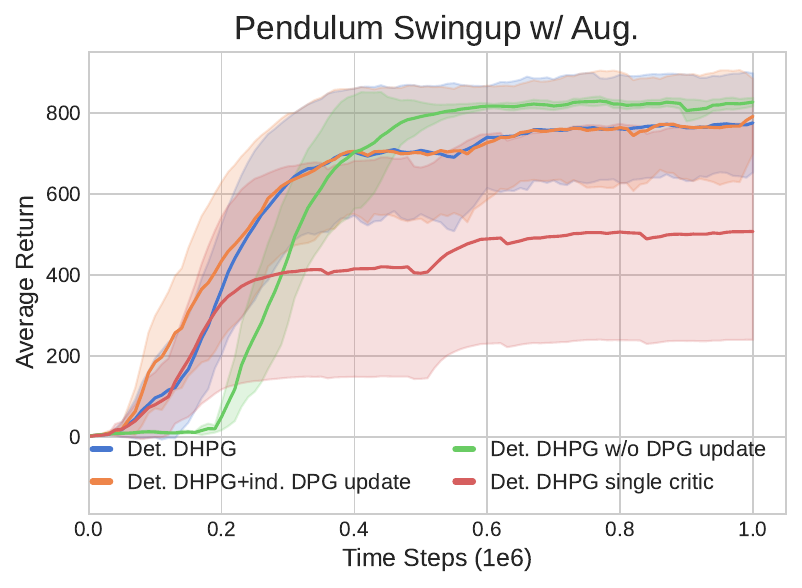}
     \end{subfigure}
     \hfill
     \begin{subfigure}[b]{0.24\textwidth}
         \centering
         \includegraphics[width=\textwidth]{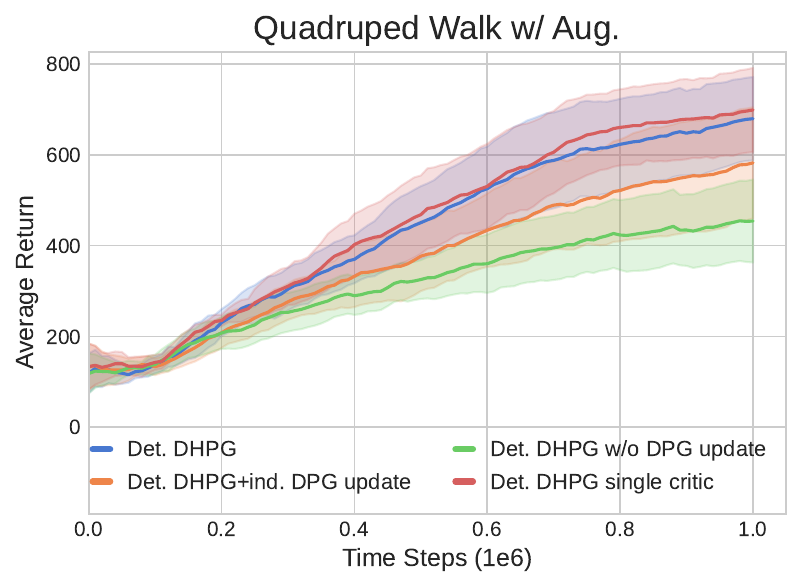}
     \end{subfigure}
     \hfill
     \begin{subfigure}[b]{0.24\textwidth}
         \centering
         \includegraphics[width=\textwidth]{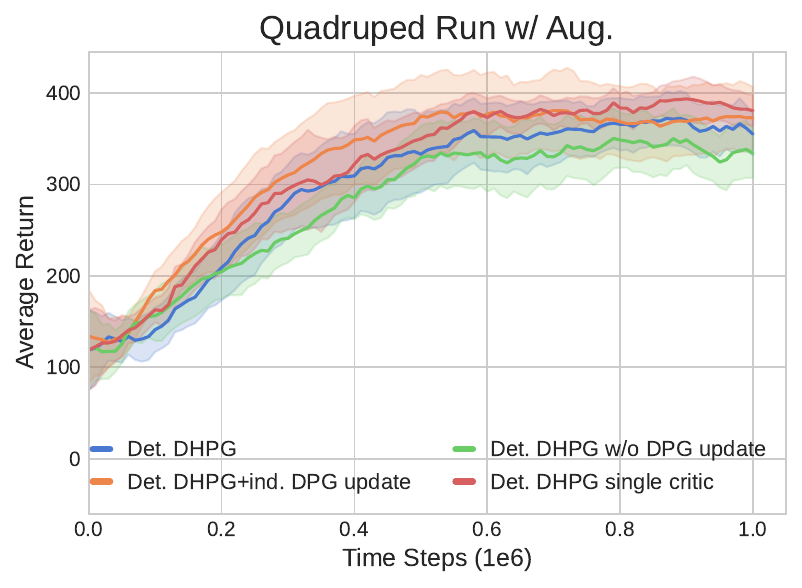}
     \end{subfigure}
     \hfill
     
     \begin{subfigure}[b]{0.24\textwidth}
         \centering
         \includegraphics[width=\textwidth]{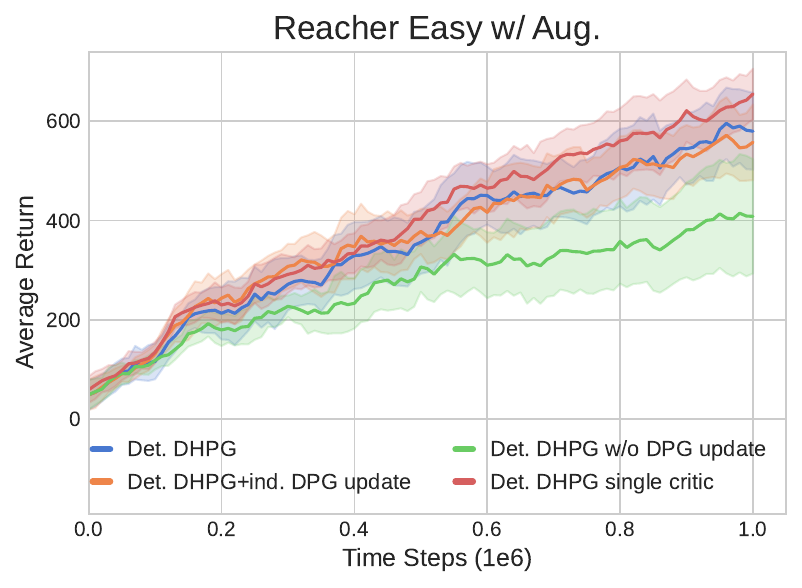}
     \end{subfigure}
     \hfill
     \begin{subfigure}[b]{0.24\textwidth}
         \centering
         \includegraphics[width=\textwidth]{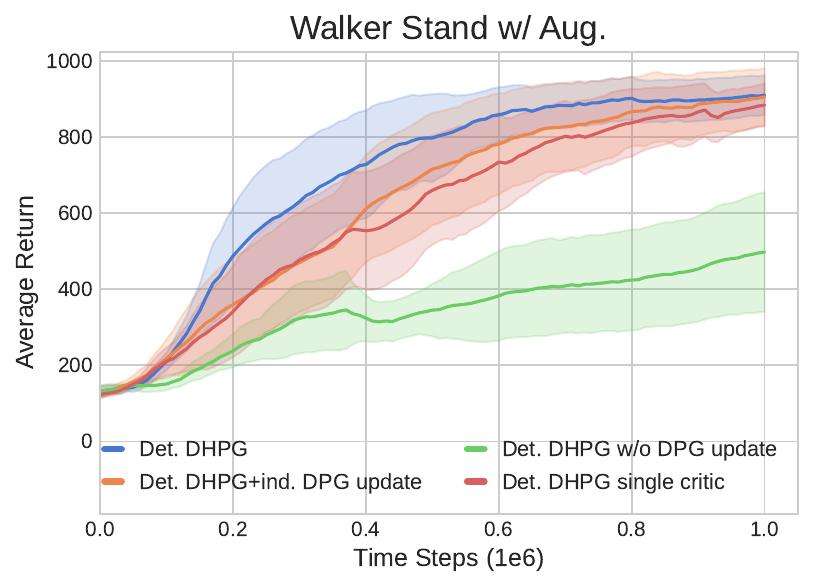}
     \end{subfigure}
     \begin{subfigure}[b]{0.24\textwidth}
         \centering
         \includegraphics[width=\textwidth]{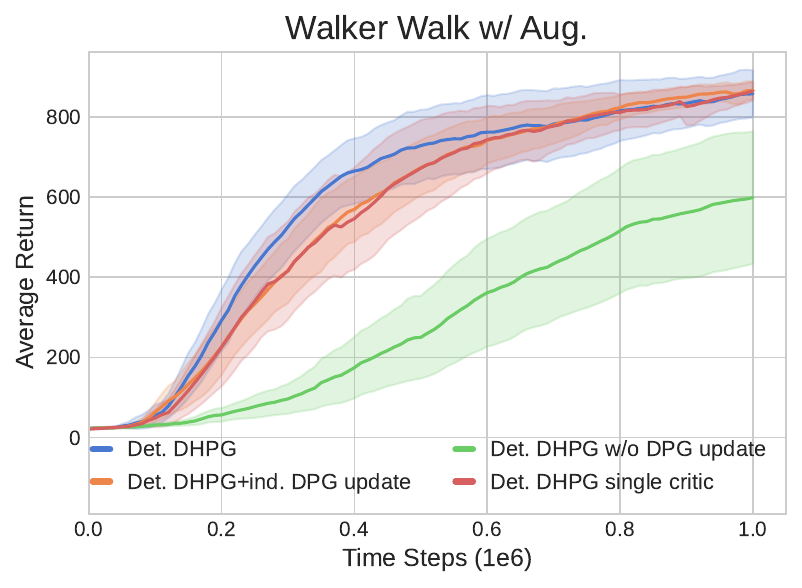}
     \end{subfigure}
     \hfill     
     \begin{subfigure}[b]{0.24\textwidth}
         \centering
         \includegraphics[width=\textwidth]{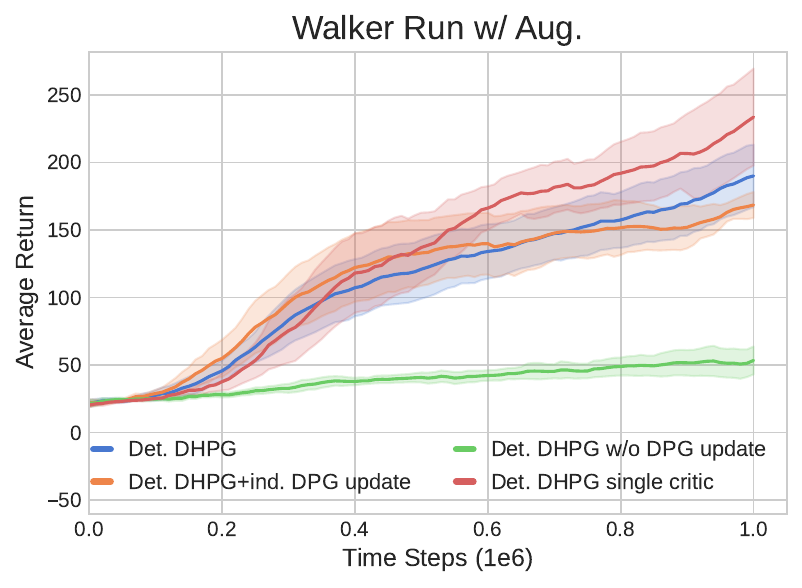}
     \end{subfigure}
     \hfill
    \caption{Ablation study on the combination of HPG and DPG. Learning curves for 16 DM control tasks with pixel observations. Mean performance is obtained over 10 seeds and shaded regions represent $95\%$ confidence intervals. Plots are smoothed uniformly for visual clarity.}
    \label{fig:ablation_dhpg_variants}
\end{figure}

\clearpage

\subsection{Ablation Study on n-step Return}
We carry out an ablation study on the choice of $n$-step return for DHPG. Similarly to the ablation study in Appendix \ref{sec:ablation_dhpg_variants}, we select the deterministic DHPG since  it is generally simpler and its lifted policy can be analytically obtained. Figure \ref{fig:ablation_nstep_rliable} shows RLiable \citep{agarwal2021deep} evaluation metrics for DHPG with $1$-step and $3$-step returns for pixel observations. We show the impact of $n$-step return on DHPG with and without image augmentation. Overall, $n$-step return appears to improve the early stages of training. In the case of DHPG without image augmentation, the final performance of $1$-step return is better than $3$-step return, perhaps indicating that using $n$-step return can render learning MDP homomorphisms more difficult.
\label{sec:ablation_n_step}
\begin{figure}[h!]
    \centering
    \begin{subfigure}[b]{0.95\textwidth}
         \centering
         \includegraphics[width=\textwidth]{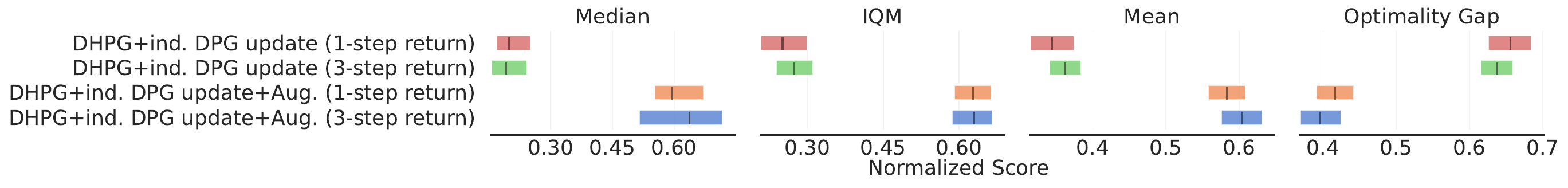}
         \caption{Aggregate metrics at 500k steps.}
    \end{subfigure}
    \vspace{2em}
    
    \begin{subfigure}[b]{0.325\textwidth}
         \centering
         \includegraphics[width=\textwidth]{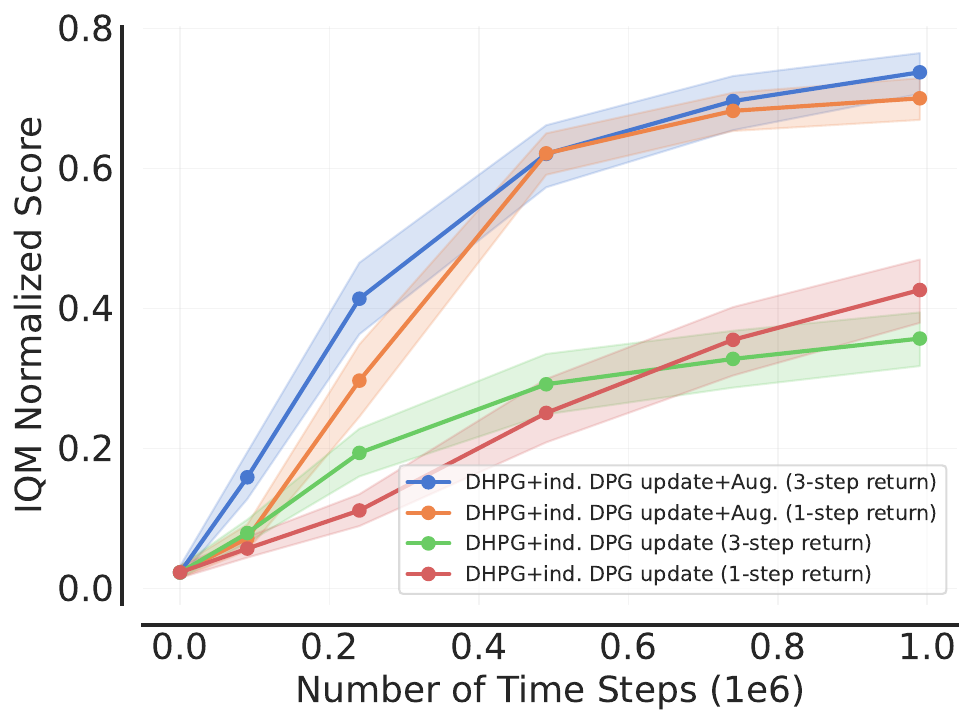}
         \caption{Sample efficiency.}
    \end{subfigure}
    \hfill
    \begin{subfigure}[b]{0.325\textwidth}
         \centering
         \includegraphics[width=\textwidth]{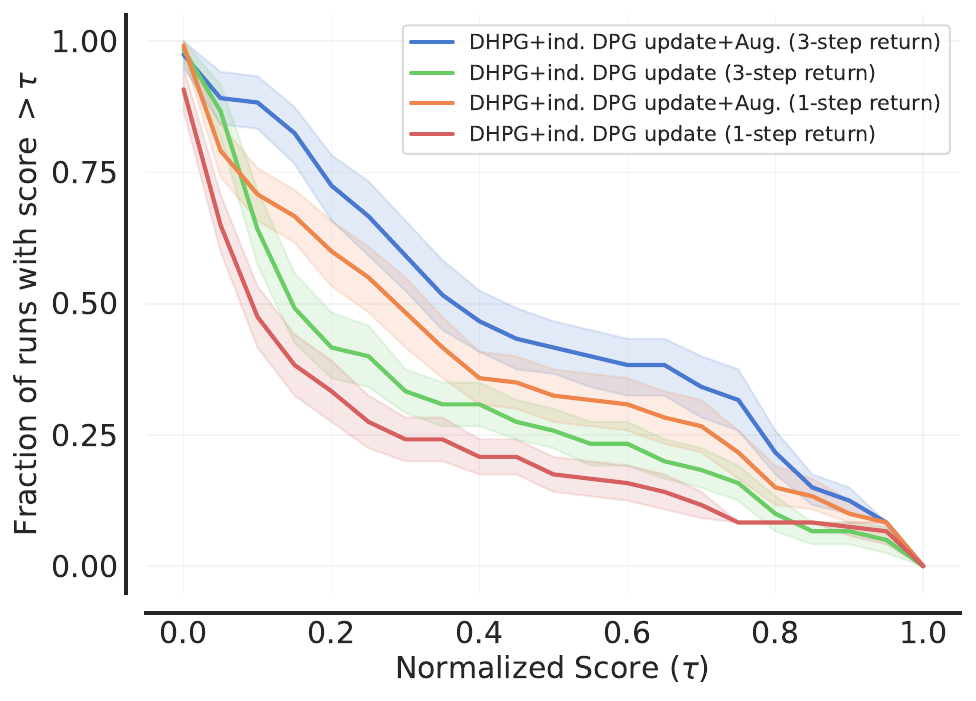}
         \caption{Performance profiles at 250k.}
    \end{subfigure}
    \hfill
    \begin{subfigure}[b]{0.325\textwidth}
         \centering
         \includegraphics[width=\textwidth]{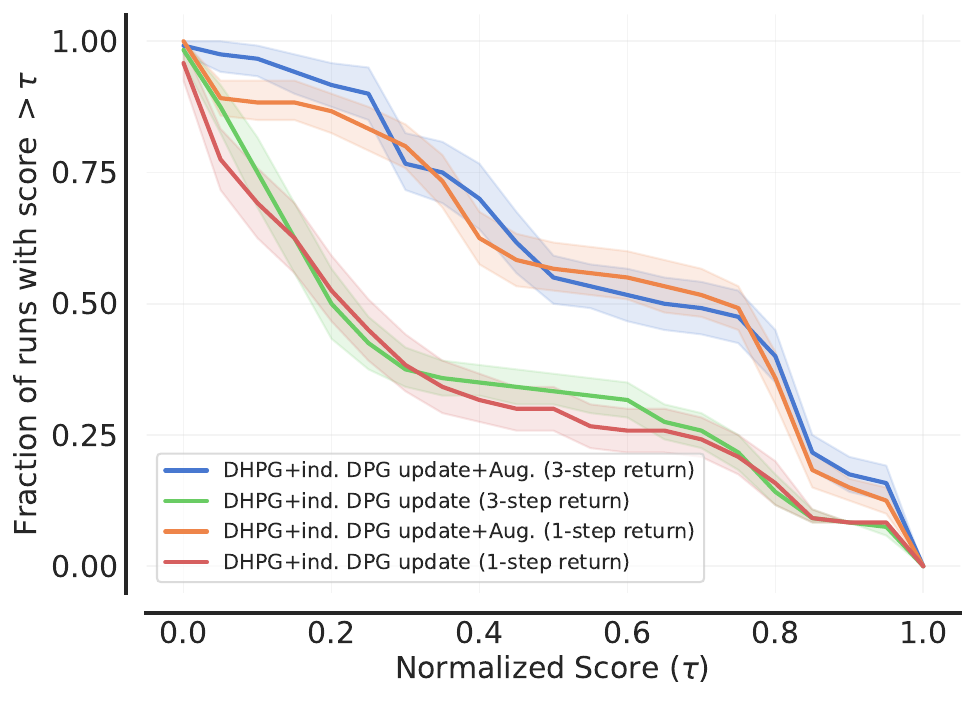}
         \caption{Performance profiles at 500k.}
    \end{subfigure}
    \caption{Ablation study on $n$-step return. RLiable evaluation metrics for pixel observations averaged on 12 tasks over 10 seeds. Aggregate metrics at 1m steps \textbf{(a)}, IQM scores as a function of number of steps for comparing sample efficiency \textbf{(b)}, performance profiles at 250k steps \textbf{(c)}, and performance profiles at 500k steps \textbf{(d)}. Shaded regions represent $95\%$ confidence intervals.}
    \label{fig:ablation_nstep_rliable}
\end{figure}
\clearpage

\subsection{Comparison of Computation Time}
\label{sec:computation_supp}

Figure \ref{fig:pixel_results_supp_computation} compares the computation cost of our method against the baselines. The horizontal axis represents wall clock time in hours. Since our method does not require image reconstruction, it is more computationally efficient than SAC-AE and DeepMDP. However, the bisimulation computation, the HPG update, and the policy lifting loss (in the case of stochastic DHPG) increase the computation costs of our method in comparison to DrQ-v2.

\begin{figure}[h!]
     \centering
     \begin{subfigure}[b]{0.24\textwidth}
         \centering
         \includegraphics[width=\textwidth]{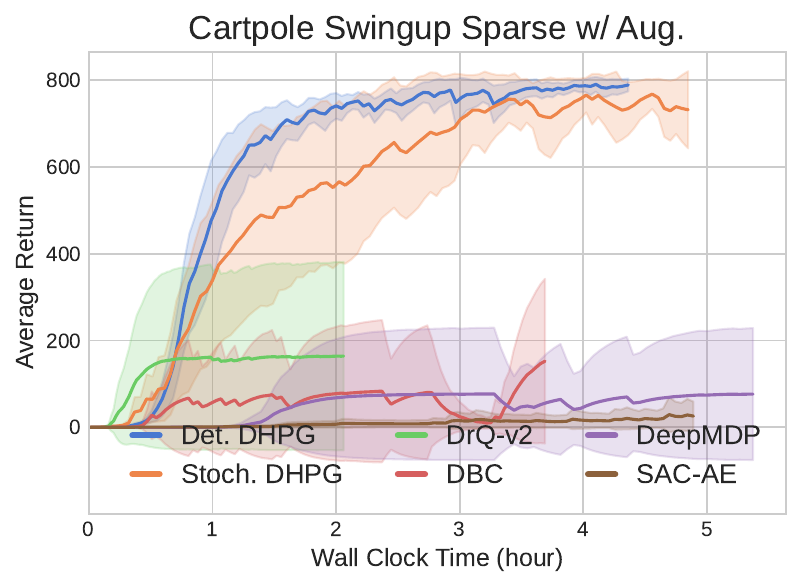}
     \end{subfigure}
     \hfill
     \begin{subfigure}[b]{0.24\textwidth}
         \centering
         \includegraphics[width=\textwidth]{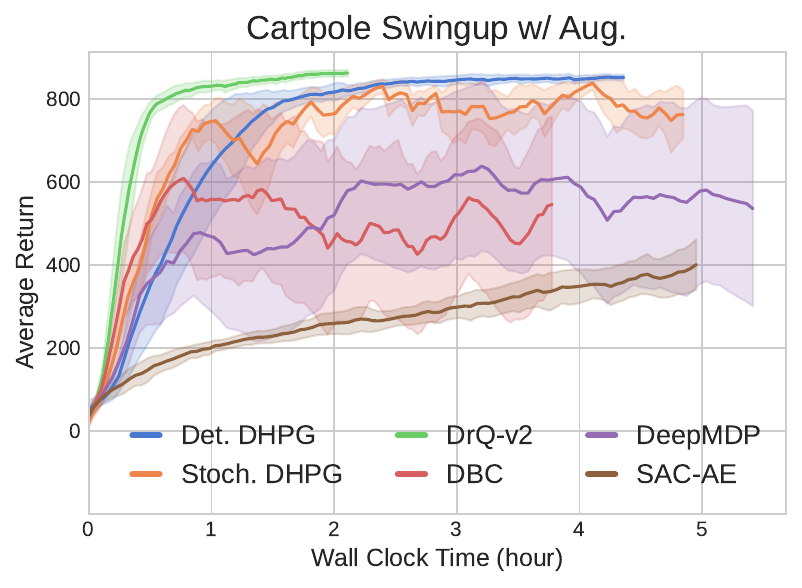}
     \end{subfigure}
     \hfill
     \begin{subfigure}[b]{0.24\textwidth}
         \centering
         \includegraphics[width=\textwidth]{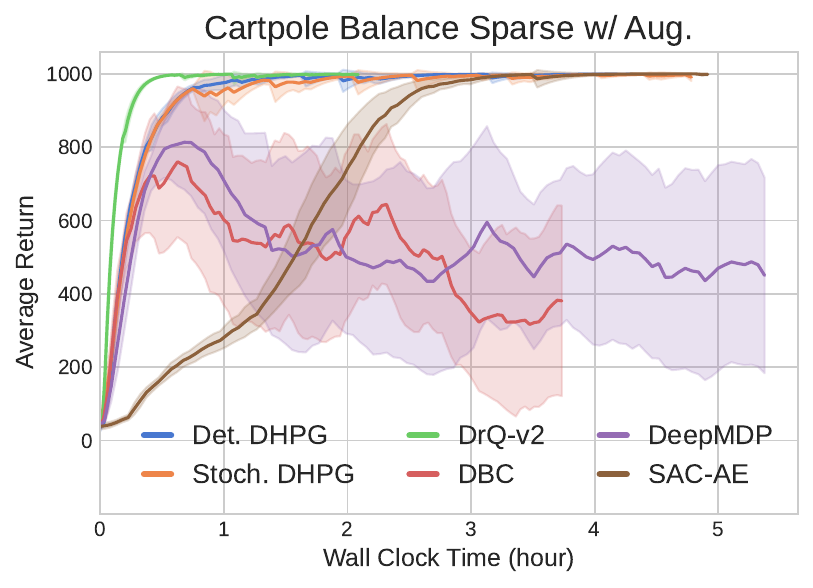}
     \end{subfigure}
     \hfill     
     \begin{subfigure}[b]{0.24\textwidth}
         \centering
         \includegraphics[width=\textwidth]{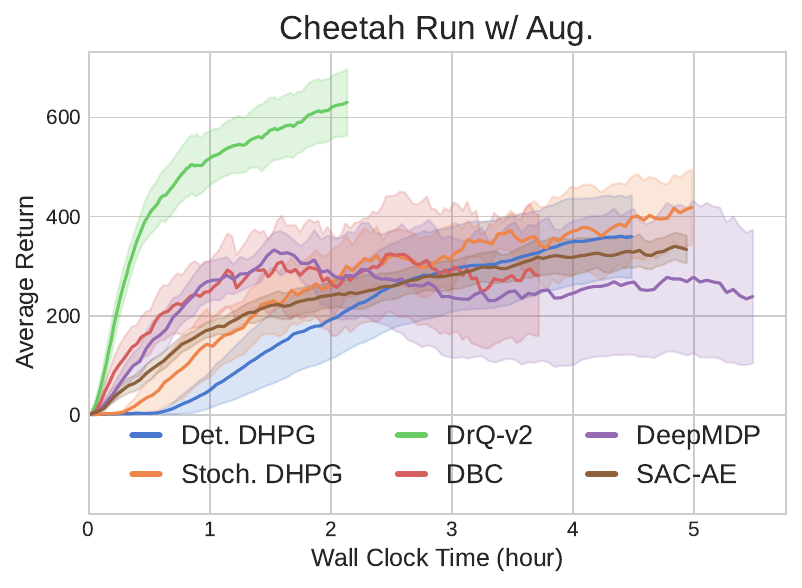}
     \end{subfigure}
     \hfill
     
     \begin{subfigure}[b]{0.24\textwidth}
         \centering
         \includegraphics[width=\textwidth]{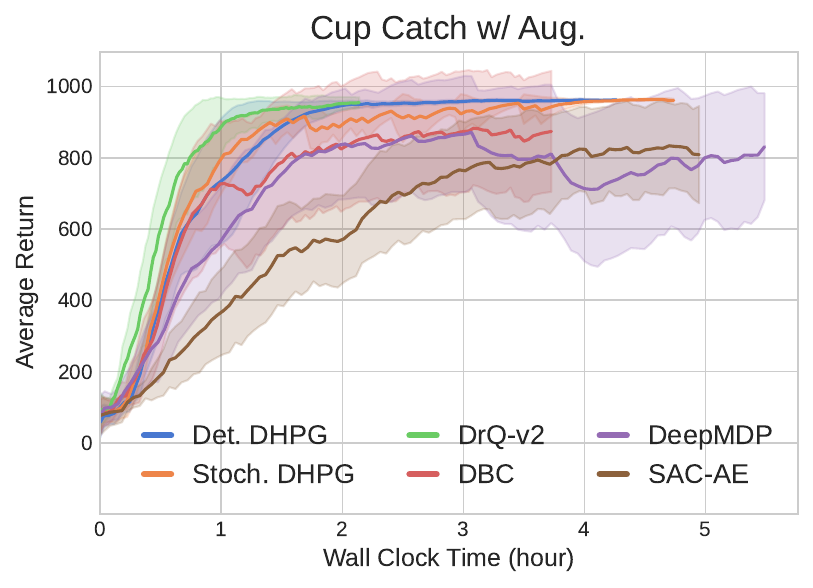}
     \end{subfigure}
     \hfill
     \begin{subfigure}[b]{0.24\textwidth}
         \centering
         \includegraphics[width=\textwidth]{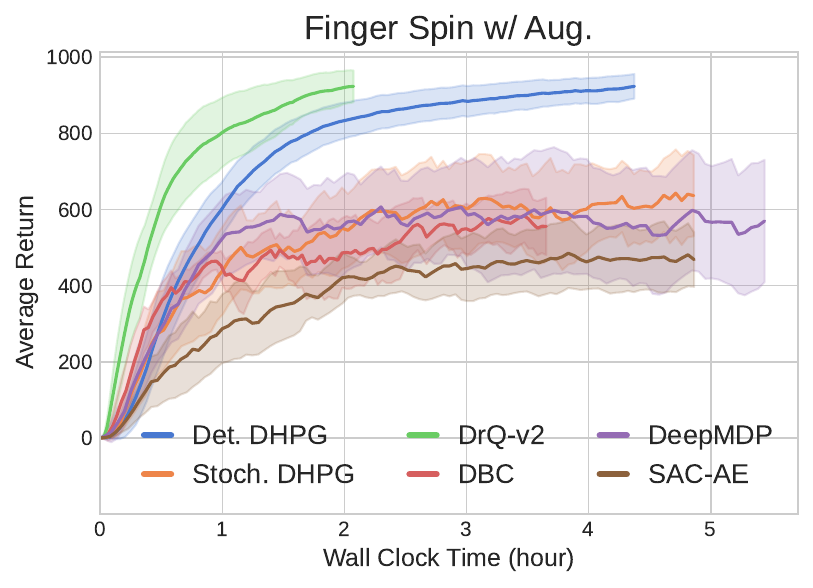}
     \end{subfigure}
     \hfill
     \begin{subfigure}[b]{0.24\textwidth}
         \centering
         \includegraphics[width=\textwidth]{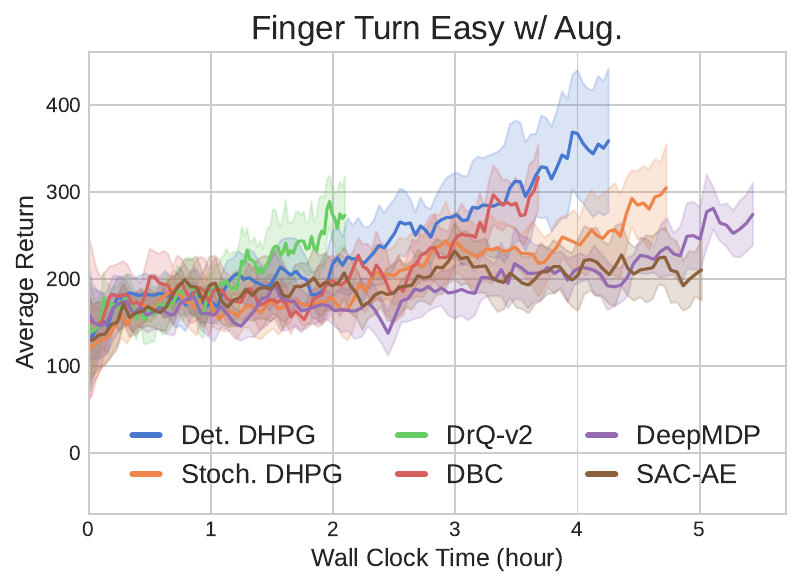}
     \end{subfigure}     
     \hfill
     \begin{subfigure}[b]{0.24\textwidth}
         \centering
         \includegraphics[width=\textwidth]{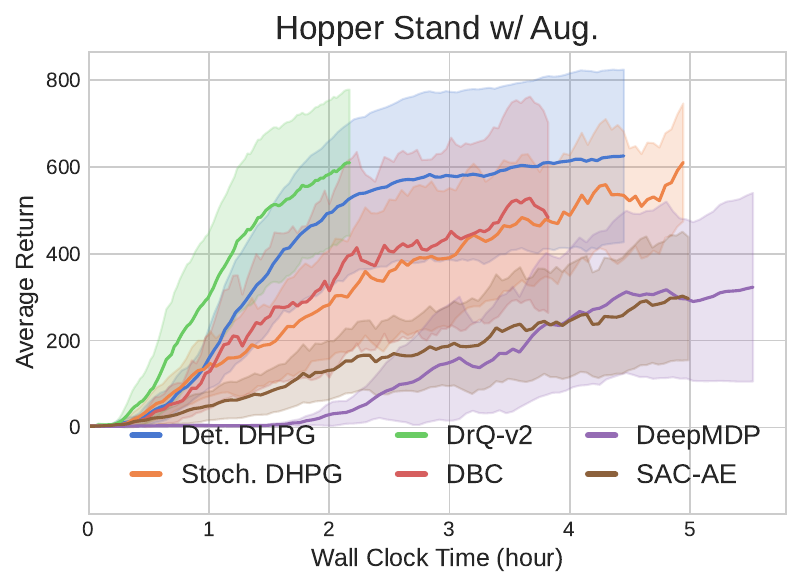}
     \end{subfigure}
     \hfill     
     
     \begin{subfigure}[b]{0.24\textwidth}
         \centering
         \includegraphics[width=\textwidth]{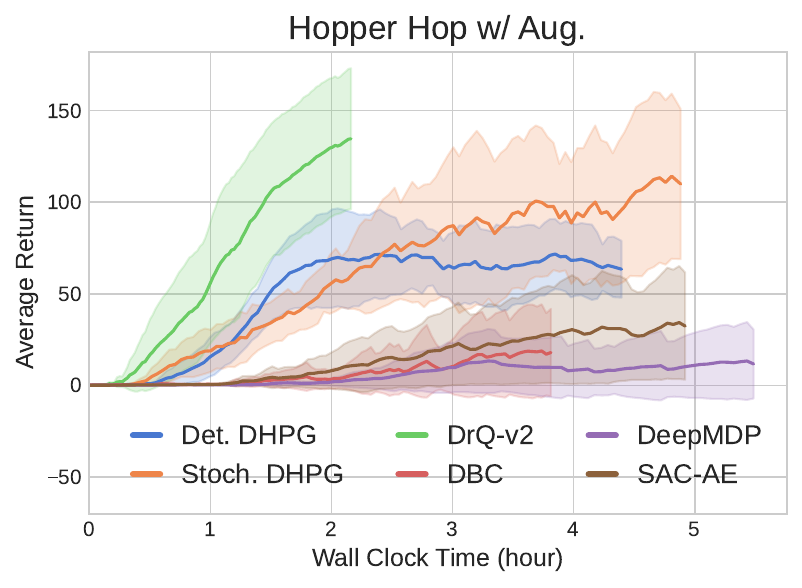}
     \end{subfigure}    
     \hfill
     \begin{subfigure}[b]{0.24\textwidth}
         \centering
         \includegraphics[width=\textwidth]{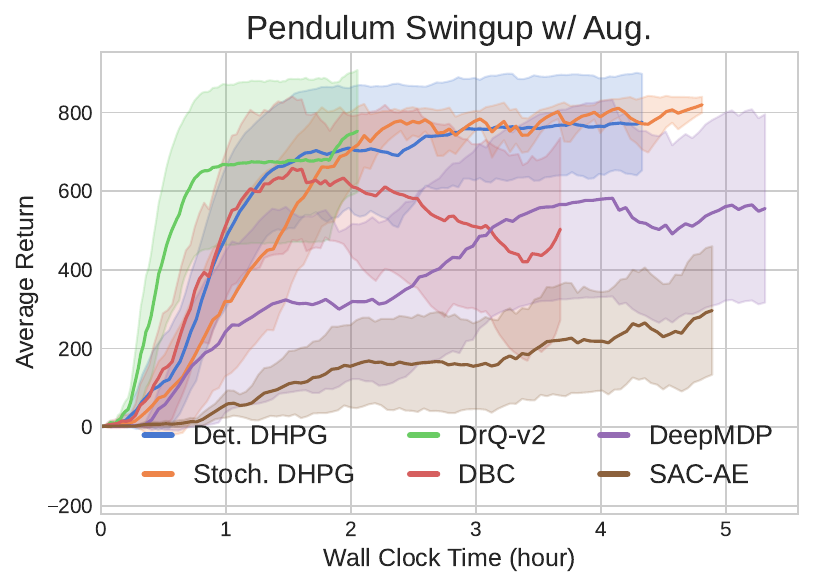}
     \end{subfigure}
     \hfill
     \begin{subfigure}[b]{0.24\textwidth}
         \centering
         \includegraphics[width=\textwidth]{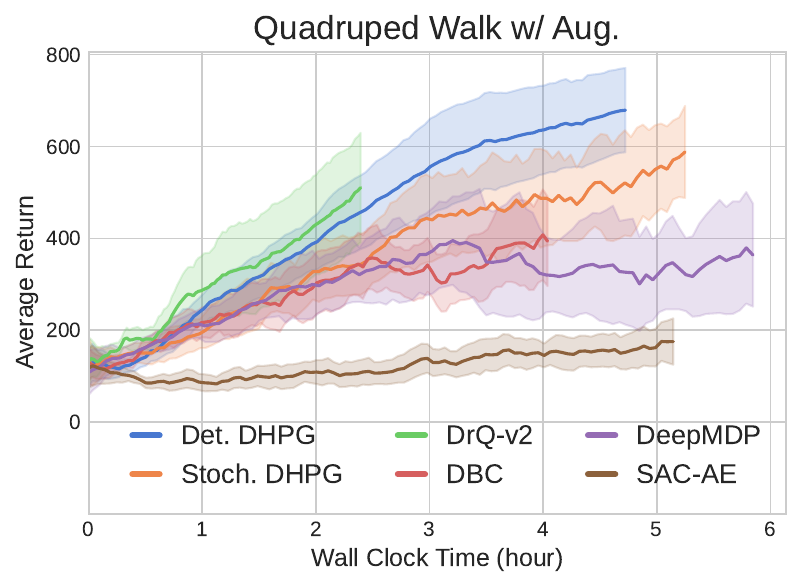}
     \end{subfigure}
     \hfill
     \begin{subfigure}[b]{0.24\textwidth}
         \centering
         \includegraphics[width=\textwidth]{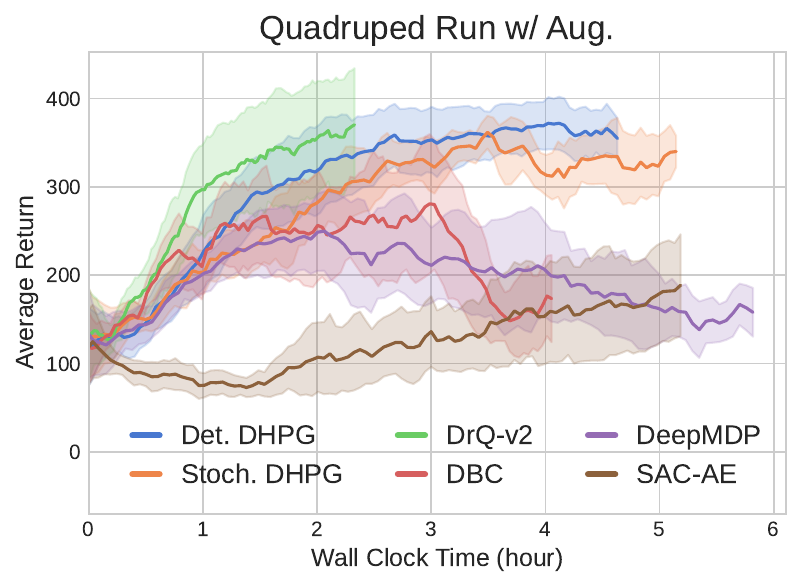}
     \end{subfigure}
     \hfill
     
     \begin{subfigure}[b]{0.24\textwidth}
         \centering
         \includegraphics[width=\textwidth]{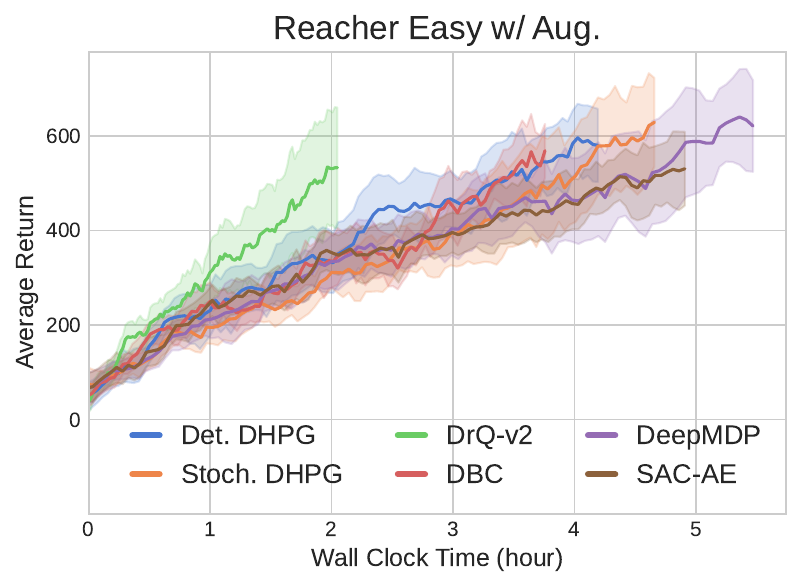}
     \end{subfigure}
     \hfill
     \begin{subfigure}[b]{0.24\textwidth}
         \centering
         \includegraphics[width=\textwidth]{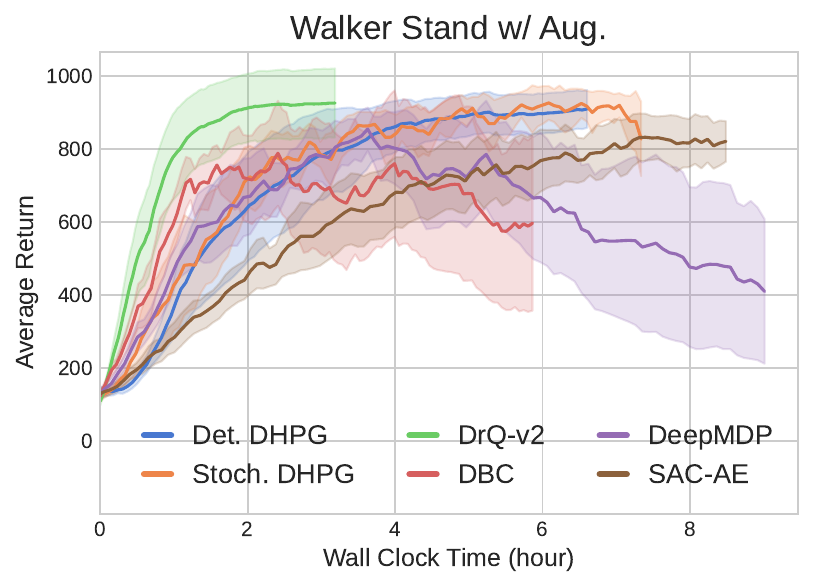}
     \end{subfigure}
     \hfill
     \begin{subfigure}[b]{0.24\textwidth}
         \centering
         \includegraphics[width=\textwidth]{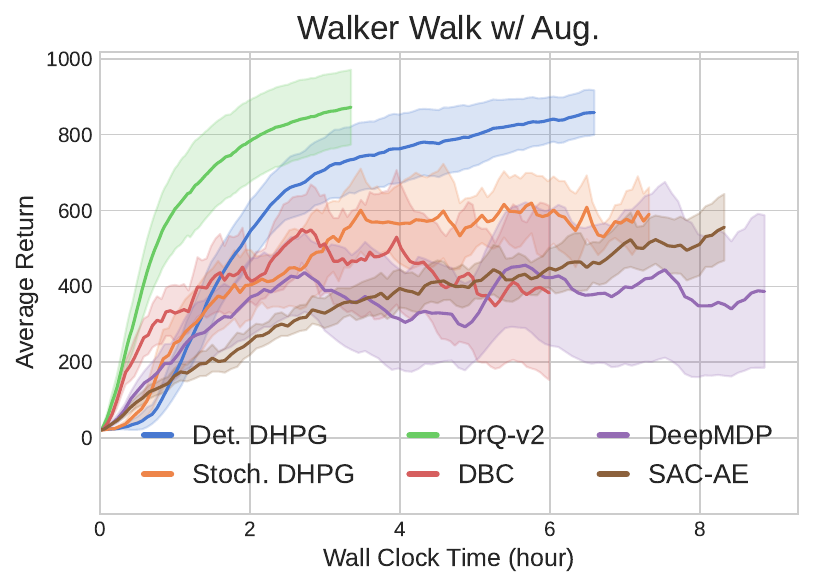}
     \end{subfigure}
     \hfill
     \begin{subfigure}[b]{0.24\textwidth}
         \centering
         \includegraphics[width=\textwidth]{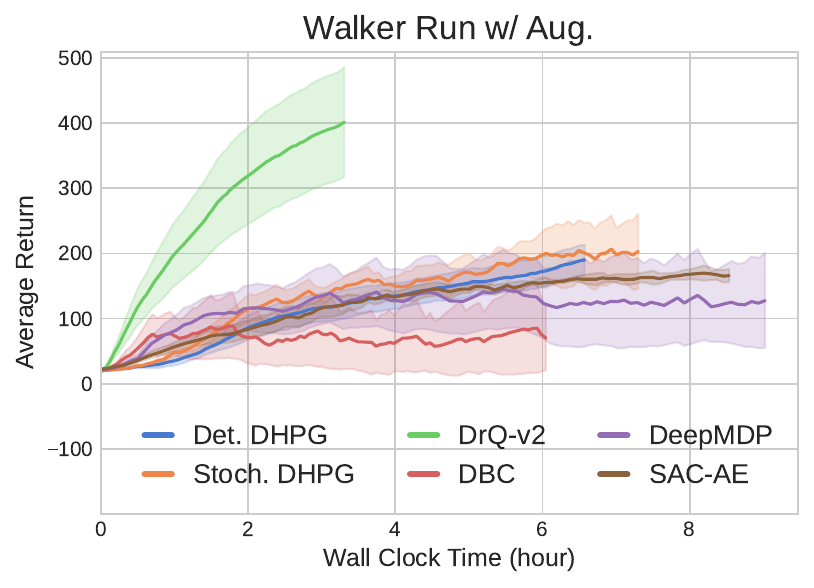}
     \end{subfigure}     
    \caption{Learning curves for 16 DM control tasks with pixel observations. The horizontal axis is the \textbf{wall clock time} in hours. All methods are \textbf{with} image augmentation. Mean performance is obtained over 10 seeds and shaded regions represent $95\%$ confidence intervals. Plots are smoothed uniformly for visual clarity.}
    \label{fig:pixel_results_supp_computation}
\end{figure}

\clearpage
\section{Implementation Details}
\label{sec:implementation}

\subsection{Hyperparameters}
\label{sec:hyperparams}

\begin{table}[b!]
    \centering
    \caption{Hyperparameters used in our experiments.}
    \label{tab:hyperparams}
    \begin{tabular}{cc}
        \hline
        \textbf{Hyperparameter}                       & \textbf{Setting}         \\
        \hline
        Learning rate                            & $1$e$-4$                   \\
        Optimizer                                & Adam                     \\
        $n$-step return                          & $3$                        \\
        Mini-batch size                          & $256$                      \\
        Actor update frequency $d$               & $2$                        \\
        Target networks update frequency         & $2$                        \\
        Target networks soft-update $\tau$       & $0.01$                     \\
        Target policy smoothing stddev. clip $c$ & 0.3                      \\
        Feature dim.                             & $50$                     \\
        Action repeat                            & $2$                      \\
        Frame stack                              & $3$                      \\
        Hidden dim.                              & $256$                      \\
        Replay buffer capacity                   & $10^6$                   \\
        Discount $\gamma$                        & $0.99$                     \\
        Seed frames                              & $4000$                     \\
        Exploration steps                        & $2000$                     \\
        Exploration stddev. schedule             & linear$(1.0, 0.1, 1$e$6)$ \\
        \hline
    \end{tabular}
\end{table}

Algorithm \ref{alg:dhpg} in Section \ref{sec:homomorphic_ac} presents the pseudo-code of stochastic and deterministic DHPG algorithms for pixel observations. 

In the image augmentation version of DHPG, as well as all the baselines, we use image augmentation of DrQ \citep{yarats2020image} that simply applies random shifts to pixel observations. First, $84 \times 84$ images are padded by $4$ pixels (by repeating boundary pixels), and then a random $84 \times 84$ crop is selected, rendering the original image shifted by $\pm 4$ pixels. Similarly to \citet{yarats2021mastering}, we also apply bilinear interpolation on top of the shifted image by replacing each pixel value with the average of four nearest pixel values. Our code is publicly available at \texttt{\href{https://github.com/sahandrez/homomorphic_policy_gradient}{https://github.com/sahandrez/homomorphic\_policy\_gradient}}.

We implemented our method in PyTorch \citep{paszke2019pytorch} and results were obtained using Python v3.8.10, PyTorch v1.10.0, CUDA 11.4, and Mujoco 2.1.1 \citep{todorov2012mujoco} on A100 GPUs on a cloud computing service. Table \ref{tab:hyperparams} present the hyperparameters used in our experiments. The hyperparameters are all adapted from DrQ-v2 \citep{yarats2021mastering} \emph{without any further hyperparameter tuning}. We have kept the same set of hyperparameters across all algorithms and tasks, except for the walker domain which similarly to DrQ-v2 \citep{yarats2021mastering}, we used $n$-step return of $n=1$ and mini-batch size of $512$.

The core RL components (actor and critic networks), as well as the components of DHPG (state and action encoders, transition and reward models) are all MLP networks with the ReLU activation function and one hidden layer with dimension of $256$. The image encoder is based on the architecture of DrQ-v2 which is itself based on SAC-AE \citep{yarats2021improving} and consists of four convolutional layers of $32 \times 3 \times 3$ with ReLU as their activation functions, followed by a one-layer fully-connected neural network with layer normalization \citep{ba2016layer} and tanh activation function. The stride of the convolutional layers are $1$, except for the first layer which has stride $2$. The image decoder of the baseline models with image reconstruction is based on SAC-AE \citep{yarats2021improving} and has a single-layer fully connected neural network followed by four transpose convolutional layers of $32 \times 32 \times 3$ with ReLU activation function. The stride of the transpose convolutional layers are $1$, except for the last layer which has stride $2$.

\subsection{Baseline Implementations}
\label{sec:baseline_impl}
All of the baselines are submitted in the supplemental material.
We use the official implementations of DBC and SAC-AE. DeepMDP does not have a publicly available code, and we use the implementation available in the official DBC code-base. As discussed in Section \ref{sec:experiments}, we have run two versions of the baselines, with and without image augmentation. The image augmented variants, use the same image augmentation method of DrQ-v2 described in Appendix \ref{sec:hyperparams}.

\end{document}